\def\paperpublic{}
\makeatletter
\def\input@path{{arxiv}}
\makeatother
\documentclass[11pt]{article}
\IfFileExists{pkg.sty}{\input{pkg.sty}}{\input{arxiv/pkg.sty}}
\usepackage{wrapfig}

\PassOptionsToPackage{hyphens}{url}
\usepackage[colorlinks,
            linkcolor=red,
            anchorcolor=blue,
            citecolor=magenta,
            ]{hyperref}
\usepackage{color}
\usepackage[dvipsnames]{xcolor}
\usepackage{latexsym}
\usepackage{dsfont}
\usepackage{appendix}
\usepackage{amssymb}
\usepackage{subcaption}
\usepackage[margin=1in]{geometry}
\usepackage{float}
\usepackage{algorithm,algorithmic}
\usepackage[sort&compress,numbers]{natbib}
\usepackage{amsmath, amsfonts,amssymb,euscript,array,amsfonts,mathrsfs}
\usepackage{enumerate}
\usepackage[dvipsnames]{xcolor}
\usepackage{amsthm}
\newtheorem{Theorem}{Theorem}[section]

\newtheorem{Assumption}[Theorem]{Assumption}
\newtheorem{Lemma}[Theorem]{Lemma}

\usepackage[T1]{fontenc}
\usepackage[utf8]{inputenc}
\usepackage{hhline}
\usepackage{graphicx}
\usepackage{verbatim}
\usepackage{eurosym}
\usepackage{dsfont}
\usepackage{mathtools}
\allowdisplaybreaks

\DeclareMathOperator{\Tr}{Tr}
\DeclareMathOperator{\E}{\mathbb{E}}
\DeclareMathOperator{\N}{\mathbb{N}}
\DeclareMathOperator{\R}{\mathbb{R}}

\DeclareMathOperator*{\argmin}{arg\,min}

\makeatletter
\@addtoreset{equation}{section}

\newcommand*{\rom}[1]{\expandafter\@slowromancap\romannumeral #1@}
\makeatother
\newcommand{\vertiii}[1]{{\left\vert\kern-0.25ex\left\vert\kern-0.25ex\left\vert #1 
    \right\vert\kern-0.25ex\right\vert\kern-0.25ex\right\vert}}

\usepackage{soul}
\usepackage{cleveref}
\usepackage{booktabs}

\makeatletter
\renewcommand{\@makefntext}[1]{\noindent\raggedright\@makefnmark\,#1}
\makeatother
\newif\ifpaperanonymous
\paperanonymoustrue
\ifdefined\paperpublic
  \paperanonymousfalse
\fi
\newcommand{\paperappendixtitle}{\ifpaperanonymous Online Companion\else Appendix\fi}
\begin{document}
\title{Neural Network-Based Score Estimation in Diffusion Models: Optimization and Generalization}

\ifpaperanonymous
  \author{(Anonymous for peer review)}
\else
  \author{Yinbin Han
    \thanks{Department of Management Science and Engineering, Stanford University. \textbf{Email:} \{yinbinha,renyuanxu\}@stanford.edu}
    \and
    Meisam Razaviyayn
    \thanks{Daniel J. Epstein Department of Industrial and Systems Engineering, University of Southern California. \textbf{Email:} razaviya@usc.edu}
    \and
    Renyuan Xu
    \footnotemark[1]
  }
\fi
\ifpaperanonymous
  \date{\today}
\else
  \date{First version:~Jan 2024; this version:~Apr 2026}
\fi
\maketitle
\allowdisplaybreaks

\begin{abstract} 
Diffusion models have emerged as a dominant paradigm in generative AI, rivaling GANs in producing high-fidelity and robust samples. A core component of these models is learning the score function of perturbed data distribution via denoising score matching.  While recent theoretical works have established strong statistical guarantees for diffusion models, they predominantly rely on algorithm-agnostic assumptions, presuming access to a theoretical oracle that perfectly minimizes the empirical risk. In practice, however, score functions are parameterized by highly non-convex neural networks and trained via gradient descent (GD). It remains a major open question whether practical gradient-based algorithms can navigate the optimization landscape of score matching to achieve provable accuracy.  As a first step toward answering this question, this paper establishes a mathematical framework for analyzing score estimation using neural networks trained by GD. 
     
     Our analysis covers both the optimization and the generalization aspects of the learning procedure. 
    In particular, we propose a novel parametric formulation that reduces denoising score matching to a regression problem with inherently noisy labels. 
    Unlike standard supervised learning, the score-matching problem introduces unique theoretical challenges, including unbounded input, vector-valued output, and an additional time variable, preventing existing techniques from being applied directly. 
    We address these challenges by showing that, with proper designs, the evolution of GD-trained neural networks can be accurately approximated by a sequence of localized kernel regression problems. Our analysis is grounded in a novel parametric form of the neural network and an innovative connection between score matching and regression analysis, which facilitate the application of advanced statistical and optimization techniques. Furthermore, since prolonged training on noisy labels causes catastrophic overfitting, we derive a novel extension of early-stopping rules for unbounded domains. This, in turn, allows us to establish the first minimax-optimal generalization error (sample complexity) bounds for GD-trained neural networks in diffusion models. Finally, we validate our theory-inspired optimization framework on a real-world Credit Default dataset, demonstrating that our principled approach achieves performance comparable to heavily tuned heuristic training schemes in generating high-fidelity financial tabular data. 
\end{abstract}
\vspace{-10pt}

\section{Introduction}
Diffusion models excel in diverse generative tasks, spanning image, video, and audio generation \citep{song2019generative,dathathri2019plug,song2020score,ho2020denoising}, often outperforming their contemporaries, including GANs, VAEs, normalizing flows, and energy-based models \citep{goodfellow2014generative,kingma2013auto,rezende2015variational,zhao2016energy}.

A typical diffusion model consists of two diffusion processes \citep{song2020score,sohl2015deep,ho2020denoising}:  one moving forward in time and the other moving backward. 
The forward process transforms a given data sample into white noise in the limit by gradually injecting noise through the diffusion term, while the backward process transforms noise to a sample from the data distribution by sequentially removing the added noise.  The implementation of the backward process depends on the  score function, defined as the gradient of the logarithmic density, at each timestamp of the forward process. In practice, however, the score function is unknown and one can only access the true data distribution via finitely many samples. To ensure the fidelity of the backward process in generating realistic samples, it is essential to develop efficient methods to estimate the score function using samples. This estimation is typically achieved through a process known as \textit{score matching}, employing powerful nonlinear functional approximations such as neural networks.

Despite the empirical success, it is theoretically less clear whether a gradient-based algorithm can train a neural network to learn the score function. Existing theoretical work \citep{de2021diffusion, chen2022improved,chen2022sampling,lee2023convergence,chen2023score,oko2023diffusion,mei2023deep,li2023towards,chen2023probability,shah2023learning,tang2024contractive,li2024generalization,liang2025low,gao2025wasserstein,huang2025convergence,zhang2025generalization} 
 predominantly focus on \textit{algorithm-agnostic} properties, such as score approximation bounds and distribution recovery guarantees. 
 A critical limitation of these works is that they implicitly assume the availability of an optimization oracle that {\it perfectly} solves the score-matching objective within a given function class. In reality, score functions are parameterized by overparameterized, highly non-convex neural networks and optimized using local search methods such as gradient descent (GD), and showing non-asymptotic convergence to the oracle solution is highly non-trivial. Furthermore, the Denoising Score Matching (DSM) objective utilizes inherently noisy targets. Standard deep learning theory cautions that overparameterized networks, when trained to convergence on noisy labels, can catastrophically overfit the noise. Consequently, it is theoretically ambiguous whether the actual optimization trajectory of GD will converge to a statistically optimal score estimator or simply memorize the noise. This work bridges this fundamental gap between statistical theory and practical deep learning optimization by providing the first end-to-end, algorithm-dependent analysis of score matching.

Our contributions are summarized as follows.

\paragraph{Our Work and Contributions.}
This work investigates the training of a two-layer fully connected neural network via gradient descent (GD) to learn the score function in diffusion models, and establishes a minimax generalization error bound.

We propose a neural network-based parametric form for the score estimator based on the score decomposition (see Lemma \ref{lemma:score-decomp}). This novel design transforms the score-matching objective into a regression with noisy labels. After a truncation argument, we further decompose the regression objective into four terms (see \eqref{eq:decomposition} for the decomposition): 1) score approximation, 2) coupling error, 3) label mismatch, and 4) early stopping. 

To handle score {\it approximation error}, we establish a universal approximation theorem with respect to the score function using the reproducing kernel Hilbert space (RKHS), induced by the neural tangent kernel (NTK) (see Theorem \ref{thm:approx-score-L2}). Compared with the existing results in the literature, we handle a time variable in addition to space variables.  {Next}, we leverage recent NTK-based analysis of neural networks to show the approximate equivalence between neural network training and kernel regression (see Theorem~\ref{thm:coupling}), with the gap  referred to as {\it coupling error}. Furthermore, we propose a virtual dataset to address the issue of {\it label mismatch} caused by the approximation step. In the presence of multi-output labels, a vector-valued localized Rademacher complexity bound is utilized to control the prediction error of {two kernel regressions}{, the original one and the one with the shifted target} (see Theorem \ref{thm:mismatch}).  Finally, we employ an {\it early-stopping} rule for the kernel regression to minimize the score-matching objective and {provide a minimax optimal generalization result} (see Theorem \ref{thm:score-estimation}).    The construction of the early-stopping rule is highly nontrivial and the main technical challenges are twofold. First,  we reduce the vector-valued learning problem to an equivalent scalar-valued kernel regression problem; see \eqref{eq:update-gamma-i}--\eqref{eq:update-uK-i}. Second,  existing early-stopping results in prior work typically assume bounded unit input domains. We overcome this limitation by extending a known stopping rule to handle unbounded input domains; see Theorem \ref{thm:early-stop}.

From an application standpoint, we test our theoretical framework on a high-stakes real-world problem:~generating financial tabular data from a Credit Default dataset. Although diffusion models are most often demonstrated on continuous data such as images, there is strong demand for high-fidelity generative models for the structured tabular datasets common in quantitative finance. Because such domains are often subject to privacy and regulatory constraints, realistic synthetic data generation is especially valuable. At the same time, reliability in this setting requires mathematically principled estimators to mitigate distribution shift and mode collapse. Our experiments show that the proposed theoretically grounded GD method, even with a simple neural network architecture, achieves fidelity comparable to much more sophisticated heuristic-based training schemes used in the modern diffusion literature.

To the best of our knowledge, this is the first work to establish end-to-end sample complexity bounds for GD-trained neural networks in the context of score matching.  By contrast, most prior statistical analyses assume that the empirical minimizer within a given neural network class can be directly obtained, without proving that it can in fact be achieved by any optimization algorithm.
 Specifically, our work is the first to employ NTK in establishing theoretical results for diffusion models. 
 While the Neural Tangent Kernel (NTK) framework has been extensively used to study supervised classification and regression, adapting it to diffusion models presents new mathematical challenges. Standard overparameterization theory heavily relies on bounded input domains (e.g., data on a unit hypersphere) and deterministic or bounded-noise labels. In contrast, learning the score function requires handling: 1) \textit{unbounded input domains} due to the Gaussian noise injected during the forward process; 2) \textit{noisy target labels} resulting from the DSM objective; and 3) \textit{vector-valued, time-conditioned outputs}. Our analysis overcomes these challenges by formulating the problem in a novel parametric form of the neural network. By rigorously applying truncation arguments, deriving a localized multi-output Rademacher complexity bound, and extending kernel early-stopping rules to unbounded domains to prevent noise memorization, we unlock the application of advanced statistical optimization techniques for generative modeling. Furthermore, the building blocks of our analysis can be applied to other supervised learning problems in non-standard forms {such as conditional generative adversarial networks \citep{liao2020conditional}, conditional flow matching \citep{lipman2022flow},  sequence-to-sequence modeling \citep{gu2021efficiently,smith2022simplified} (with unbounded input and vector-valued output)}, multimodal state space modeling \citep{rojas2025diffuse} and downstream decision-making problems \citep{zhao2025diffusion}, which is interesting in its own right.

\paragraph{Related Literature.}
Our work is related to  three categories of prior work:

First, our framework is closely related to the recent study of diffusion models. A line of work on this topic provides theoretical guarantees of diffusion models for recovering data distribution, assuming access to an accurate score estimator under $L^2$ or $L^\infty$ norm \citep{de2021diffusion, chen2022improved,chen2022sampling,lee2023convergence,shah2023learning,li2023towards,chen2023probability,gao2025wasserstein,tang2024contractive,li2024accelerating,li2024sharp,huang2025convergence,hu2024statistical,wu2024stochastic,li2024d,liang2025low,li2024generalization,mei2023deep,li2024good,zhang2025generalization,zhang2024emergence}. These results offer only a partial understanding of diffusion models as the score estimation part is omitted. To our best knowledge, \cite{chen2023score} and \cite{oko2023diffusion} are the only papers that provide score estimation rates under $L^2$ norm, assuming linear data structure or compactly supported data density. %
In parallel to non-parametric score estimation results, \cite{wibisono2024optimal,dou2024optimal,zhang2024minimax} establish minimax optimal score estimation guarantees from a kernel density estimation perspective.
However, all their emphasis is on algorithm-agnostic analysis without evaluation of any specific algorithms. In contrast, our work offers the first generalization error (sample complexity) bounds for GD-trained neural networks. Finally, we note that several related works have emerged after our conference proceeding \citep{han2024anonymous}, including studies on training behavior of diffusion models \citep{wang2024evaluating,han2024feature,zhang2025convergence}, score matching with transformer networks \citep{cao2024exploring}, and timestep scheduling and discretization in diffusion sampling \citep{huang2026art,aghapour2026entropy}.

Second, our techniques relate to the rich literature of deep learning theory. Inspired by the framework of NTK introduced by \cite{jacot2018neural}, recently \cite{du2018gradient,du2019gradient,allen2019convergence,zou2020gradient} establish a linear convergence rate of neural networks for fitting random labels. One key property of GD-trained neural networks is the so-called implicit regularization of parameters. Namely, the minimizer of overparameterized neural networks is close to the random initialization.  Combined with uniform convergence results in statistical learning, this implicit regularization leads to the generalization property of neural networks in the absence of label noise \citep{arora2019fine}. However, none of these works delves into the generalization ability of neural networks when confronted with noisy labels. \cite{kuzborskij2022learning} is the only work that attempts to study the GD-trained neural networks with additive noise. %
To tackle the challenge posed by the score matching, our approach and, consequently, our theoretical results differ from the existing literature on deep learning theory for supervised learning in three key aspects: 1) handling unbounded input, 2) dealing with vector-valued output, and 3) incorporating an additional time variable. %

 Lastly, our work is connected to a body of research focused on early-stopping rules in kernel regression. See \cite{celisse2021analyzing} for a comprehensive overview of this topic. Our work considers a multi-output extension of the early-stopping rule developed in \cite{raskutti2014early}, which controls the complexity of the predictor class based on empirical distribution. %

\paragraph{Comparison to the Earlier Conference Version.} We point out that this work is an extended version of the conference proceeding \citep{han2024anonymous}, which was published at ICLR 2024. Compared to the conference version, this work is  strengthened significantly in terms of mathematical results, as well as by providing a comprehensive numerical experiment (which was not included in the conference version). In particular,
\begin{enumerate}
    \item Theorem 3.12 in the conference version \citep{han2024anonymous} heavily relies on Assumption 3.11. This assumption requires that there exists an algorithm-aware early-stopping rule $\widehat{T}(N)$ for which the corresponding error bound $\varepsilon(N,\widehat{T}(N))$, governing the kernel approximation error, vanishes as the sample size $N$ tends to infinity. However, whether this assumption holds is highly non-trivial, and justifying it plays an essential role for the ultimate result. In this journal version, we provide an explicit early-stopping rule in Theorem \ref{thm:early-stop} to verify Assumption 3.11. (See the third paragraph under ``Our Work and Contributions'' for the technical novelties of this new addition.)
\item In addition, by leveraging this explicit stopping rule, we establish a minimax-optimal rate for score estimation in this journal version; see Theorem \ref{thm:score-estimation}. By contrast, the conference proceeding does not provide an explicit estimation error bound. %
\item This journal version also includes an experiment based on a credit default dataset (see Section \ref{sec:exp}), demonstrating promising empirical performance of our method. This experimental component is new compared to the conference version.
\end{enumerate}

\section{Preliminaries and Problem Statement}\label{sec:setup}

{\color{black}
In this section, we introduce the mathematical framework of diffusion models \citep{song2020score}, which consists of two stochastic processes: a forward process to add noise and a time-reversed process to denoise. %

\paragraph{Forward Process.} The forward process progressively injects noise into a data sampled from the target data distribution $p_0 \in \mathcal{P}(\R^d)$. In the context of data generation, we have the flexibility to work with any forward diffusion process of our choice. {Common instances include variance-exploding and variance-preserving SDEs \citep{song2020score}.} Let $(\Omega, \mathcal{F}, \P, (\mathcal{F}_t)_{t \geq 0})$ be a filtered probability space with filtration $(\mathcal{F}_t)_{t \geq 0}$. For the sake of theoretical convenience, we adhere to the standard convention in the literature~\citep{song2020improved,ho2020denoising} and focus on the $d$-dimensional Ornstein-Uhlenbeck (OU) process on $\R^d$. In particular, we study a  simple OU process with a deterministic weight function $ g(t) > 0 $:
\begin{align}\label{eq:OU-process}
    {\rm d} X_t  = -\frac{1}{2}g(t)X_t{\rm d}t + \sqrt{g(t)}{\rm d} B_t, \quad X_0\sim p_{0},
\end{align}
where  $ \parenthesis{B_{t}}_{t \geq 0} $ is a standard $ d $-dimensional Brownian motion defined on the aforementioned probability space and $ p_{0}$ represents the unknown data distribution from which we have access to only a limited number of samples. Our objective is to generate additional realistic samples from this distribution. %
The explicit solution to~(\ref{eq:OU-process}) is given by 
\begin{align*}
    X_t = e^{-\int_0^t\frac{1}{2}g(s){\rm d}s}X_0 + e^{-\int_0^t\frac{1}{2}g(s){\rm d}s}\int_0^t e^{\int_{0}^{s}\frac{1}{2}g(u){\rm d}u}\sqrt{g(s)}{\rm d} B_{s}.
\end{align*}
Consequently, the conditional distribution $ X_{t} \vert X_{0} $ follows a multi-variate Gaussian distribution $ \mathcal{N}(\alpha(t)X_{0}, h(t)I_d) $ with $ \alpha(t) \coloneqq \exp\parenthesis{-\int_0^t\frac{1}{2}g(s){\rm d}s} $ and $ h(t) \coloneqq 1 - \alpha^{2}(t) $. Furthermore, under mild assumptions, the OU process converges exponentially to the standard Gaussian distribution \citep{bakry2014analysis}. In practice, the forward process (\ref{eq:OU-process}) will terminate at a sufficiently large timestamp $T > 0$ such that the distribution {of $X_T$} is close to the standard Gaussian distribution.

\paragraph{Time-reversed Process.}
By reversing the forward process in time,  we obtain a process $Y_t:= X_{T-t}$ (well-defined under mild assumptions~\citep{haussmann1986time}) that transforms white noise into samples from the target data distribution, fulfilling the purpose of generative modeling. To start, let us first define a time-reversed SDE associated with (\ref{eq:OU-process}):
\begin{align}\label{eq:backward}
    {\rm d}Y_{t} = \parenthesis{\frac{1}{2}g(T - t)Y_{t} + g(T - t)\nabla \log p_{T - t}(Y_{t})}{\rm d}t  + \sqrt{g(T - t)}{\rm d}\bar{B}_{t},\quad Y_{0} \sim q_{0}
\end{align}
where $(\bar{B}_{t})_{t \ge 0}$ is another $d$-dimensional Brownian motion defined on the same probability space as $(B_t)_{t \geq 0}$,  {$p_t$ is the density of the forward process $X_t$}, the {\it score function} $ \nabla\log p_{t}(\cdot) $ is defined as the gradient of log density of $X_t$, and  $q_{0} $ is the initial distribution of the backward process. If the score function is known at each time $t$ and if we set $q_0 := p_T$, under mild assumptions, $(Y_t)_{0\leq t \leq T}$ has the {\it same time-marginal distribution} as $(X_{T-t})_{0\leq t \leq T}$; see \cite{follmer2005entropy,cattiaux2021time,haussmann1986time} for details. 

In practice, however, (\ref{eq:backward}) cannot be directly used to generate samples from the target data distribution as both the score function and the distribution $ p_{T} $ are {\it unknown}. To address this issue, it is common practice to replace $p_{T} $ by the standard Gaussian distribution as the initial distribution of the backward process. Then, we replace the ground-truth score $ \nabla \log p_{t}(x) $ by an estimator $ s_{\theta}(x, t) $. The estimator $ s_{\theta} $ is parameterized (and learned) by a neural network. With these modifications, we obtain an approximation of the backward process, which is practically implementable:
\begin{align}\label{eq:approx-back}
    {\rm d}Y_{t}' = \parenthesis{\frac{1}{2}g(T - t)Y_{t}' + g(T - t)s_{\theta}(Y_{t}', t)}{\rm d}t + \sqrt{g(T - t)}{\rm d}\bar{B}_t,\quad Y_{0}' \sim \mathcal{N}(0, I_{d}).
\end{align} 
To generate data using (\ref{eq:approx-back}), SDE solvers or  discrete-time approximation schemes can be used \citep{chen2023score,ho2020denoising,chen2022sampling,song2020score,chen2023probability}.

\paragraph{Score Matching.}To implement the time-reversed process \eqref{eq:approx-back}, we need to use samples to estimate the score function. A natural choice is to minimize the $ L^{2} $ loss between the estimated and actual scores:
\begin{align}\label{eq:ESM}
\min_{\theta}\;\; \frac{1}{T - T_{0}}\int_{T_{0}}^{T}\lambda(t) \E\bracket{\norm{s_{\theta}(X_{t}, t) - \nabla \log p_{t}(X_{t})}_{2}^{2}}{\rm d}t,
\end{align}
where {$ \lambda(t) $ is the weight function that captures time inhomogeneity} and $ s_{\theta} $ is the estimator of the score function. Here, $ T_{0}>0 $ is some small value to prevent the score function from blowing up and to stabilize the training procedure \citep{song2019generative,chen2023score,vahdat2021score}. %
A major drawback of the score-matching loss (\ref{eq:ESM}) is its intractability as $ \nabla \log p_{t} $ cannot be computed based on the available data samples. Thus,  instead of minimizing the loss in (\ref{eq:ESM}), one can equivalently minimize the following denoising score matching as shown by \cite{vincent2011connection}:
\begin{align}\label{eq:DSM}
    \min_{\theta}\;\;\frac{1}{T - T_{0}}\int_{T_{0}}^{T}\lambda(t) \E\bracket{\norm{s_{\theta}(X_{t}, t) - \nabla \log p_{t\vert 0}(X_{t}\vert X_{0})}_{2}^{2}}{\rm d}t.
\end{align}
Here, $ p_{t\vert 0}(\cdot \vert x_{0}) $ denotes the conditional probability density of $ X_{t} $ given $ X_{0} = x_0 $. It is easy to show that the choice of our forward process in (\ref{eq:OU-process}) implies
\begin{align}\label{eq:denosing-score}
\nabla \log p_{t\vert 0}(X_{t} \vert X_{0}) = \frac{\alpha(t)}{h(t)}X_{0} - \frac{X_{t}}{h(t)}.
\end{align}
Now, we can plug (\ref{eq:denosing-score}) into (\ref{eq:DSM}) and learn the score function estimator.  In practice, however, the score function estimator is parameterized by a neural network. Next, we discuss such a parameterization.

\begin{algorithm}[H]
    \caption{Sample Collection Procedure}\label{alg:score estimation}
    \begin{algorithmic}[1] 
    \STATE {\bf {Input:}} number of samples $ N $ and a small value $ T_{0} > 0 $
    \FOR{$j = 1, 2, \dots, N$} \label{line: sampling starts}
    \STATE Sample $ X_{0, j} \sim p_{0} $
    \STATE Sample $ t_{j} \sim {\rm Unif}[T_{0}, T] $
    \STATE Sample $ X_{t_{j}} \sim p_{t_{j}\vert 0}({}\cdot{} \vert X_{0, j}) $
    \ENDFOR 
    \RETURN $ \curly{\parenthesis{t_{j}, X_{0, j}, X_{t_{j}}}}_{j = 1}^{N} $
    \end{algorithmic}
\end{algorithm}

\paragraph{Neural Network-Based Parameterization.} To parametrize the function $ s_{\theta} $, we consider a two-layer ReLU neural network $ f_{{\bf W}, a} = \parenthesis{f^{i}_{{\bf W}, a}}_{i = 1}^{d} $ of the following form: %
\begin{align}
f^{i}_{{\bf W}, a}(x, t) & = \frac{1}{\sqrt{m}}\sum_{r = 1}^{m}a_{r}^{i}\sigma(w_{r}^{\top}(x, t - T_{0})). \label{eq:nn-def}
\end{align}
Here, $ (x, t) = (x^{1}, \dots, x^{d}, t)^{\top} \in \R^{d+1} $ is the input vector, $ w_{r} \in \R^{d+1} $ is a weight vector in the first layer, $ a_{r}^{i} \in \R $ is a weight scalar in the second layer, and $ \sigma(\cdot) $ is the ReLU activation.  The  {design term $ T_{0} $} introduced in the architecture plays an important role in the theoretical analysis and also offers valuable insights in practice; see Theorem \ref{thm:approx-score-L2}.  For ease of exposition, we denote $ {\bf W} = (w_{1}, \dots, w_{m})\in \R^{(d+1)\times m} $ and $ a = [a_{r}^{i}] \in \R^{m \times d} $. {We adopt the usual trick in the overparameterization literature \citep{cai2019neural,wang2019neural,allen2019convergence} with the parameter $a$ fixed throughout training and only  updating~${\bf W}$. This seemingly shallow architecture poses significant challenges when analyzing the convergence of gradient-based algorithms due to its non-convex and non-smooth objective. On the other hand, its ability to effectively approximate a diverse set of functions makes it a promising starting point for advancing theoretical developments.}

To train the neural network, we need to collect samples and design objective to measure the ``goodness-of-fit" of the neural network. We use Algorithm \ref{alg:score estimation} to generate $ N $ i.i.d.~data samples. In particular, for each $ j = 1, \dots, N $, we first sample $ X_{0, j} $ from $ p_{0} $ and  a timestamp $ t_{j} $ uniformly over the interval $ [T_{0}, T] $. Given $ X_{0, j} $ and $ t_{j} $, we then sample $ X_{t_{j}} $ from the Gaussian distribution $ p_{t_{j}\vert 0}({}\cdot{} \vert X_{0, j}) $. Given the output dataset  $ S := \curly{(t_{j}, X_{0, j}, X_{t_{j}})}_{j = 1}^{N} $, we train the neural network by minimizing a quadratic loss:
\begin{align}\label{eq:training-loss}
    \min_{\bf W}\;\;\widehat{\mcal{L}}({\bf W}) \coloneqq\frac{1}{2}\sum_{j = 1}^{N}\norm{f_{{\bf W}, a}(X_{t_{j}}, t_{j}) - X_{0, j}}_{2}^{2}.
\end{align}
Particularly, we perform the gradient descent (GD) update rule:\vspace{-2pt}
\begin{align}
    & w_{r}(\tau +1) - w_{r}(\tau) = - \eta\frac{\partial \widehat{\mcal{L}}(w_{r}(\tau))}{\partial w_{r}(\tau)} \nonumber  \\ & = -\frac{\eta}{\sqrt{m}}\sum_{j = 1}^{N}\sum_{i = 1}^{d}(f_{{\bf W}, a}^{i}(X_{t_{j}}, t_{j}) - X^{i}_{0, j})a_{r}^{i}(X_{t_{j}}, t_{j} - T_{0})\mathbb{I}\curly{w_{r}^{\top}(X_{t_{j}}, t_{j} - T_{0}) \geq 0}, \label{eq:gradient}
\end{align}
for~$r = 1,\ldots, m$. Here, $ \eta > 0 $ is the learning rate and $ \mathbb{I}\curly{\cdot} $ denotes the indicator function. We initialize the parameter $\bf W$ and $ a $ according to the following neural tangent kernel (NTK) regime \citep{jacot2018neural}:
\begin{align}
    a_{r}^{i} \sim {\rm Unif}\curly{-1, +1} \;\; \text{if} \;\; r \leq \frac{m}{2} \;\; \text{and} \;\; a_{r}^{i} = -a_{r-\frac{m}{2}}^{i} \;\; \text{if} \;\; r > \frac{m}{2}, \;\; \forall i \in [d] \label{eq:init a}
\end{align}
and
\begin{align}
w_r(0) \sim \mcal{N}(0, I_{d+1}) \;\; \text{if} \;\; r \leq \frac{m}{2} \;\; \text{and} \;\; w_{r}(0) = w_{r - \frac{m}{2}}(0) \;\; \text{if} \;\; r > \frac{m}{2}. \label{eq:init w}
\end{align}
Without loss of generality, we assume $m$ is an even number. Under the initialization scheme \eqref{eq:init a}--\eqref{eq:init w}, we guarantee that $f_{{\bf W}(0), a} \equiv 0$ with the choice of initial parameters. One can show that the training loss (\ref{eq:training-loss}) is an empirical version of the denoising score-matching loss defined in (\ref{eq:DSM}) under a carefully chosen $ s_{\theta} $; see \eqref{eq:regression}. Correspondingly, the finite sample performance of $s_\theta$ w.r.t.~(\ref{eq:ESM}) is referred to as {\it generalization ability}.

\paragraph{Neural Tangent Kernels.} 
For a two-layer ReLU neural network of the form (\ref{eq:nn-def}), we follow \citep{jacot2018neural} to introduce an associated NTK 
$ K: \R^{d+1}\times \R^{d+1} \to \R^{d \times d} $ whose $ (i, k) $-th entry is defined as \vspace{-6pt}
\begin{align*}
K^{ik}\left((x, t), \left(\tilde{x},\tilde{t}\right)\right) & \coloneqq 
z^{\top}\tilde{z}\,\E\bracket{a_{1}^{i}a_{1}^{k}\mathbb{I}\curly{w_{1}(0)^{\top}z \geq 0}\mathbb{I}\curly{w_{1}(0)^{\top}\tilde{z} \geq 0}},
\end{align*}
where $ z = (x, t - T_{0}) $ and $ \tilde{z} = (\tilde{x}, \tilde{t} - T_{0}) $. Here, the expectation is taken over all the randomness of $ a_{1}^{i} $, $ a_{1}^{k} $ and $ w_{1}(0) $. Similarly, we define a scalar-valued NTK $ \kappa:\R^{d+1} \times \R^{d+1} \to \R $ associated with each coordinate of the neural network: 
\begin{align*}
\kappa\Big((x, t), (\tilde{x}, \tilde{t})\Big) \coloneqq z^{\top}\tilde{z}\,\E\bracket{\mathbb{I}\curly{w_{1}(0)^{\top}z \geq 0}\mathbb{I}\curly{w_{1}(0)^{\top}\tilde{z} \geq 0}}.
\end{align*}
From the definition of the matrix-valued NTK, it is easy to see that $ K $ is a diagonal matrix and in particular,
$
    K\left((x, t), \left(\tilde{x}, \tilde{t}\right)\right)  = \kappa((x, t), (\tilde{x}, \tilde{t}))I_{d},
$
where $I_d$ is the $d$-dimensional identity matrix. 
Moreover, we let $ \mcal{H} $ be the reproducing Hilbert space (RKHS) induced by the matrix-valued NTK $ K $ and $ \mcal{H}_{1} $ be the RKHS induced by the scalar-valued NTK $ \kappa $ \citep{jacot2018neural,carmeli2010vector}. Finally,  given a dataset $ S $ and defining~$ z_{j} = (X_{t_{j}}, t_{j} - T_{0}) $, the Gram matrix~$ H $ of the kernel~$ K $ is defined as a $ dN \times dN $ block matrix with 
\begin{align}
        H \coloneqq  
        \begin{pmatrix}
            H_{11} & \cdots &  H_{1N} \\ \vdots & \ddots & \vdots \\ H_{N1} & \cdots & H_{NN}
        \end{pmatrix}, \quad 
    H^{ik}_{j\ell} \coloneqq z_{j}^{\top}z_{\ell}\E\bracket{a_{1}^{i}a_{1}^{k}\mathbb{I}\curly{z_{j}^{\top}w_{1}(0) \geq 0, z_{\ell}^{\top}w_{1}(0) \geq 0}}. \label{eq:ntk} 
\end{align}

\section{Main Results}\label{sec:results}

This section introduces our main theoretical results. We first propose a parametric form of $ s_{\theta} $ to simplify the score-matching loss in (\ref{eq:ESM}). Next, we show that the empirical version of DSM (\ref{eq:DSM}) is indeed equivalent to the quadratic loss defined in (\ref{eq:training-loss}). Finally, we provide a decomposition of an upper bound on the loss function into four terms:~a coupling term, a label mismatch term, a term related to early stopping, and an approximation error term. These terms are carefully analyzed later. %

To motivate our parametric form of $ s_{\theta} $,  we start by the following decomposition of the score function:
\begin{Lemma}\label{lemma:score-decomp}
    The score function $ \nabla \log p_{t}(x) $ admits the following decomposition:
    \begin{equation}\label{eq:equivalent_score}
    \nabla \log p_{t}(x) = \frac{\alpha(t)}{h(t)}\E\bracket{X_{0}\vert X_{t} = x} - \frac{x}{h(t)}.
    \end{equation}
\end{Lemma}

The proof of Lemma \ref{lemma:score-decomp} follows from the Gaussianity of the transition kernel $ p_{t\vert 0} $. {A similar decomposition has been proved in \cite[Lemma 1]{chen2023score} for data with linear structure, and in \cite{li2023towards} for discrete time analysis and its concurrent work \citep{mei2023deep}.} Compared to the expression of $ \nabla \log p_{t\vert 0}(\cdot \vert x_{0})  $ computed in (\ref{eq:denosing-score}), we replace $ X_{0} $ by $ \E\bracket{X_{0}\vert X_{t}} $ to obtain the ground-truth score function in  (\ref{eq:equivalent_score}). Consequently, we call $X_0$ the {\it noisy label} and $\E\bracket{X_0 \vert X_t}$ the {\it true label}. 

\begin{proof}
    Recall that the density function $ p_{t} $ can be written as
    \begin{align*}
    p_{t}(x) = \int p_{t\vert 0}(x \vert x_{0})p_{0}(x_{0}){\rm d}x_{0},
    \end{align*}
    where the transition kernel $ p_{t\vert 0}(x \vert x_{0}) = (2\pi h(t))^{-d/2}\exp\parenthesis{-\frac{1}{2h(t)}\norm{x - \alpha(t)x_{0}}_{2}^{2}} $. 
   The dominated convergence theorem implies, 
\begin{align*}
\nabla \log p_{t}(x) & = \frac{\nabla \int p_{t\vert 0}(x\vert x_{0})p_{0}(x_{0}){\rm d}x_{0}}{p_{t}(x)}  \\ & = \frac{\parenthesis{2\pi h(t)}^{-d/2}\int -\frac{x - \alpha(t)x_{0}}{h(t)}\exp\parenthesis{-\frac{\norm{x- \alpha(t)x_{0}}^{2}}{2h(t)}}p_{0}(x_{0}){\rm d}x_{0}}{p_{t}(x)} \\ & = \int - \frac{x - \alpha(t)x_{0}}{h(t)}\cdot\frac{p_{t \vert 0}(x \vert x_{0})p_{0}(x_{0})}{p_{t}(x)}{\rm d}x_{0} \\ & = \int - \frac{x - \alpha(t)x_{0}}{h(t)}\cdot p_{0 \vert t}(x_{0} \vert x){\rm d}x_{0} \\ & = \E\bracket{\left.\frac{\alpha(t)X_{0} - X_{t}}{h(t)} \right\vert X_{t} = x}\\ & = \frac{\alpha(t)}{h(t)}\E\bracket{X_{0}\vert X_{t} = x} - \frac{x}{h(t)},
\end{align*}
which completes the proof.  
\end{proof}

We make the following assumption on the diffusion models (\ref{eq:OU-process}). 
\begin{Assumption}\label{ass:bounded target}
    The target density function $p_0$ has a  compact support with $ \norm{X_{0}}_{2} \leq D$ almost surely, for some constant $ D > 0$. 
\end{Assumption}

Assumption~\ref{ass:bounded target} is satisfied in most practical settings, including the generation of images, videos, and audio. This assumption simplifies the subsequent analysis and can be relaxed to the sub-Gaussian tail assumption. Next, we propose the parametric form of $ s_{\theta} $ and $ \lambda(t) $ in the score-matching loss (\ref{eq:ESM}):
\begin{align*}
s_{{\bf W}, a}(x, t) = \frac{\alpha(t)}{h(t)}\Pi_{D}(f_{{\bf W}, a}(x, t)) - \frac{x}{h(t)}, \quad \textit{with} \,\, \lambda(t) = \frac{h(t)^{2}}{\alpha(t)^{2}},
\end{align*}
where $ \Pi_{D} $ is the projection operator onto the $ L^{2} $-ball with radius $ D $ centered at zero. With the choice of $ s_{{\bf W}, a} $ and $ \lambda(t) $ specified above, the score-matching loss (\ref{eq:ESM}) becomes
\begin{align}\label{eq:regression}
	\min_{{\bf W}}\;\;\frac{1}{T - T_{0}}\int_{T_{0}}^{T}\E\bracket{\norm{\Pi_{D}(f_{{\bf W}, a}(X_{t}, t)) - f_{*}(X_{t}, t)}_{2}^{2}}{\rm d}t,
	\end{align}
in which we define the target function as $ f_{*}(x, t) \coloneqq \E\bracket{X_{0} \vert X_{t} = x} $ and the expectation in \eqref{eq:regression} is taken over $ X_{t} $.  Given that only  $ {\bf W} $ is updated during optimization, in what follows, we omit $ a $ in the subscript of the neural network.  Our  loss function (\ref{eq:regression}) is also supported by empirical studies~\citep{ho2020denoising}.  In addition, (\ref{eq:regression}) can be viewed as a  regression task with noisy labels. In what follows, we will show that neural networks trained on noisy labels generalize well w.r.t.~(\ref{eq:regression}); see Theorem \ref{thm:score-estimation}.

One major challenge in the theoretical analysis, which distinguishes us from the standard supervised learning problems, is the unboundedness of the input $X_t$ in the objective function. To overcome this challenge, %
we employ a truncation argument with a threshold $ R $:
\begin{align}
	& \frac{1}{T - T_{0}}\int_{T_{0}}^{T}\E\bracket{\norm{\Pi_{D}(f_{\bf W}(X_{t}, t)) - f_{*}(X_{t}, t)}_{2}^{2}}{\rm d}t \nonumber \\  = \; & \frac{1}{T - T_{0}}\int_{T_{0}}^{T}\E\bracket{\norm{\Pi_{D}(f_{\bf W}(X_{t}, t)) - f_{*}(X_{t}, t)}_{2}^{2}\mathbb{I}\curly{\norm{X_{t}}_{2}\leq R}}{\rm d}t \label{eq:truncation<R}\\ &  \qquad + \frac{1}{T - T_{0}}\int_{T_{0}}^{T}\E\bracket{\norm{\Pi_{D}(f_{\bf W}(X_{t}, t)) - f_{*}(X_{t}, t)}_{2}^{2}\mathbb{I}\curly{\norm{X_{t}}_{2} > R}}{\rm d}t. \label{eq:truncation>R}
\end{align}

The next lemma controls the tail behavior in (\ref{eq:truncation>R}). 
\begin{Lemma}\label{lemma:tail bound}
    Suppose Assumption \ref{ass:bounded target} holds. Then, uniformly over all $\mathbf{W}$, it holds that
    \begin{align*}
        \frac{1}{T - T_{0}}\int_{T_{0}}^{T}\E\bracket{\norm{\Pi_{D}(f_{\bf W}(X_{t}, t)) - f_{*}(X_{t}, t)}_{2}^{2}\mathbb{I}\curly{\norm{X_{t}}_{2} > R}}{\rm d}t = \mathcal{O}({R}^{d-2}e^{-R^{2}/4}).
    \end{align*}
\end{Lemma}

\begin{proof}
The ideas in the proof are motivated by \cite{chen2023score}.
First,     note that
\begin{align}
p_{t\vert 0}(x_{t} \vert x_{0}) & = (2\pi h(t))^{-d/2}\exp\parenthesis{-\frac{1}{2h(t)}\norm{x_{t} - \alpha(t)x_{0}}_{2}^{2}} \nonumber\\ & \leq (2\pi h(t))^{-d/2}\exp\parenthesis{-\frac{1}{2h(t)}\parenthesis{\frac{1}{2}\norm{x_{t}}_{2}^{2} - \alpha^{2}(t)\norm{x_{0}}_{2}^{2}}} \nonumber\\ & \leq (2\pi h(t))^{-d/2}\exp\parenthesis{-\frac{1}{2h(t)}\parenthesis{\frac{1}{2}\norm{x_{t}}_{2}^{2} - \norm{x_{0}}_{2}^{2}}}. \label{eq:pt0-bound}
\end{align}
Denote the expectation with respect to the marginal distribution of $X_0$ as $\E_{X_0}$. 
With inequality (\ref{eq:pt0-bound}), we have
\begin{align}
&\frac{1}{T - T_{0}}\int_{T_{0}}^{T}\E\bracket{\norm{\Pi_{D}(f_{\bf W}(X_{t}, t)) - f_{*}(X_{t}, t)}_{2}^{2}\mathbb{I}\curly{\norm{X_{t}}_{2} > R}}{\rm d}t \nonumber\\ & \leq \frac{4D^{2}}{T - T_{0}}\int_{T_{0}}^{T}\E_{X_{0}}\bracket{\int_{\norm{x_{t}}_{2} \geq {R}}p_{t\vert 0}(x_{t}\vert X_{0}){\rm d}x_{t}}{\rm d}t \nonumber\\ & \leq \frac{4D^{2}}{T - T_{0}}\int_{T_{0}}^{T}(2\pi h(t))^{-d/2}\E_{X_{0}}\bracket{\exp\parenthesis{\frac{\norm{X_{0}}_{2}^{2}}{2h(t)}}}\parenthesis{\int_{\norm{x_{t}}\geq {R}}\exp\parenthesis{-\frac{\norm{x_{t}}^{2}}{4h(t)}}{\rm d}x_{t}}{\rm d}t \nonumber
\\ &  \leq \frac{4D^2\exp(D^2/2h(T_0))}{(T-T_0)(2\pi h(T_0))^{d/2}}\int_{T_{0}}^{T}\int_{\norm{x_{t}} \geq {R}}\exp\parenthesis{-\frac{\norm{x_{t}}^{2}}{4h(t)}}{\rm d}x_{t}{\rm d}t, \label{eq:tail-1}
\end{align}
where the last step holds due to the facts that $ h(t) \in [h(T_{0}), h(T)] $ and $ \norm{X_{0}} \leq D $. We bound the inner integral in (\ref{eq:tail-1}) by using the polar coordinate \cite[Corollary 2.51]{folland1999real}:
\begin{align*}
\int_{\norm{x_{t}} \geq {R}}\exp\parenthesis{-\frac{\norm{x_{t}}^{2}}{4h(t)}}{\rm d}x_{t} & = \frac{2\pi^{d/2}}{\Gamma(d/2)}\int_{{R}}^{\infty}\exp\parenthesis{-\frac{r^{2}}{4h(t)}}r^{d-1}{\rm d}r \\ & = \frac{(4h(t))^{d/2}\pi^{d/2}}{\Gamma(d/2)}\int_{{R}^{2}/(4h(t))}^{\infty}\exp\parenthesis{-u}u^{d/2 - 1}{\rm d}u \\ & = \frac{2(4h(t))^{d/2}\pi^{d/2}}{d\Gamma(d/2)}\int_{\parenthesis{{R}^{2}/(4h(t))}^{d/2}}^{\infty}\exp\parenthesis{-v^{2/d}}{\rm d}v \\ & \leq \frac{8h(t)\pi^{d/2}}{\Gamma(d/2)}{R}^{d-2}e^{-{R}^{2}/(4h(t))},
\end{align*}
where the last inequality follow from \cite[Equation 10]{qi1999some}. 
Therefore, we conclude that 
\begin{align*}
& \frac{1}{T - T_{0}}\int_{T_{0}}^{T}\E\bracket{\norm{\Pi_{D}(f_{\bf W}(X_{t}, t)) - f_{*}(X_{t}, t)}_{2}^{2}\mathbb{I}\curly{\norm{X_{t}}_{2} > R}}{\rm d}t  \\ & = 
\frac{4D^2\exp(D^2/2h(T_0))}{(T-T_0)(2\pi h(T_0))^{d/2}}\int_{T_0}^T \frac{8h(t)\pi^{d/2}}{\Gamma(d/2)}{R}^{d-2}e^{-{R}^{2}/(4h(t))} {\rm d} t \\ & \leq \frac{32 D^2\exp(D^2/2h(T_0))}{\Gamma(d/2)(2 h(T_0))^{d/2}} \cdot R^{d-2}e^{-R^2/4} = \mathcal{O}(R^{d-2}e^{-R^2/4}).
\end{align*}
\end{proof}

Lemma \ref{lemma:tail bound} states the term (\ref{eq:truncation>R}) is exponentially small w.r.t.~the threshold $ R $. Thus, it suffices to focus on the loss (\ref{eq:truncation<R}) over the ball with radius $ R $. As an immediate consequence of Lemma \ref{lemma:tail bound}, the following result shows that the input data of training, in the format of $(t_{j}, X_{0, j}, X_{t_{j}})$, also enjoys a concentration property.
\begin{Lemma}\label{lemma:sampling-prob}
    Suppose Assumption \ref{ass:bounded target} holds.
    Let $ \curly{(t_{j}, X_{0, j}, X_{t_{j}})}_{j = 1}^{N} $ be samples collected from Algorithm~\ref{alg:score estimation}. With probability at least $ 1 - \delta_1 $, we have
    \begin{align*}
        t_{j} \in [T_{0} + \Delta, T], \quad \norm{X_{t_{j}}}_{2} \leq R,
    \end{align*}
    where $ \delta_1 \coloneqq \frac{N\Delta}{T - T_{0}} + \mathcal{O}\parenthesis{NR^{d - 2}e^{-R^{2}/4}} $. 
\end{Lemma}

\begin{proof}
    Note that in the proof of Lemma \ref{lemma:tail bound}, we have shown that for any {$t\in[T_0+\Delta,T]$} 
    \begin{align*}
    \E\bracket{\mathbb{I}\curly{\norm{X_{t}}_{2} > R}} = \mcal{O}\parenthesis{R^{d - 2}e^{-R^{2}/ 4}}.
    \end{align*}
    It then follows
    \begin{align*}
         \frac{1}{T - T_{0}}\int_{T_{0} + \Delta}^{T}\E\bracket{\mathbb{I}\curly{\norm{X_{t}}_{2} \leq R}}{\rm d}t & = \frac{1}{T - T_{0}}\int_{T_{0} + \Delta}^{T}\parenthesis{1 - \E\bracket{\mathbb{I}\curly{\norm{X_{t}}_{2} > R}}}{\rm d}t \\ & \gtrsim 1 - \frac{\Delta}{T - T_{0}} - R^{d - 2}e^{-R^{2}/4}. 
    \end{align*}
    Set $ \delta' = \frac{\Delta}{T - T_{0}} + \mcal{O}\parenthesis{R^{d - 2}e^{-R^{2}/ 4}} $ and we have
    \begin{align*}
    \frac{1}{T - T_{0}}\int_{T_{0} + \Delta}^{T}\P\parenthesis{\norm{X_{t}}_{2} \leq R}{\rm d}t \geq 1 - \delta'.
    \end{align*}
    We set $ \delta = N\delta' $ and apply the union bound. Therefore, with probability at least $ 1 - \delta $, it holds that
    \begin{align*}
    t_{j} \in [T_{0} + \Delta, T], \quad \norm{X_{t_{j}}}_{2} \leq R.
    \end{align*}  
This completes the proof.

\end{proof}

Inspired by \cite{kuzborskij2022learning} for learning Lipschitz functions, we upper-bound (\ref{eq:truncation<R})
by the following decomposition at each iteration~$\tau$:
\begin{align}\label{eq:decomposition}
    & \frac{1}{4(T - T_{0})}\int_{T_{0}}^{T}\E\bracket{\norm{\Pi_{D}\parenthesis{f_{\bf W(\tau)}(X_{t}, t)} - f_{*}(X_{t}, t)}_{2}^{2}\mathbb{I}\curly{\norm{X_{t}}_{2}\leq {R}}}{\rm d}t  \\  \leq \; & \frac{1}{T - T_{0}}\int_{T_{0}}^{T}\E\bracket{\norm{\Pi_{D}\parenthesis{f_{\bf W(\tau)}(X_{t}, t)} - f^{K}_{\tau}(X_{t}, t)}_{2}^{2}\mathbb{I}\curly{\norm{X_{t}}_{2}\leq {R}}}{\rm d}t \tag{coupling}\\ & \qquad + \frac{1}{T - T_{0}}\int_{T_{0}}^{T}\E\bracket{\norm{f^{K}_{\tau}(X_{t}, t) - \tilde{f}^{K}_{\tau}(X_{t}, t)}_{2}^{2}\mathbb{I}\curly{\norm{X_{t}}_{2}\leq {R}}}{\rm d}t \tag{label mismatch} \\ & \qquad + \frac{1}{T - T_{0}}\int_{T_{0}}^{T}\E\bracket{\norm{\tilde{f}^{K}_{\tau}(X_{t}, t) - f_{\mcal{H}}(X_{t}, t)}_{2}^{2}\mathbb{I}\curly{\norm{X_{t}}_{2}\leq {R}}}{\rm d}t \tag{early stopping}\\ & \qquad +  \frac{1}{T - T_{0}}\int_{T_{0}}^{T}\E\bracket{\norm{f_{\mcal{H}}(X_{t}, t) - f_{*}(X_{t}, t)}_{2}^{2}\mathbb{I}\curly{\norm{X_{t}}_{2}\leq {R}}}{\rm d}t. \tag{approximation}
\end{align}

The first term is the coupling error between  neural networks $f_{\bf W(\tau)}$ and a function $f_{\tau}^{K}$ defined~as:
\begin{align}\label{eq:kernel-predictor}
    f_{\tau}^{K}(x, t) = \sum_{j = 1}^{N}K((X_{t_{j}}, t_{j}), (x, t))\gamma_{j}(\tau), \quad \gamma(\tau+1) = \gamma(\tau) - \eta (H \gamma(\tau) - y),
\end{align}
 where we initialize $\gamma(0) = 0$ and $y = (X_{0, 1}^\top, \dots, X_{0, N}^\top)^\top$.  The fourth term is the approximation error of the target function $f_*$ by a function~$f_\mcal{H}$ in the RKHS~$\mcal{H}$. These two terms transform the training of neural networks into a problem of kernel regression. To learn the function~$f_\mcal{H}$, we define an auxiliary function $\tilde{f}_\tau^K$ of the same functional form as ${f}_\tau^K$, but trained on a different dataset $\tilde{S} = \{(t_{j}, \tilde{X}_{0, j}, X_{t_{j}})\}_{j = 1}^{N} $ with
\begin{align}
    \tilde{X}_{0, j} \coloneqq f_{\mcal{H}}(X_{t_{j}}, t_{j}) + \varepsilon_{j}, \quad \varepsilon_{j} \coloneqq X_{0, j} - f_{*}(X_{t_{j}}, t_{j}).\label{eq:virtual-dataset}
\end{align}
Finally, we control the third term in the above decomposition by the early-stopping rule, {which is introduced in the statistical learning literature \citep{raskutti2014early,wei2017early}.}

\subsection{Approximation}\label{subsec:approx}

We start by analyzing the approximation term in our decomposition. This subsection focuses on the approximation error of the target function $f_*$ by a function in the RKHS $\mcal{H}$ induced by the NTK $K$. We start with a regularity assumption on the coefficient $ g(t) $ in the OU process. 

\begin{Assumption}\label{ass:g}
    The function $ g $ is almost everywhere continuous and bounded on $ [T_0, T] $. 
\end{Assumption}

Assumption \ref{ass:g} imposes a minimal requirement to guarantee that both $ \alpha(t) $ and $ h(t) $ are well-defined at each timestamp $ t \geq 0 $. In addition, the boundedness of $ g $ over the interval $[T_0, T] $ is used to establish the Lipschitz property of the score function with respect to $ t $ in the literature \citep{chen2023score,chen2022improved,chen2022sampling}. We also make the following smoothness assumption on the target function $f_*$. 
\begin{Assumption}\label{ass:lip-f*}
    For all $ (x, t)  \in \R^{d}\times [T_0, T] $, the  function $ f_{*}(x, t) $ is $ \beta_1 $-Lipschitz in $ x $, i.e., $ \abs{f_{*}(x, t) - f_{*}(x', t)}_{2} \leq \beta_1\norm{x - x'}_{2} $. 
\end{Assumption}

Assumption \ref{ass:lip-f*} implies the score function is Lipschitz w.r.t.~the input~$x$. This assumption is standard in the literature \citep{chen2022improved,chen2022sampling,chen2023score}. Yet, the Lipschitz continuity in Assumption \ref{ass:lip-f*} is only imposed on the regression function $f_*$, which is a consequence of the score decomposition. To justify Assumption~\ref{ass:lip-f*}, we provide an upper bound for the Lipschitz constant $ \beta_1 $ in Lemma \ref{lemma:betax-bound}.  The following theorem states a universal approximation theorem of using RKHS for score functions.

\begin{Theorem}[Universal Approximation of Score Function]\label{thm:approx-score-L2}
    Suppose Assumptions \ref{ass:bounded target}, \ref{ass:g} and \ref{ass:lip-f*} hold. Let $ R \geq T - T_{0} $ and $ R_{\mathcal{H}} $ be larger than a constant $ c_{1}(d) $ that depends only on $ d $. There exists a function $ f_{\mcal{H}} \in \mcal{H} $ such that $ \norm{f_{\mcal{H}}}_{\mcal{H}}^{2} \leq dR_{\mcal{H}} $ and
    \begin{align}
    \frac{1}{T - T_0}\int_{T_{0}}^{T}\E\bracket{\norm{f_{\mcal{H}}(X_{t}, t) - f_{*}(X_{t}, t)}_{2}^{2}\mathbb{I}\curly{\norm{X_{t}}_{2}\leq {R}}}{\rm d}t \leq dA^{2}(R_{\mcal{H}}, {R}), \label{eq:l_2 approx}
    \end{align}
    where $A(R_{\mcal{H}}, R) \coloneqq  c_{1}(d)\Lambda(R)\parenthesis{\frac{\sqrt{R_{\mcal{H}}}}{\Lambda(R)}}^{-\frac{2}{d-1}}\log\parenthesis{\frac{\sqrt{R_{\mcal{H}}}}{\Lambda(R)}}$ and $ \Lambda(R) = \mathcal{O}(R^{2}) $.
\end{Theorem}
Theorem \ref{thm:approx-score-L2} provides an approximation result of the target function by the RKHS under the $L^{2}$ norm. For each given $ R $,  $ R_{\mathcal{H}} $ is chosen large enough such that $ A(R_\mathcal{H}, R) $ is arbitrarily small.  We provide here a proof  sketch of Theorem~\ref{thm:approx-score-L2}. We first construct an auxiliary function $ \tilde{f}_{*}(x, t) \coloneqq f_{*}(x, \abs{t} + T_{0}) $. One can show that $ \tilde{f}_{*} $ is Lipschitz continuous in $ (x,t) \in \R^{d+1} $. Then for each coordinate $ i $, we apply an approximation result on RKHS for Lipschitz functions over a $L^{\infty}$-ball \citep[Proposition 6]{bach2017breaking} 
to find a function that approximates $ \tilde{f}_{*}^{i} $ well. Since the NTK is not a translation invariant kernel, we need to construct a shifted NTK such that $ f_{\mcal{H}}^{i} \in \mcal{H}_{1} $ is close to $ f_{*}^{i} $ after translation. The rest is to show that $ f_{\mcal{H}} = (f^{i}_{\mcal{H}})_{i = 1}^{d}  $ lies in the vector-valued RKHS $ \mcal{H} $. 

\begin{proof}[Proof of Theorem \ref{thm:approx-score-L2}.]
    It suffices to find a function $ f_{\mcal{H}}: \R^{d+1} \to \R^{d} $ in the RKHS induced by $ {K} $ such that $ \norm{f_{\mcal{H}}}_{\mcal{H}}^{2} \leq dR_{\mcal{H}}$ and for all $R \geq T - T_0$,
    \begin{align}
    \sup_{\norm{x}_{\infty} \leq R}\,\,\sup_{t \in [T_{0}, T]}\norm{f_{*}(x, t) - f_{\mcal{H}}(x, t)}_{\infty} \leq A(R_{\mcal{H}}, R).\label{eq:L_inf approx}
    \end{align}
    To see \eqref{eq:L_inf approx} leads to \eqref{eq:l_2 approx}, we notice that
    \begin{align}
        \norm{f_{\mcal{H}}(X_{t}, t) - f_{*}(X_{t}, t)}_{2}^{2}\mathbb{I}\curly{\norm{X_{t}}_{2}\leq {R}}  \leq d\sup_{\norm{x}_{\infty} \leq {R}}\sup_{t \in [T_{0}, T]}\norm{f_{\mcal{H}}(x, t) - f_{*}(x, t)}_{\infty}^{2} \leq dA^{2}(R_{\mcal{H}}, {R}), \label{eq:l_2 approx pre}
    \end{align}
    for all $ {R} \geq T - T_{0} $ and $ t \in [T_{0}, T]$. Integration on both sides of \eqref{eq:l_2 approx pre} over $[T_0, T]$ results in
    \begin{align*}
         \int_{T_{0}}^{T}\E\bracket{\norm{f_{\mcal{H}}(X_{t}, t) - f_{*}(X_{t}, t)}_{2}^{2}\mathbb{I}\curly{\norm{X_{t}}_{2}\leq {R}}}{\rm d}t \leq d(T - T_{0})A^{2}(R_{\mcal{H}}, {R}).
    \end{align*}
    Dividing both sides by $T - T_0$ completes the proof. 

    To prove \eqref{eq:L_inf approx}, we need the definition of two auxiliary kernels without a bias term. Let $\tilde{\mcal{H}}_1$ be the real-valued RKHS induced by the scalar-valued NTK $ \tilde{\kappa}:\R^{d+1} \times \R^{d+1} \to \R $ defined as
    \begin{align*}
    \tilde{\kappa}(z, \tilde{z}) \coloneqq z^{\top}\tilde{z}\E\bracket{\mathbb{I}\curly{w_{1}(0)^{\top}z \geq 0}\mathbb{I}\curly{w_{1}(0)^{\top}\tilde{z} \geq 0}}.
    \end{align*}
    Similarly, let  $\tilde{\mcal{H}}$ be the vector-valued RKHS induced by the matrix-valued NTK $ \tilde{K}: \R^{d+1}\times \R^{d+1} \to \R^{d \times d} $ defined as
    \begin{align*}
        \tilde{K}(z, \tilde{z})  = \tilde{\kappa}(z, \tilde{z})I_{d}.
    \end{align*}
    We also consider an auxiliary target function $ \tilde{f}_{*} :\R^{d} \times \R \to \R^{d} $ by $ \tilde{f}_{*}(x, t) \coloneqq f_{*}(x, \abs{t} + T_{0})$. By Assumptions \ref{ass:lip-f*} and Lemma \ref{lemma:lip-t} in Appendix \ref{sec:pf of lip-t}, the function $ f_{*}(x, t) $ is $ \beta_1 $-Lipschitz in $ x $ and $ \beta_2(R) $-Lipschitz in $ t $ for all $ \norm{x}_{\infty} \leq R$ and $ t \in [T_{0},  T]$; so is each coordinate map. Since $\norm{f_*(x, t)}_2 \leq D$ for any $(x, t) \in \R^{d} \times [T_0, T]$, it follows $ \sup_{\norm{(x, t)}_{\infty} \leq R}\norm{\tilde{f}_{*}(x, t)}_{2}\leq D $. Moreover, for any two pairs $ \norm{(x, t)}_{\infty}, \norm{(x', t')}_{\infty} \leq R $, it holds that
    \begin{align}\label{eq:approx-inf}
    \norm{\tilde{f}_{*}(x, t) - \tilde{f}_{*}(x', t')}_{2}  & \leq \norm{\tilde{f}_{*}(x, t) - \tilde{f}_{*}(x', t)}_{2} + \norm{\tilde{f_{*}}(x', t) - \tilde{f}_{*}(x', t')}_{2} \nonumber \\ & = \norm{f_{*}(x, \abs{t} + T_{0}) - f_{*}(x', \abs{t} + T_{0})}_{2} +  \norm{f_{*}(x', \abs{t} + T_{0}) - f_{*}(x', \abs{t'} + T_{0})}_{2} \nonumber \\ & \leq \beta_1\norm{x - x'}_{2} + \beta_2(R)\abs{\abs{t} - \abs{t'}} \nonumber \\ & \leq \parenthesis{\beta_1 + \beta_2(R)}\norm{(x, t) - (x', t')}_{2}.
    \end{align}
    For each coordinate $ i \in [d] $,  it follows from \citep[Proposition 6]{bach2017breaking} that there exists a function $ \tilde{f}_{\tilde{\mcal{H}}_{1}}^{i} \in \tilde{\mcal{H}}_{1} $ with $ \norm{\tilde{f}^{i}_{\tilde{\mcal{H}}_{1}}}_{\tilde{\mcal{H}}_{1}}^{2} \leq R_{\mcal{H}} $ such that
    \begin{align*}
    \sup_{\norm{(x, t)}_{\infty} \leq R}\abs{\tilde{f}_{*}^{i}(x, t) - \tilde{f}^{i}_{\tilde{\mcal{H}}_{1}}(x, t)} \leq A(R_{\mcal{H}}, R) \coloneqq  c_{1}(d)\Lambda(R)\parenthesis{\frac{\sqrt{R_{\mcal{H}}}}{\Lambda(R)}}^{-\frac{2}{d-1}}\log\parenthesis{\frac{\sqrt{R_{\mcal{H}}}}{\Lambda(R)}}. 
    \end{align*}
    where $\Lambda(R) = \max\curly{D, R\left\lbrace\beta_1 + \beta_2(R)\right\rbrace}$.
    
    Let $ f^{i}_{\mcal{H}}(x, t) \coloneqq \tilde{f}_{\tilde{\mcal{H}}_{1}}^{i}(x, t - T_{0}) $. For all coordinate $i \in [d]$ we have
    \begin{align*}
        \sup_{\norm{x}_{\infty} \leq R}\,\,\sup_{t \in [T_{0}, R + T_{0}]}\abs{f_{*}^{i}(x, t) - f_{\mcal{H}}^{i}(x, t)} \leq A(R_{\mcal{H}}, R).
    \end{align*}
    Note that $ f^{i}_{\mcal{H}}:\R^{d+1} \to \R $ lies in the RKHS induced by the kernel $ {\kappa}((x, t), (x', t')) = \tilde{\kappa}((x, t-T_{0}), (x', t' - T_{0})) $ and $ \norm{f^{i}_{\mcal{H}}}_{\mathcal{H}_1} = \norm{\tilde{f}^{i}_{\tilde{\mcal{H}}_{1}}}_{\tilde{\mcal{H}}_{1}} $. 
    We next show that $ f_{\mcal{H}} = (f^{1}_{\mcal{H}}, \dots, f^{d}_{\mcal{H}})$ is in the RKHS induced by $ K $. Since each coordinate of $ f^{i}_{\mcal{H}} $ lies in the RKHS induced by $ {\kappa} $, by relabeling data points, without loss of generality, it suffices to consider
    \begin{align*}
    f^{i}_{\mcal{H}}(\cdot) = \sum_{p = 1}^{P}\alpha_{p}^{i}{\kappa}((x, t)_{p}, \cdot), \quad (x, t)_{p} \in \R^{d+1}, \alpha_{p}^{i} \in \R.
    \end{align*} 
    It follows
    \begin{align*}
    f_{\mcal{H}}(\cdot) = \sum_{i = 1}^{d}{f}_{\mcal{H}}^{i}(\cdot)\mathbf{e}_{i} = \sum_{i = 1}^{d}\parenthesis{\sum_{p = 1}^{p}\alpha_{p}^{i}\kappa((x, t)_{p}, \cdot)}\mathbf{e}_{i} = \sum_{p = 1}^{P}{K}((x, t)_{p}, \cdot)\parenthesis{\sum_{i = 1}^{d}\alpha_{p}^{i}\mathbf{e}_{i}} \in \mcal{H}.
    \end{align*}
    Moreover, reproducing property implies
    \begin{align}
    \norm{f_{\mcal{H}}}_{\mcal{H}}^{2} & = \inpro{\sum_{p = 1}^{P}{K}((x, t)_{p}, \cdot)\parenthesis{\sum_{i = 1}^{d}\alpha_{p}^{i}\mathbf{e}_{i}}, \sum_{q = 1}^{P}{K}((x, t)_{q}, \cdot)\parenthesis{\sum_{k = 1}^{d}\alpha_{q}^{k}\mathbf{e}_{k}}}_{\mcal{H}} \nonumber \\ & = \sum_{p, q}\sum_{i, k}\alpha_{p}^{i}\alpha_{q}^{k}\mathbf{e}_{i}^{\top}{K}((x, t)_{p}, (x, t)_{q})\mathbf{e}_{k} \nonumber \\ & = \sum_{i = 1}^{d}\sum_{p, q}\alpha_{p}^{i}\alpha_{q}^{i}{\kappa}((x, t)_{p}, (x, t)_{q}) \nonumber \\ & = \sum_{i = 1}^{d}\inpro{\sum_{p = 1}^{P}\alpha_{p}^{i}{\kappa}((x, t)_{p}, \cdot), \sum_{q = 1}^{P}\alpha_{q}^{i}{\kappa}((x, t)_{q}, \cdot)}_{\mcal{H}_1} \nonumber \\ & = \sum_{i = 1}^{d}\norm{f^{i}_{\mcal{H}}}_{\mathcal{H}_1}^{2} =  \sum_{i = 1}^{d}\norm{\tilde{f}^{i}_{\tilde{\mcal{H}}_{1}}}_{\tilde{\mcal{H}}_{1}}^{2} \leq dR_{\mcal{H}}.
    \end{align}
    Therefore, we have found a function $ f_{\mcal{H}}$ such that $ \norm{f_{\mcal{H}}}_{\mcal{H}}^{2} \leq dR_{\mcal{H}}$ and \eqref{eq:L_inf approx} holds.
\end{proof}

\subsection{Coupling}\label{subsec:coupling}
This subsection provides a coupling argument to control the error between the neural network training and the kernel regression.
The next assumption is on the minimum eigenvalue of the Gram matrix $H$ of the kernel $ K $ and is standard in literature \citep{du2018gradient,bartlett2021deep,nguyen2021tight,suh2024survey}.

\begin{Assumption}\label{ass:gram-eigen}
	There exists a constant $ \lambda_{0}\geq 1 $,  dependent on $d$, such that the smallest eigenvalue  satisfies $\lambda_{\min}(H) \geq \lambda_{0}$, with probability at least $ 1 - \delta_2(d) $, and $\delta_2 \to 0$ as $d$ increases. 
\end{Assumption}

As shown in the literature of deep learning theory \citep{allen2019learning,arora2019fine,liu2022loss}, the Gram matrix~$ H $ is a fundamental quantity that determines the convergence rate of neural network optimization. {We also remark that Assumption \ref{ass:gram-eigen} is usually satisfied with a sample-dependent lower bound $\lambda_0$; see Lemma \ref{lemma:H-Hii-eigen} in the appendix for a justification and see also \cite{nguyen2021tight} for analysis of scalar-valued NTK.}  Now we are ready to state our main theorem for the coupling error. Let us denote {$ C_{\max} = \sqrt{R^{2} + (T - T_{0})^{2}} $} and recall $\delta_1$ as defined in Lemma \ref{lemma:sampling-prob}.

\begin{Theorem}[Coupling Error]\label{thm:coupling}
    Suppose Assumptions \ref{ass:bounded target} and \ref{ass:gram-eigen} hold. If we set $ m = \Omega\parenthesis{\frac{(dN)^{6}C_{\max}^{6}}{\lambda_{0}^{10}\delta^{3}\Delta^{2}}}$,  initialize $ a_{r}^{i}$  and $ w_{r}(0)  $  as in \eqref{eq:init a} and \eqref{eq:init w},  initialize $\gamma(0) = 0$, and set $ \eta = \mcal{O}\parenthesis{\frac{\lambda_{0}}{(dN)^{2}C_{\max}^{4}}} $, then with probability at least $ 1 - \delta - \delta_1 $, for all $ \tau \geq 0 $ and $ r = 1, \dots, m $ simultaneously, we have
    \begin{align*}
        & \frac{1}{T - T_{0}}\int_{T_{0}}^{T}\E\bracket{\norm{\Pi_{D}\parenthesis{f_{\bf W(\tau)}(X_{t}, t)} - f^{K}_{\tau}(X_{t}, t)}_{2}^{2}\mathbb{I}\curly{\norm{X_{t}}_{2}\leq {R}}}{\rm d}t 
        \lesssim \frac{\Delta D^{2}}{T - T_{0}} + \frac{d^7 N^7C_{\max}^{12}D^4}{\sqrt{m}\lambda_0^2\delta^2\Delta^2}. 
    \end{align*}
\end{Theorem}

The complete proof  is deferred to Appendix \ref{sec:proof-couploing}. Theorem \ref{thm:coupling} controls the error between the neural network training and the kernel regression. One can choose $ m = {\rm Poly}(d, N, R, \Delta, \lambda_{0}, \delta) $ and optimize over $ R $ and $ \Delta $ to make the error term small. For each {\it fixed} input data sample, \cite[Theorem 3.2]{arora2019exact} shows that the coupling error is small with high probability. Our analysis improves this result by showing that the $L^2$ coupling error also remains small with high probability.
To prove Theorem \ref{thm:coupling}, we first show that the training loss (\ref{eq:training-loss}) converges with a linear rate ({\it cf.} Theorem~\ref{thm:gd-conv}). Next, we show that $ f_{{\bf W}(\tau)} $ performs similarly as a linearized function $ f_{{\bar{\bf W}}(\tau)}^{\rm lin} $ at each iteration $ \tau $. Finally, we argue that the $ L^{2} $ loss between the  $  f_{{\bar{\bf W}}(\tau)}^{\rm lin} $ and $ f_{\tau}^{K} $ is small because of the concentration of kernels and a carefully chosen initialization~$\gamma(0)$ depending on the neural network initialization.

\subsection{Label Mismatch}\label{subsec:label}
In this subsection, we provide an upper bound for the error term induced by the label mismatch. Recall that $ f_{\tau}^{K} $ in \eqref{eq:kernel-predictor} is trained by the kernel regression on the  dataset $ S $ while $ \tilde{f}_{\tau}^{K} $ is trained on the virtual dataset $ \tilde{S} $. We control the error induced by the label mismatch in the following theorem.

\begin{Theorem}[Label Mismatch]\label{thm:mismatch}
    Suppose Assumptions \ref{ass:bounded target} and \ref{ass:gram-eigen} hold. With probability at least $ 1 - \delta - \delta_1$ it holds simultaneously for all $ \tau $ that
    {\small
    \begin{align*}
        \frac{1}{T - T_{0}}\int_{T_{0}}^{T}\E\bracket{\norm{f_{\tau}^{K}(x, t) - \tilde{f}_{\tau}^{K}(x,t)}_{2}^{2}\mathbb{I}\curly{\norm{X_{t}}_{2} \leq R}}dt   \leq dA^{2}(R_{\mcal{H}}, R) + C_{0}\parenthesis{\sqrt{dA^{2}(R_{\mcal{H}}, R)\Gamma_{\delta}} + \Gamma_{\delta}}, 
    \end{align*}
    }
    \normalsize
    where
    \begin{align*}
    \Gamma_{\delta} & \coloneqq 4d\parenthesis{\frac{\sqrt{d}A(R_{\mcal{H}}, R)C_{\max}}{\lambda_{0}}\log^{3/2}\parenthesis{\frac{eC_{\max}(dN)^{3/2}A(R_{\mcal{H}}, R)}{\lambda_{0}}} + \frac{1}{\sqrt{N}}}^{2} \\ & \qquad + \frac{d^{2}A^{2}(R_{\mcal{H}}, R)C_{\max}^{2}}{\lambda_{0}^{2}}\parenthesis{\log(1/\delta) + \log\parenthesis{\log N}},
    \end{align*}
    $C_0$ is a constant defined in Lemma \ref{lemma:local-Rad} and $C_{\max}$ is defined in Theorem \ref{thm:coupling}.
\end{Theorem}

Theorem \ref{thm:mismatch} links the error between $ f_{\tau}^{K} $ and $ \tilde{f}_{\tau}^{K} $ to the approximation error $ A(R_{\mcal{H}}, R) $. The proof of Theorem \ref{thm:mismatch} consists of two parts. We first utilize  the kernel regression structure to show that the predictions of $ f_{\tau}^{K} $ and $ \tilde{f}_{\tau}^{K} $ are similar over all the samples $ \curly{(t_{j}, X_{t_{j}})}_{j=1}^N $. Next, we apply the vector-valued localized Rademacher complexity ({\it cf.}~Lemma \ref{lemma:local-Rad}) to show that the performance of these two functions is also close in terms of the population loss.
The localized Rademacher complexity result is stated as follows.
\begin{Lemma}{\cite[Theorem 1]{reeve2020optimistic}}\label{lemma:local-Rad}
    Let $ \mcal{F} = \curly{f: \R^{d} \times [T_{0}, T] \to [-\beta, \beta]^{d}} $ for some $ \beta \geq 1 $. Take $ \delta \in (0, 1) $ and define
    \begin{align*}
    \Gamma_{\delta}(\mcal{F}) & \coloneqq \parenthesis{2\sqrt{d}\parenthesis{\sqrt{d}\log^{3/2}\parenthesis{e\beta dN}\widehat{\mathcal{R}}_{dN}\parenthesis{\Pi \circ \mcal{F}} + \frac{1}{\sqrt{N}}}}^{2} + \frac{d\beta^{2}}{N}\parenthesis{\log(1/\delta) + \log(\log N)}.
    \end{align*}
    Here, the worst-case empirical Rademacher complexity is defined as
    \begin{align*}
    \widehat{\mcal{R}}_{n}(\Pi \circ \mcal{F}) \coloneqq \sup_{\curly{(z_{\ell}, i_{\ell})}_{\ell = 1}^{n}} \E_{\epsilon}\bracket{\sup_{f \in \mcal{F}}\frac{1}{n}\sum_{\ell = 1}^{n}\epsilon_{\ell}f^{i_{\ell}}(z_{\ell})}, 
    \end{align*}
    where the expectation is taken over  independent Rademacher random variables $\epsilon = (\epsilon_\ell)_{\ell=1}^n$ conditioned on all the given samples $ (z_{\ell}, i_{\ell})\in (\R^{d} \times [T_{0}, T]) \times [d]  $.
    There exists a numerical constant $ C_{0} $  such that with probability at least $ 1 - \delta $, it holds for all $ f \in \mcal{F} $ simultaneously that
    \begin{align*}
    &\frac{1}{T - T_{0}}\int_{T_{0}}^{T}\int \norm{f(x, t)}_{2}^{2} {\rm d}P_{X_{t}}(x)dt \\ & \leq \frac{1}{N}\sum_{j = 1}^{N}\norm{f(X_{t_{j}}, t_{j})}_{2}^{2} + C_{0} \parenthesis{\sqrt{\frac{1}{N}\sum_{j = 1}^{N}\norm{f(X_{t_{j}}, t_{j})}_{2}^{2} \cdot \Gamma_{\delta}(\mcal{F})} + \Gamma_{\delta}(\mcal{F})}.
    \end{align*}
\end{Lemma}

Lemma \ref{lemma:local-Rad} is a result of \cite[Theorem 1]{reeve2020optimistic} by setting $ \mathcal{X} = \R^{d} \times [T_{0}, T] $, $ \mcal{V} = [-\beta, \beta]^{d} $ and $ \mcal{Y} = \curly{0} \subset \R^{d} $, and letting $ \mcal{L}(v, y) = \norm{v}_{2}^{2} \leq d\beta^{2} $. Note that the loss function $ \mcal{L}(\cdot) $ is $ (2\sqrt{d}, 1/2) $-self-bounding Lipschitz as defined in \cite{reeve2020optimistic} since for any $ u, v \in \mcal{V} $, it holds that
\begin{align*}
\abs{\norm{u}_{2}^{2} - \norm{v}_{2}^{2}} = \abs{\norm{u}_{2} - \norm{v}_{2}}\parenthesis{\norm{u}_{2} + \norm{v}_{2}} \leq 2\sqrt{d} \max\curly{\norm{u}_{2}^{2}, \norm{v}_{2}^{2}}^{1/2}\norm{u - v}_{\infty}.
\end{align*}
Now we are ready to prove Theorem \ref{thm:mismatch}.

\begin{proof}
    We first bound the performance of the two kernel predictors $f_\tau^K$ and $\tilde{f}_\tau^K$ on training dataset $S$. Let $u^K(\tau)$ and $\tilde{u}^{K}(\tau)$ be the prediction of $f_\tau^K$ and $\tilde{f}_\tau^K$ on the training input, respectively. We \underline{claim} the following result holds:
    \begin{align}
        \norm{u^{K}(\tau) - \tilde{u}^{K}(\tau)}_{2}^{2} \leq dN A^{2}(R_{\mcal{H}}, R). \label{eq:emp-mismatch}
    \end{align}
    The proof of \eqref{eq:emp-mismatch} is deferred to Appendix \ref{sec:mismatch}. To further derive a bound on the population loss, we will apply Lemma \ref{lemma:local-Rad} to the following functional class:
\begin{align*}
\mcal{F}_{\rho}^{R} \coloneqq \curly{(x, t) \mapsto f(x, t)\mathbb{I}\curly{\norm{x}_{2} \leq R}\vert (x, t) \in \R^{d} \times [T_{0}, T], f \in \mcal{H}, \norm{f}_{\mcal{H}} \leq \rho}.
\end{align*}
Given a dataset $ \curly{(z_{\ell}, i_{\ell})}_{\ell = 1}^{n} $ with $ z_{\ell} = (x_\ell, t_{\ell}) \in \R^d \times [T_0, T]$ and $i_\ell \in [d]$, we define an index set $ L = \curly{\ell: \norm{x_{\ell}}_{2} \leq R} $. 
The empirical Rademacher complexity of $ \mcal{F}_{\rho}^{R} $ can be bounded as
\begin{align}
    \widehat{\mcal{R}}_{n}(\Pi \circ \mcal{F}_{\rho}^{R})&  = \sup_{\curly{(z_{\ell}, i_{\ell})}_{\ell = 1}^{n}} \E_{\epsilon}\bracket{\sup_{\norm{f}_{\mcal{H}} \leq \rho}\frac{1}{n}\sum_{\ell = 1}^{n}\epsilon_{\ell}f^{i_{\ell}}(z_{\ell})\mathbb{I}\curly{\norm{x_{\ell}}_{2} \leq R}} \nonumber \\ & = \sup_{\curly{(z_{\ell}, i_{\ell})}_{\ell = 1}^{n}} \E_{\epsilon}\bracket{\sup_{\norm{f}_{\mcal{H}} \leq \rho}\frac{1}{n}\sum_{\ell \in L}\epsilon_{\ell}f^{i_{\ell}}(z_{\ell})} \nonumber \\ & = \sup_{\curly{(z_{\ell}, i_{\ell})}_{\ell = 1}^{n}} \E_{\epsilon}\bracket{\sup_{\norm{f}_{\mcal{H}} \leq \rho}\frac{1}{n}\sum_{\ell \in L}\epsilon_{\ell}f(z_{\ell})^{\top}{\bf e}_{i_{\ell}}} \nonumber \\ & = \sup_{\curly{(z_{\ell}, i_{\ell})}_{\ell = 1}^{n}} \E_{\epsilon}\bracket{\sup_{\norm{f}_{\mcal{H}} \leq \rho}\frac{1}{n}\sum_{\ell \in L}\epsilon_{\ell}\inpro{f, K(\cdot, z_{\ell}){\bf e}_{i_{\ell}}}_{\mcal{H}}} \label{eq:reproducing}
    \end{align}
Here, Eq.~(\ref{eq:reproducing}) holds due to the reproducing property:
\begin{align*}
\inpro{f, K(\cdot, z)c} = f(z)^{\top}c, \quad \forall  f \in \mcal{H} , c \in \R^{d}.
\end{align*}
By linearity of inner product, we further deduce
    \begin{align}
     \widehat{\mcal{R}}_{n}(\Pi \circ \mcal{F}_{\rho}^{R})& = \sup_{\curly{(z_{\ell}, i_{\ell})}_{\ell = 1}^{n}} \frac{1}{n}\E_{\epsilon}\bracket{\sup_{\norm{f}_{\mcal{H}} \leq \rho}\inpro{f, \sum_{\ell \in L}\epsilon_{\ell}K(\cdot, z_{\ell}){\bf e}_{i_{\ell}}}_{\mcal{H}}} \nonumber \\ & =  \sup_{\curly{(z_{\ell}, i_{\ell})}_{\ell = 1}^{n}} \frac{1}{n}\E_{\epsilon}\bracket{\inpro{\rho \frac{ \sum_{\ell \in L}\epsilon_{\ell}K(\cdot, z_{\ell}){\bf e}_{i_{\ell}}}{\norm{ \sum_{\ell \in L}\epsilon_{\ell}K(\cdot, z_{\ell}){\bf e}_{i_{\ell}}}_{\mcal{H}}}, \sum_{\ell \in L}\epsilon_{\ell}K(\cdot, z_{\ell}){\bf e}_{i_{\ell}}}_{\mcal{H}}} \label{eq:Cauchy}\\ & = \sup_{\curly{(z_{\ell}, i_{\ell})}_{\ell = 1}^{n}} \frac{\rho}{n}\E_{\epsilon}\bracket{\norm{\sum_{\ell \in L}\epsilon_{\ell}K(\cdot, z_{\ell}){\bf e}_{i_{\ell}}}_{\mcal{H}}}, \nonumber
    \end{align}
where we utilize the equality condition of Cauchy-Schwarz inequality to obtain (\ref{eq:Cauchy}). Next, we apply the Jensen's inequality to have
    \begin{align}
     \widehat{\mcal{R}}_{n}(\Pi \circ \mcal{F}_{\rho}^{R})& = \sup_{\curly{(z_{\ell}, i_{\ell})}_{\ell = 1}^{n}} \frac{\rho}{n}\E_{\epsilon}\bracket{\sqrt{\norm{\sum_{\ell \in L}\epsilon_{\ell}K(\cdot, z_{\ell}){\bf e}_{i_{\ell}}}_{\mcal{H}}^{2}}} \nonumber \\ & \leq \sup_{\curly{(z_{\ell}, i_{\ell})}_{\ell = 1}^{n}} \frac{\rho}{n}\sqrt{\E_{\epsilon}\bracket{\norm{\sum_{\ell \in L}\epsilon_{\ell}K(\cdot, z_{\ell}){\bf e}_{i_{\ell}}}_{\mcal{H}}^{2}}} \nonumber 
     \\ &  = \sup_{\curly{(z_{\ell}, i_{\ell})}_{\ell = 1}^{n}} \frac{\rho}{n}\sqrt{\sum_{\ell \in L}\norm{K(\cdot, z_{\ell}){\bf e}_{i_{\ell}}}^{2}_{\mcal{H}}},
     \label{eq:exp-epsilon}
     \end{align}
where we apply the facts that $ \E\bracket{\epsilon_{\ell}\epsilon_{\ell'}} = 0 $ for $ \ell \neq \ell' $ and $ \E\bracket{\epsilon_{\ell}^{2}} = 1 $ to derive (\ref{eq:exp-epsilon}). Finally, we use the reproducing property again to have
     \begin{align}
     \widehat{\mcal{R}}_{n}(\Pi \circ \mcal{F}_{\rho}^{R}) & \leq \sup_{\curly{(z_{\ell}, i_{\ell})}_{\ell = 1}^{n}} \frac{\rho}{n}\sqrt{\sum_{\ell \in L}{\bf e}_{i_{\ell}}^{\top}K(z_{\ell}, z_{\ell}){\bf e}_{i_{\ell}}} \label{eq:diag}\\ & \leq \sup_{\abs{L}}\frac{\rho}{n}\sqrt{\abs{L}C_{\max}^{2}} \leq \frac{\rho C_{\max}}{\sqrt{n}}. \nonumber
\end{align}

We next calculate $\beta$ associated with the function class $\mathcal{F}_\rho^R$. Note that the reproducing property and the Cauchy-Schwarz inequality imply that
\begin{align*}
\beta & = \sup_{(x, t) \in \R^{d}\times [T_{0}, T]}\max_{1 \leq i \leq d}\abs{f^{i}(x, t)}\mathbb{I}\curly{\norm{x}_{2} \leq R} \\ & = \sup_{\norm{x}_{2} \leq R}\sup_{t \in [T_{0}, T]}\max_{1 \leq i \leq d}\abs{\inpro{f, K(\cdot, (x, t)){\bf e}_{i}}_{\mcal{H}}} \\ & \leq\sup_{\norm{x}_{2} \leq R}\sup_{t \in [T_{0}, T]}\norm{f}_{\mcal{H}}\max_{1 \leq i \leq d}\norm{K(\cdot, (x, t)){\bf e}_{i}}_{\mcal{H}}\\ & \leq \rho C_{\max}.
\end{align*}

It remains to find a $ \rho $ such that $ \norm{f_{\tau}^{K} - \tilde{f}_{\tau}^{K}}_{\mcal{H}} \leq \rho$. Note that 
\begin{align*}
\norm{f_{\tau}^{K} - \tilde{f}_{\tau}^{K}}_{\mcal{H}}^{2} & = \norm{\sum_{j = 1}^{N}K((X_{t_{j}}, t_{j}), \cdot)(\gamma_{j}(\tau) - \tilde{\gamma}_{j}(\tau))}_{\mcal{H}}^{2} \\ & = \sum_{j = 1}^{N}\sum_{\ell = 1}^{N}(\gamma_{j}(\tau) - \tilde{\gamma}_{j}(\tau))^{\top}K((X_{t_{j}}, t_{j}), (X_{t_{\ell}}, t_{\ell}))(\gamma_{j}(\tau) - \tilde{\gamma}_{j}(\tau)) \\ & = (\gamma(\tau) - \tilde{\gamma}(\tau))^{\top}H(\gamma(\tau) - \tilde{\gamma}(\tau)) \\ & = (u^K(\tau) - \tilde{u}^K(\tau))^{\top}H^{-1}(u^K(\tau) - \tilde{u}^K(\tau)).
\end{align*}
Note that update rule \eqref{eq:kernel-predictor} implies
\begin{align*}
u^K(\tau) - \tilde{u}^K(\tau) = (I_{dN} - (I_{dN} - \eta H)^{\tau})(y - \tilde{y}).
\end{align*}
Therefore, Assumption \ref{ass:gram-eigen} and \eqref{eq:approx-inf} lead to
\begin{align*}
    \norm{f_{\tau}^{K} - \tilde{f}_{\tau}^{K}}_{\mcal{H}} & = \norm{(I_{dN} - (I_{dN} - \eta H)^{\tau})(y - \tilde{y})}_{H^{-1}} \\ & \leq \norm{H^{-1}}_{2}\norm{I_{dN} - (I_{dN} - \eta H)^{\tau}}_{2}\norm{y - \tilde{y}}_{2}  \\ & \leq \frac{\norm{y - \tilde{y}}_{2}}{\lambda_{0}} \leq \frac{\sqrt{dN }A(R_{\mcal{H}}, R)}{\lambda_{0}} := \rho.
\end{align*}

Finally, we put all the results together and apply Lemma \ref{lemma:local-Rad} to conclude that with probability $ 1 - \delta $ it holds that
\begin{align*}
    &\frac{1}{T - T_{0}}\int_{T_{0}}^{T}\int_{\norm{x}_{2} \leq R} \norm{f_{\tau}^{K}(x, t) - \tilde{f}_{\tau}^{K}(x,t)}_{2}^{2} {\rm d}P_{X_{t}}(x)dt \\ & \leq \frac{1}{N}\norm{u^{K}(\tau) - \tilde{u}^{K}(\tau)}_{2}^{2} + C_{0} \parenthesis{\sqrt{\frac{1}{N}\norm{u^{K}(\tau) - \tilde{u}^{K}(\tau)}_{2}^{2} \cdot \Gamma_{\delta}} + \Gamma_{\delta}} \\ & \leq dA^{2}(R_{\mcal{H}}, R) + C_{0}\parenthesis{\sqrt{dA^{2}(R_{\mcal{H}}, R)\Gamma_{\delta}} + \Gamma_{\delta}},
\end{align*}
in which we overload the notation with  
\begin{align*}
  \Gamma_{\delta}(\mcal{F}_{\rho}^{R}) & =  \parenthesis{2\sqrt{d}\parenthesis{\sqrt{d}\log^{3/2}\parenthesis{e\beta dN}\widehat{\mathcal{R}}_{dN}\parenthesis{\Pi \circ \mcal{F}} + \frac{1}{\sqrt{N}}}}^{2} + \frac{d\beta^{2}}{N}\parenthesis{\log(1/\delta) + \log(\log N)} \\ & \leq \parenthesis{2\sqrt{d}\parenthesis{\sqrt{d}\log^{3/2}\parenthesis{e\rho C_{\max} dN}\frac{\rho C_{\max}}{\sqrt{dN}} + \frac{1}{\sqrt{N}}}}^{2} + \frac{d\rho^{2}C_{\max}^{2}}{N}\parenthesis{\log(1/\delta) + \log(\log N)} \\ & = 4d\parenthesis{\frac{\sqrt{d}A(R_{\mcal{H}}, R)C_{\max}}{\lambda_{0}}\log^{3/2}\parenthesis{\frac{eC_{\max}(dN)^{3/2}A(R_{\mcal{H}}, R)}{\lambda_{0}}} + \frac{1}{\sqrt{N}}}^{2} \\ & \qquad + \frac{d^{2}A^{2}(R_{\mcal{H}}, R)C_{\max}^{2}}{\lambda_{0}^{2}}\parenthesis{\log(1/\delta) + \log\parenthesis{\log N}} \eqqcolon \Gamma_{\delta}.
\end{align*}
\end{proof}

\subsection{Early Stopping}\label{subsec:stop}

Recall in \eqref{eq:kernel-predictor}--\eqref{eq:virtual-dataset}, the virtual kernel predictor is parametrized as $ \tilde{f}_{\tau}^{K}(\cdot) = \sum_{j = 1}^{N}K((X_{t_{j}}, t_{j}), \cdot)\tilde{\gamma}(\tau) $, where the parameter $\tilde{\gamma}$ is updated as 
\begin{align}
\tilde{\gamma}(\tau + 1) = \tilde{\gamma}(\tau) - \eta (H\tilde{\gamma}(\tau) - \tilde{y}), \quad \tilde{\gamma}(0) = 0, \label{eq:update-tilde-gamma-recall}
\end{align}
with $ \tilde{y}_{j} = f_{\mcal{H}}(X_{t_{j}}, t_{j}) + \varepsilon_{j} $. 
This enables us to transform the score matching problem to a classical kernel regression problem. The next technical result allows us to reduce the excess risk bound for the early-stopped GD learning in RKHS to the excess risk bound for learning Lipschitz functions. To proceed,  let us denote the eigenvalues of the Gram matrix $H^{11} = [H^{11}_{j\ell}]$ as $\widehat{\lambda}_1 \geq \dots \geq \widehat{\lambda}_N$. %

\begin{Theorem}\label{thm:early-stop}
    Suppose Assumptions \ref{ass:bounded target} and \ref{ass:gram-eigen} hold. Fix any $ f_{\mcal{H}} \in \mcal{H} $ with $ \norm{f_{\mcal{H}}}_{\mcal{H}}^{2} \leq R_{\mcal{H}} $ and assume the labels are constructed with  $ \tilde{X}_{0, j} = f_{\mcal{H}}(X_{t_{j}}, t_{j}) + \varepsilon_{j}$. 
    Let $ \tilde{f}_{\mathcal{T}}^{K} $  be obtained in \eqref{eq:kernel-predictor} by a GD-trained kernel regression with the number of iterations $\mathcal{T}$,~where
    \begin{align}
    \mathcal{T} = \argmin\curly{\tau \in \N: \widehat{R}_{\kappa}(1/\sqrt{\eta\tau N}) > (2e D \eta\tau N)^{-1}} - 1, \;\; \text{with} \;\;
    \widehat{R}_{\kappa}(x) \coloneqq \sqrt{\frac{1}{N}\sum_{j = 1}^{N}\min\curly{\frac{\widehat{\lambda}_{j}}{N}, x^{2}}}. \label{eq:def-stop-rule}
    \end{align}
    Then, there exist positive constants $c_1$ and $c_2$ independent of $N$ such that
    \begin{align*}
        \frac{1}{T - T_{0}}\int_{T_{0}}^{T}\E\bracket{\norm{\tilde{f}^{K}_{\mathcal{T}}(X_{t}, t) - f_{\mcal{H}}(X_{t}, t)}_{2}^{2}\mathbb{I}\curly{\norm{X_{t}}_{2}\leq {R}}}{\rm d}t \lesssim dR_\mcal{H}C_{\max}^{4}N^{-\frac{d+1}{d+3}},
    \end{align*}
    with probability at least $ 1 - c_{1}\exp(-c_{2}N^{\frac{2}{d+3}}) $.
\end{Theorem}
Here, $\mathcal{T}$ is a data-dependent {\it early-stopping rule}  to control the excess risk of kernel regression. Intuitively, the empirical localized Rademacher complexity $\widehat{R}_{\kappa}(\cdot)$ captures the complexity of functions contained in a ball under $\norm{\cdot}_{\kappa}$-norm within the RKHS.
For supervised learning with noisy labels, an early-stopping rule for GD is necessary to minimize the excess risk \citep{hu2021regularization,bartlett2002rademacher,li2020gradient}. {We remark that, although $\mathcal{T}$ is defined through the kernel regression for analytical purposes, it can be directly implemented in the neural network training.} To prove Theorem \ref{thm:early-stop}, we first reduce the vector-valued learning problem to a scalar-valued learning problem by leveraging the diagonal property of the NTK and then we follow the strategy as in \cite{raskutti2014early}. 

\begin{proof}

We start with the update rule for each coordinate $ i \in [d] $. Note that
\begin{align*}
(H\tilde{\gamma}(\tau))_{j}^{i} = \parenthesis{\sum_{\ell = 1}^{N}H_{j\ell}\tilde{\gamma}_{\ell}(\tau)}^{i} = \sum_{\ell = 1}^{N}\sum_{k = 1}^{d}H_{j\ell}^{ik}\tilde{\gamma}_{\ell}^{k}(\tau) = \sum_{\ell = 1}^{N}H_{j\ell}^{ii}\tilde{\gamma}_{\ell}^{i}(\tau) = \parenthesis{H^{ii}\tilde{\gamma}^{i}(\tau)}_{j}.
\end{align*}
Hence, the update rule for $ \tilde{\gamma}^{i} $ in \eqref{eq:update-tilde-gamma-recall} becomes
\begin{align}\label{eq:update-gamma-i}
    \tilde{\gamma}^{i}(\tau+1) = \tilde{\gamma}^{i}(\tau) - \eta \parenthesis{H^{ii}\tilde{\gamma}^{i}(\tau) - \tilde{y}^{i}}, \quad \tilde{\gamma}^{i}(0) = 0.
\end{align}
Furthermore, since $ \tilde{u}^{K,i}(\tau) = (H\tilde{\gamma}(\tau))^{i} = H^{ii}\tilde{\gamma}^{i}(\tau) $, multiplying $  H^{ii} $ on both sides of \eqref{eq:update-gamma-i} results in
\begin{align}\label{eq:update-uK-i}
\tilde{u}^{K,i}(\tau+1) = \tilde{u}^{K, i}(\tau) - \eta H^{ii}(\tilde{u}^{K,i}(\tau) - \tilde{y}^{i}), \quad \tilde{u}^{K,i}(0) = 0.
\end{align}

For ease of exposition, we denote $ u_{\mcal{H}} = \parenthesis{(u_{\mcal{H}})_{j}^{i}}_{i, j = 1}^{d, N} \in \R^{dN} $,  where $ (u_{\mcal{H}})_{j}^{i} = f_{\mcal{H}}^{i}(X_{t_{j}}, t_{j}) $. In order to establish the upper bound in Theorem \ref{thm:early-stop}, we first provide a bound for the empirical risk $ \frac{1}{N}\norm{\tilde{u}^{K,i}(\tau) - u_{\mcal{H}}^{i}}_{2}^{2} $ at each iteration $ \tau $. Consider the eigen-decomposition  $ H^{ii} = V\Lambda V^{\top} $, where $ V $ and $ \Lambda $ are universal for all $ i $ as $ H^{ii} = H^{11} $. Let us further denote $ \zeta^{i}(\tau) = V^{\top}\tilde{u}^{K,i}(\tau) $, $ \zeta_{\mcal{H}}^{i} = V^{\top}u_{\mcal{H}}^{i} $ and $\tilde{\varepsilon}^{i} = V^{\top}\varepsilon^{i}$. It follows from \eqref{eq:update-uK-i} that
\begin{align}
\zeta^{i}(\tau+1) & = \zeta^{i}(\tau) - \eta\Lambda (\zeta^{i}(\tau) - V^{\top}\tilde{y}^{i}) \nonumber \\ & = \zeta^{i}(\tau) - \eta\Lambda (\zeta^{i}(\tau) - V^{\top}(u^{i}_{\mcal{H}} + \varepsilon^{i})) \nonumber \\ & = \zeta^{i}(\tau) - \eta\Lambda(\zeta^{i}(\tau) - \zeta^{i}_{\mcal{H}}) + \eta \Lambda \tilde{\varepsilon}^{i},  \label{eq: recursion-zeta}
\end{align}
Unrolling the recursion \eqref{eq: recursion-zeta} yields
\begin{align*}
\zeta^{i}(\tau) - \zeta^{i}_{\mcal{H}} & = S(\tau)\parenthesis{\zeta^{i}(0) - \zeta^{i}_{\mcal{H}}} + (I_{N} - S(\tau))\tilde{\varepsilon}^{i} = (I_{N} - S(\tau))\tilde{\varepsilon}^{i} - S(\tau)\zeta^{i}_{\mcal{H}}.
\end{align*}
Here, we denote $ S(\tau) = (I_{N} - \eta \Lambda)^{\tau} $. Since $ S(\tau) $ is a diagonal matrix, we obtain the following upper bound:
\begin{align*}
\norm{\zeta^{i}(\tau) - \zeta^{i}_{\mcal{H}}}_{2}^{2} & \leq 2\norm{S(\tau) \zeta_{\mcal{H}}^{i}}_{2}^{2} + 2\norm{(I_{N} - S(\tau))\tilde{\varepsilon}^{i}}_{2}^{2} \\ & = 2\sum_{j = 1}^{N}(S(\tau))_{jj}^{2}(\zeta^{i}_{\mcal{H}, j})^{2} + 2\sum_{j = 1}^{N}(1 - S_{jj}({\tau}))^{2}\parenthesis{\tilde{\varepsilon}_{j}^{i}}^{2}.
\end{align*}
Moreover, since $ V $ is orthonormal, we have
\begin{align}\label{eq:bisa-var-decomp}
\frac{1}{N}\norm{\tilde{u}^{K,i}(\tau) - u_{\mcal{H}}^{i}}_{2}^{2} \leq \underbrace{\frac{2}{N}\sum_{j = 1}^{N}(S(\tau))_{jj}^{2}(V^{\top}u_{\mcal{H}}^{i})_{j}^{2}}_{\text{Squared Bias } (B_{\tau}^{i})^{2}} + \underbrace{\frac{2}{N}\sum_{j = 1}^{N}(1 - S_{jj}(\tau))^{2}\parenthesis{V^{\top}\varepsilon^{i}}_{j}^{2}}_{\text{Variance }V_{\tau}^{i}}
\end{align}
We \underline{claim} that at each iteration $\tau = 1, 2, \dots$, the squared bias is upper bounded with
\begin{align}\label{eq:upper-B-1}
(B_{\tau}^{i})^{2} \leq  \frac{R_\mathcal{H}}{e\tau \eta N},
\end{align}
and moreover there is a constant $c' > 0$ such that for all $\tau = 1,\dots, \mcal{T}$, it holds with probability at least $ 1 - \exp(-c'N\widehat{r}_{N}^{2}/D^{2}) $ we have
\begin{align}\label{eq:upper-V-1}
V_{\tau}^{i} \leq \frac{5}{e^{2} \eta \tau N}.
\end{align}
We defer the proof of the claim to Lemma \ref{lemma:bias-variance} in Appendix \ref{sec: proof bias-var}. Now, assuming the event \eqref{eq:upper-V-1} happens,  along with \eqref{eq:upper-B-1}, the empirical risk in \eqref{eq:bisa-var-decomp} can be further upper bounded by
\begin{align}
\frac{1}{N}\norm{\tilde{u}^{K, i}(\tau) - u_{\mcal{H}}^{i}}_{2}^{2}   \leq (B_{\tau}^{i})^{2} + V_{\tau}^{i}  \leq \frac{R_{\mcal{H}}}{e\eta\tau N} + \frac{5}{e^{2}\eta \tau N} \leq \frac{R_\mcal{H} + 5}{e\eta \tau N}. \label{eq:piece-1}
\end{align}
By the definition of $ \mcal{T} $ in \eqref{eq:def-stop-rule}, we must have $ \frac{1}{\eta\parenthesis{\mcal{T}+1}N} \leq \widehat{r}_{N}^{2} \leq \frac{1}{\eta{\mcal{T}}N} $ and 
\begin{align}
\frac{1}{\eta{\mcal{T}}N} \leq \frac{2}{\eta\parenthesis{\mcal{T}+1}N} \leq 2\widehat{r}_{N}^{2}. \label{eq:piece-2}
\end{align}
Combining \eqref{eq:piece-1} and \eqref{eq:piece-2},  we  have
\begin{align}
    \frac{1}{N}\norm{\tilde{u}^{K, i}(\mcal{T}) - u_{\mcal{H}}^{i}}_{2}^{2} \leq \frac{2\widehat{r}_{N}^{2}}{e}(R_\mcal{H} + 5) \leq {\widehat{r}_{N}^{2}}(R_\mcal{H} + 5). \label{eq:emp-r}
\end{align}

Recall that $ \curly{\mu_{p}}_{p = 1}^{\infty} $ are the eigenvalues of the scalar-valued kernel $ \kappa $.  Analogously to the localized empirical Rademacher complexity, we define the localized population Rademacher complexity by
\begin{align*}
R_{\kappa}(x) \coloneqq \sqrt{\frac{1}{N}\sum_{p = 1}^{\infty}\min\curly{\mu_{p}, x^{2}}}.  
\end{align*}
Let $B \coloneqq \sqrt{3 + 2R_\mathcal{H}} + \sqrt{R_\mathcal{H}}$ and consider the function class 
\begin{align*}
\mcal{F} = \curly{z \mapsto \frac{f(z) - h(z)}{BC^2_{\max}} \mathbb{I}\curly{\norm{z}_{2} \leq C_{\max}} : f, h \in \mcal{H}_{1}, \norm{f - h}_{\kappa} \leq B}.
\end{align*}
It is straightforward to see that $ \mcal{F} $ is star-shaped and $ 1 $-uniformly bounded. Moreover, with probability at least $1 - \exp(-cN\widehat{r}_N^2)$, the following inequality holds:
\begin{align}
\norm{(\tilde{f}_{\tau}^{K, i}(z) - f_{\mcal{H}}^{i}(z)) \mathbb{I}\curly{\norm{z}_2 \leq C_{\max}}}_{\infty} & \leq BC_{\max}^{2}. \label{eq:F-unif-1}
\end{align}
The proof of \eqref{eq:F-unif-1} is deferred to Lemma \ref{lemma:in-F} in Appendix \ref{sec: proof bias-var}. Conditioned on the event that  \eqref{eq:F-unif-1} holds, we apply \cite[Theorem 14.1, Corollary 14.5]{wainwright2019high} to $\mathcal{F}$ with $ r = r_{N}^{*} $ to deduce
\begin{align}
 \frac{1}{T - T_{0}}\int_{T_{0}}^{T}\int_{\norm{x}_{2} \leq R}\norm{\tilde{f}_{\mcal{T}}^{K, i}(x, t) - f_{\mcal{H}}^{i}(x, t)}_{2}^{2}{\rm d}P_{X_{t}}(x){\rm d}t \nonumber & \leq \frac{2}{N}\sum_{j = 1}^{N}\norm{\tilde{u}^{K, i}(\mcal{T}) - u_{\mcal{H}}^{i}}_{2}^{2} + B^2C_{\max}^{4}(r^{*}_{N})^{2} \nonumber \\ & \leq {\widehat{r}_{N}^{2}}(2R_\mcal{H} + 10) + B^2C_{\max}^{4}(r^*_N)^2 \nonumber \\ & \leq {\widehat{r}_{N}^{2}}(2R_\mcal{H} + 10) + 6(R_\mathcal{H} + 1)C_{\max}^{4}(r^*_N)^2, \label{eq:early-stop-bound-1}
\end{align}
where we recall \eqref{eq:emp-r}  to obtain the second inequality.
Furthermore, there exist two positive numerical constants $c_1$ and $c_2$ such that $\P(\frac{1}{4}r_N^* \leq \widehat{r}_N \leq r_N^*) \geq 1 - c_1 \exp(-c_2 N (r_N^*)^2)$ as shown in \cite[Lemma 11]{raskutti2014early}.
We combine this fact with \eqref{eq:early-stop-bound-1} to deduce that that with probability at least $1 - \tilde{c}_1 \exp(-\tilde{c}_2 N (r_N^*)^2)$ it holds 
\begin{align*}
\frac{1}{T - T_{0}}\int_{T_{0}}^{T}\int_{\norm{x}_{2} \leq R}\norm{\tilde{f}_{\mcal{T}}^{K, i}(x, t) - f_{\mcal{H}}^{i}(x, t)}_{2}^{2}{\rm d}P_{X_{t}}(x){\rm d}t & \leq \tilde{c}_3(R_\mathcal{H}+1)C_{\max}^{4}(r^*_N)^2,
\end{align*}
for some constants $\tilde{c}_1, \tilde{c}_2$ and $\tilde{c}_3$. 

Finally, over the domain $ \mathbb{B}_{R}^{2} \times [T_{0}, T] $, the eigenvalues of NTK enjoy the decay pattern $ \mu_{p} = \Theta(p^{-\frac{d+1}{2}}) $ \citep{bach2017breaking,scetbon2021spectral}. Following the steps in \citep[Corollary 3]{raskutti2014early} to prove an upper bound of the empirical Rademacher complexity, we obtain that $ (r^{*}_N)^{2} = \Theta(N^{-\frac{d+1}{d+3}}) $. Consequently, the following holds with probability at least $ 1 - \tilde{c}_{1}\exp(-\tilde{c}_{2}N^{\frac{2}{d+3}}) $:
\begin{align*}
    \frac{1}{T - T_{0}}\int_{T_{0}}^{T}\int_{\norm{x}_{2} \leq R}\norm{\tilde{f}_{\mcal{T}}^{K}(x, t) - f_{\mcal{H}}(x, t)}_{2}^{2}{\rm d}P_{X_{t}}(x){\rm d}t & \lesssim dR_\mcal{H}C_{\max}^{4}N^{-\frac{d+1}{d+3}}.
\end{align*}
This completes the proof. 
\end{proof}

\subsection{Score Estimation and Generalization}
In this subsection, we establish the score estimation and generalization result as follows. 
\begin{Theorem}[Score Estimation and Generalization]\label{thm:score-estimation}
    Suppose Assumptions \ref{ass:bounded target}, \ref{ass:g}, \ref{ass:lip-f*}, \ref{ass:gram-eigen} hold and we set $ m \gtrsim {\rm poly}(d)N^{26 + \frac{4}{d+3}} $ and $ \eta \lesssim N^{-2} $. Moreover, we use early stopping rule $ \mathcal{T} $ as in Theorem \ref{thm:early-stop} and set $T_0 = o(1)$ and $T = \mcal{O}(\log N)$. Then with probability at least $ 1 - N^{-1} - c_1 \exp(-c_2N^{\frac{2}{d+3}}) $ for constants $c_1, c_2 > 0$ as in Theorem \ref{thm:early-stop}, it holds that  
    \begin{align*}
    & \frac{1}{T - T_{0}}\int_{T_{0}}^{T}\lambda(t)\E\bracket{\norm{s_{{\bf W}(\mathcal{T})}(X_{t}, t) - \nabla \log p_{t}(X_{t})}_{2}^{2}}{\rm d}t = \tilde{\mcal{O}}\parenthesis{N^{-\frac{2}{d+3}}}.
    \end{align*}
    Here, $\tilde{\mcal{O}}(\cdot)$ hides polynomial dependency in $d$ and logarithmic terms in $N$. 
\end{Theorem}

Theorem \ref{thm:score-estimation} shows that the early-stopped neural network $ s_{{\bf W}(\mathcal{T})} $ learns the score function $ \nabla \log p_{t} $ well in the $ L^{2} $ sense over the interval $ [T_{0}, T] $. We remark that we choose a larger value of $m$ than in the existing literature in order to compensate for the unbounded nature of the input domain. To the best of our knowledge, this is the {\it first} algorithm-based analysis for score estimation with neural network parameterization. Combined with recent findings in the distribution recovery property of diffusion models, we are the first to obtain an end-to-end guarantee with a provably efficient algorithm for diffusion models.

We note that the polynomial dependence of the network width $m$ on the sample size $N$ appears pessimistic at first glance. This is primarily due to the unbounded nature of the input domain. In practice, however, our numerical experiments demonstrate that the proposed GD method performs well with much smaller network widths, suggesting that the theoretical requirement on $m$ is not tight. Improving the dependence of $m$ on $N$ remains an interesting open problem and is left for future work.

\begin{proof}
    We begin with simplifying the label mismatch error. Theorem \ref{thm:mismatch} implies that with probability $ 1 - \delta $ it holds
\begin{align*}
    \frac{1}{T - T_{0}}\int_{T_{0}}^{T}\int_{\norm{x}_{2} \leq R} \norm{f_{\tau}^{K}(x, t) - \tilde{f}_{\tau}^{K}(x,t)}_{2}^{2} {\rm d}P_{X_{t}}(x)dt  \leq dA^{2}(R_{\mcal{H}}, R) + C_{0}\parenthesis{\sqrt{dA^{2}(R_{\mcal{H}}, R)\Gamma_{\delta}} + \Gamma_{\delta}}, 
\end{align*}
where we recall the definition
\begin{align*}
\Gamma_{\delta} & = 4d\parenthesis{\frac{\sqrt{d}A(R_{\mcal{H}}, R)C_{\max}}{\lambda_{0}}\log^{3/2}\parenthesis{\frac{eC_{\max}(dN)^{3/2}A(R_{\mcal{H}}, R)}{\lambda_{0}}} + \frac{1}{\sqrt{N}}}^{2} \\ & \qquad + \frac{d^{2}A^{2}(R_{\mcal{H}}, R)C_{\max}^{2}}{\lambda_{0}^{2}}\parenthesis{\log(1/\delta) + \log\parenthesis{\log N}}.
\end{align*}
To proceed, we choose $ \delta = 1/N $ to obtain
\begin{align*}
\Gamma_{\delta} & \leq 8d\parenthesis{\frac{dA^{2}(R_{\mcal{H}}, R)C_{\max}^{2}}{\lambda_{0}^{2}}\log^{3}\parenthesis{\frac{eC_{\max}(dN)^{3/2}A(R_{\mcal{H}}, R)}{\lambda_{0}}} + \frac{1}{N}} \\ & \quad + \frac{d^{2}A^{2}(R_{\mcal{H}}, R)C_{\max}^{2}}{\lambda_{0}^{2}}\parenthesis{\log(N) + \log\parenthesis{\log N}}.
\end{align*}
We will choose $R_{\mcal{H}}$ and $R$ such that both $C_{\max}$ and $A(R_{\mcal{H}}, R)$ are both polynomial in $N$. 
Then, with probability at least $ 1 - 1/N $ it holds that
\begin{align*}
    \frac{1}{T - T_{0}}\int_{T_{0}}^{T}\int_{\norm{x}_{2} \leq R} \norm{f_{\tau}^{K}(x, t) - \tilde{f}_{\tau}^{K}(x,t)}_{2}^{2} {\rm d}P_{X_{t}}(x)dt  = \tilde{\mcal{O}}\parenthesis{\frac{A^{2}(R_{\mcal{H}}, R)C_{\max}^{2}}{\lambda_{0}^{2}}}.
\end{align*}
Together with Theorem \ref{thm:approx-score-L2} and Theorem \ref{thm:early-stop}, we have
\begin{align*}
    \frac{1}{T - T_{0}}\int_{T_{0}}^{T}\int_{\norm{x}_{2} \leq R} \norm{f_{\tau}^{K}(x, t) - f_{*}(x,t)}_{2}^{2} {\rm d}P_{X_{t}}(x)dt  & \lesssim \frac{A^{2}(R_{\mcal{H}}, R)C_{\max}^{2}}{\lambda_{0}^{2}} + A^{2}(R_{\mcal{H}}, R) + R_{\mcal{H}}C_{\max}^{4}N^{-\frac{d+1}{d+3}},
\end{align*}
where we recall $A(R_{\mcal{H}}, R) =  c_{1}\Lambda(R)\parenthesis{\frac{\sqrt{R_{\mcal{H}}}}{\Lambda(R)}}^{-\frac{2}{d-1}}\log\parenthesis{\frac{\sqrt{R_{\mcal{H}}}}{\Lambda(R)}}$ and $ \Lambda(R) = \mathcal{O}(R^{2}) $. Here, the $\lesssim$ notation hides a constant that is polynomial in $d$ and polylog in $N$.

To minimize $ R_{\mcal{H}} \mapsto x A^{2}(R, R_{\mcal{H}}) + y R_{\mcal{H}} $ with $ x = {C_{\max}^{2}}\lambda_{0}^{-2} + 1 $ and $ y = C_{\max}^{4}N^{-\frac{d+1}{d+3}} $, we choose $ R_{\mcal{H}} = \Lambda^{2}(R)(x/y)^{\frac{d-1}{d+1}} $ to have
\begin{align*}
    \frac{1}{T - T_{0}}\int_{T_{0}}^{T}\int_{\norm{x}_{2} \leq R} \norm{f_{\tau}^{K}(x, t) - f_{*}(x,t)}_{2}^{2} {\rm d}P_{X_{t}}(x)dt = \tilde{\mcal{O}}\parenthesis{N^{-\frac{2}{d+3}}{\rm poly}(C_{\max}, \lambda_{0}^{-1})}. 
\end{align*}
By Lemma \ref{lemma:H-Hii-eigen} and \cite[Proposition 2]{kuzborskij2022learning}, we know that $ \lambda_{0} \gtrsim {\rm polylog}(N, d)d $ with probability at least $ 1 - N^2 e^{-\mathcal{O}(\sqrt{d})}$. It remains to choose $ R $ and $ \Delta $ such that $ C_{\max} $ is logarithmic in $ N $. To achieve this, we choose $ R = \Theta(\log N^{\frac{8}{d+3}})^{1/2} $, $ \Delta = \Theta(N^{-2}) $, $T = \mathcal{O}(\log N)$ and $T_0 = o(1)$ as $N \to \infty$. Consequently, the following equality holds with probability at least 
\begin{align*}
    \frac{1}{T - T_{0}}\int_{T_{0}}^{T}\int_{\norm{x}_{2} \leq R} \norm{f_{\tau}^{K}(x, t) - f_{*}(x,t)}_{2}^{2} {\rm d}P_{X_{t}}(x)dt = \tilde{\mcal{O}}\parenthesis{N^{-\frac{2}{d+3}}}.
\end{align*}
Note that under the choice of $ R $, the RHS in Lemma \ref{lemma:tail bound} decays to 0 as $ N \to \infty $ with polynomial rate $\tilde{\mathcal{O}}(N^{-\frac{2}{d+3}})$. By choosing the width $ m \gtrsim {\rm poly}(d)N^{26 + \frac{4}{d+3}} $ and applying Theorem \ref{thm:coupling}, we complete the proof.  
\end{proof}
}

\section{Numerical Experiments}\label{sec:exp}
In this section, we evaluate the performance of the GD-trained score estimator (proposed in Section~\ref{sec:results}) with numerical experiments.
In our experiments, we utilized the Credit Default dataset, as detailed by \cite{yeh2009comparisons} in their comparative study. The dataset contains bill statements of
credit card customers, their default payments, history of payment as well as the demographic factors of the customers in Taiwan from April 2005 to September 2005.  This dataset consists of 30,000 instances, each characterized by 10 categorical and 13 numerical attributes. The data is divided into two distinct classes. We employ the embedding method proposed by \cite{sattarov2023findiff} to pre-process the categorical attributes. The goal of our experiments is to generate synthetic but realistic financial instances that follow the same distribution as those in the Credit Default dataset.

\begin{wrapfigure}{r}{0.48\textwidth}
\vspace{-10pt}
\centering
\includegraphics[width=0.47\textwidth]{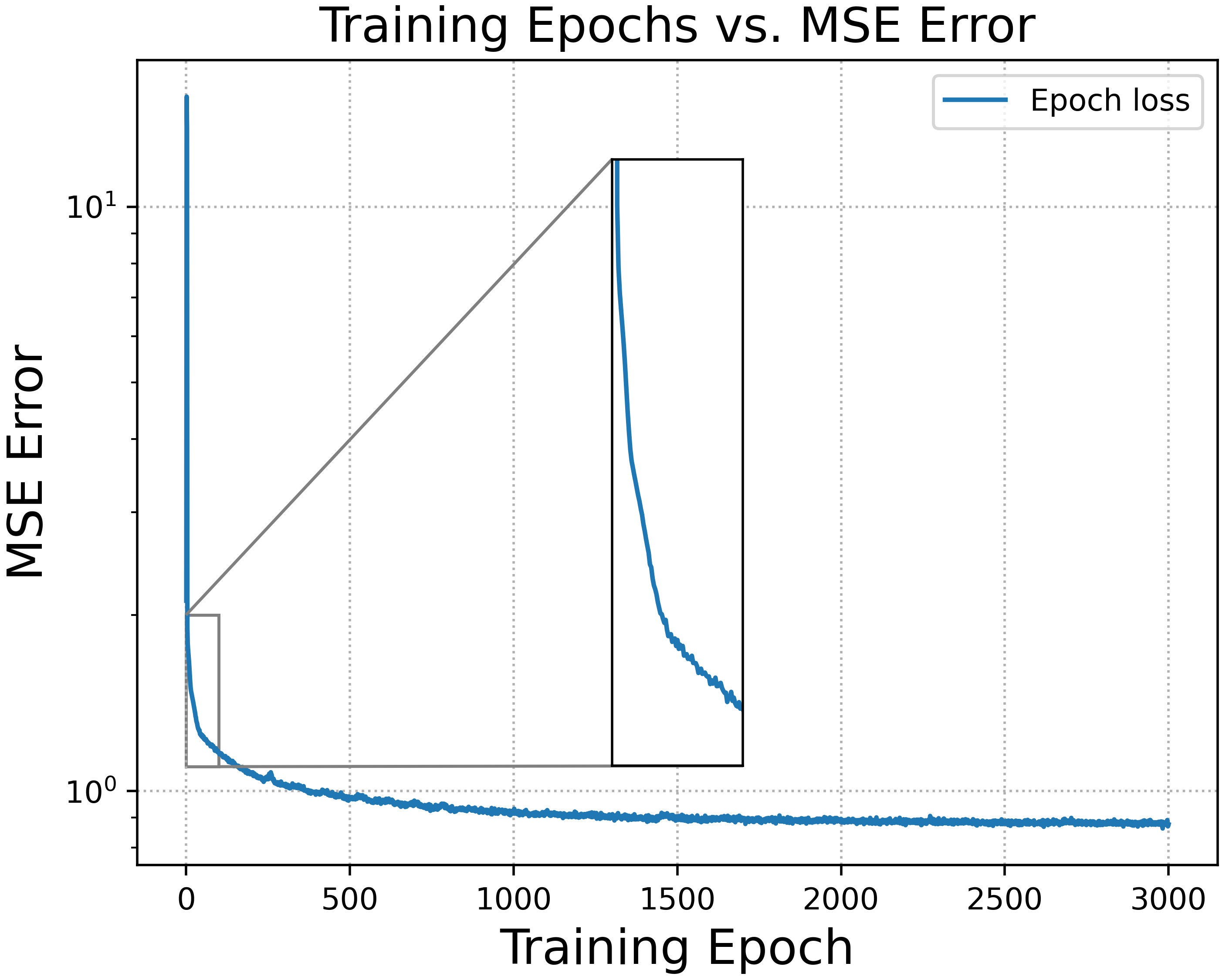}
\caption{\bf Convergence of gradient descent.}
\label{fig:GD_loss}
\vspace{-10pt}
\end{wrapfigure}

\paragraph{Model and Hyperparameter Setup.}
To implement the forward and backward processes, we consider the DDPM discretization scheme \citep{ho2020denoising} with constant variance. In particular, the number of diffusion steps is set to be 500. Moreover, the noise level is linearly increased from $10^{-4}$ to $0.02$, consistent with practical setup \citep{ho2020denoising}.

We utilize a two-hidden layer fully-connected neural network to learn the score function described in \eqref{eq:regression}. To be consistent with our theoretical set-up, we determine the number of neurons in each hidden layer to be 3000. The model parameters are initialized using Kaiming initialization \citep{he2015delving}. The neural network is trained for 3000 epochs using GD with a constant learning rate of $0.05$.

\paragraph{Evaluation and Discussion.} 

We first examine the convergence of the GD algorithm in the training of neural networks. As shown in Figure \ref{fig:GD_loss}, the GD algorithm converges within approximately 500 epochs despite the non-convex landscape of the training objective.  This observation is consistent with our theoretical analysis  ({\it cf.} Theorem \ref{thm:gd-conv}).

Next, we evaluate the quality of the data generated by the diffusion model using the neural score estimator trained by GD. In particular, we focus on the fidelity to quantify the relevance of the generated data. We adopt the \textit{column fidelity} and \textit{row fidelity} metrics introduced in \cite[Section 4.3]{sattarov2023findiff}, which evaluate the simulated samples feature-by-feature and instance-by-instance, respectively.  Table \ref{tab:fidelity} shows the fidelity results of the GD algorithm. For comparison, we also train a model using Adam \citep{kingma2014adam} for 3000 epochs with a mini-batch of size 512 and a cosine
learning rate scheduler \citep{loshchilov2016sgdr}. {In practice, this is viewed as a popular scheme for training diffusion models, achieving state-of-the-art performance.} We observe that the GD-trained score estimator's performance is comparable to that trained with Adam. The results on fidelity are visualized in Figure \ref{fig:fid}.

In summary, the training scheme inspired by our theoretical framework achieves promising performance comparable
to the complex training procedure involving Adam with large batch sizes and heuristic learning rate scheduling. This further suggests that theory sheds light on simple and effective implementations.

\begin{table}[h!]
\centering
\begin{tabular}{@{}lll@{}}
\toprule
& Column Fidelity & Row Fidelity\\ \midrule GD & 0.82               & 0.63               \\ \midrule Adam & 0.87           & 0.75                
                 \\ \bottomrule
\end{tabular}
\caption{Fidelity assessment of GD and Adam.}
\label{tab:fidelity} 
\end{table}

\begin{figure}[t]
\centering
\begin{subfigure}{0.68\textwidth}
    \centering
    \includegraphics[width=\linewidth]{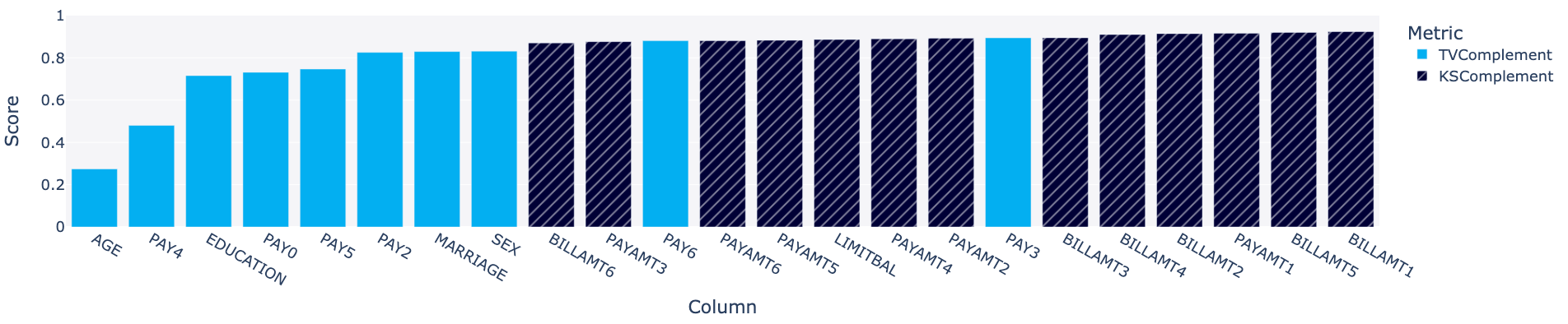}
\end{subfigure}
\hfill
\begin{subfigure}{0.28\textwidth}
    \centering
    \includegraphics[width=\linewidth]{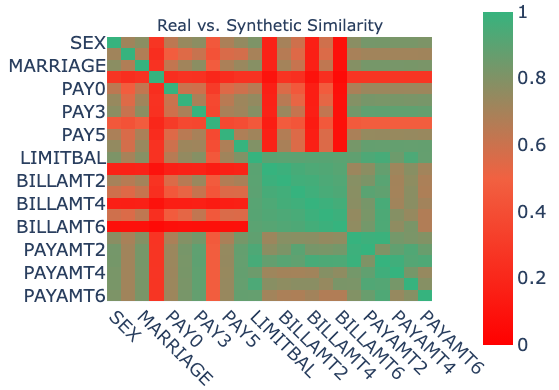}
\end{subfigure}
\caption{Left:~Column fidelity. Right:~Row fidelity.}
\label{fig:fid}
\end{figure}

\section{Conclusion and Discussions}
This work establishes the {\it first} algorithm-dependent analysis of neural network-based score estimation in diffusion models. We demonstrate that training overparametrized neural networks with GD can learn the ground-truth score function with a sufficient number of samples under an early-stopping rule. Our work investigates all three aspects of the score estimation task:~approximation, optimization, and generalization. The analytical framework laid out in this work sheds light on the understanding of diffusion models and inspires innovative architectural design. Our numerical experiments show that the proposed algorithm is comparable to the state-of-the-art training schemes for diffusion models.

In addition, our work leaves several open questions for future investigation. For instance, the dependence of our convergence results on the dimension of the problem seems sub-optimal. To address this issue, one approach is to consider the manifold structure of the data distribution, such as the linear subspace assumption as suggested by \cite{chen2023score} and \cite{oko2023diffusion}. Another direction is to understand the role of neural network architectures like U-nets and transformers in the implementation of diffusion models for machine learning tasks. Finally, the analysis of stochastic and adaptive algorithms such as SGD and Adam is crucial, closing the gap between theory and practice further.

\bibliography{arxiv/arxiv}
\bibliographystyle{plain}

\newpage
\appendix

\begin{center}
    {\Huge  \textbf{\paperappendixtitle}}
\end{center}

\section{Auxiliary Results for Theorem \ref{thm:approx-score-L2}}\label{sec:pf of lip-t}
This appendix is devoted to the proofs of an auxiliary lemma used in the proof for Theorem \ref{thm:approx-score-L2}.

\begin{Lemma}\label{lemma:lip-t}
    Suppose Assumptions \ref{ass:bounded target} and \ref{ass:g} hold.
    For each $ R > 0 $, the regression function $ f_{*}(x, t) $ is $ \beta_2 (R) $-Lipschitz in $ t $ for all $ \norm{x}_{\infty} \leq R$ and $ t \in [T_{0}, T] $, i.e., $ \abs{f_{*}(x, t) - f_{*}(x, t')}_{2} \leq \beta_2(R)\abs{t - t'} $, where $ \beta_2(R) = \mathcal{O}(R) $. 
    \end{Lemma} 

\begin{proof}
We start with computing the derivative of $f_*$ with respect to $t$. The dominated convergence theorem implies
\begin{align}
\frac{\partial}{\partial t}f_{*}(x, t) & = \frac{\partial}{\partial t}\int x_{0}p_{0\vert t}(x_{0} \vert x){\rm d}x_{0} \nonumber\\ & = \frac{\partial}{\partial t}\int \frac{x_{0}p_{t \vert 0}(x\vert x_{0})p_{0}(x_{0})}{\int p_{t \vert 0}(x \vert x_{0}')p_{0}(x_{0}'){\rm d}x_{0}'}{\rm d}x_{0} \nonumber \\ & = \int \frac{x_{0}\frac{\partial}{\partial t}p_{t\vert 0}(x \vert x_{0})p_{0}(x_{0})}{\int p_{t \vert 0}(x \vert x_{0}')p_{0}(x_{0}'){\rm d}x_{0}'}{\rm d}x_{0} \nonumber \\ & \qquad - \int \frac{x_{0}p_{t\vert 0}(x \vert x_{0})p_{0}(x_{0})\int \frac{\partial}{\partial t}p_{t\vert 0}(x \vert x_{0}'')p_{0}(x_{0}''){\rm d}x_{0}''}{\parenthesis{\int p_{t \vert 0}(x \vert x_{0}')p_{0}(x_{0}'){\rm d}x_{0}'}^{2}}{\rm d}x_{0}. \label{eq:dev-f*-t}
\end{align}
To proceed, recall that $ X_{t} \vert X_{0} \sim \mcal{N}(\alpha(t)X_{0}, h(t)I_{d}) $ with $ \alpha(t) = \exp\parenthesis{-\int_{0}^{t}\frac{g(s)}{2} {\rm d}s} $ and $ h(t) = 1 - \alpha^{2}(t) $. By straightforward calculations,
\begin{align}
& \frac{\partial }{\partial t}p_{t\vert 0}(x \vert x_{0}) \nonumber\\ & = \frac{\partial }{\partial t} \parenthesis{(2\pi h(t))^{-d/2}\exp\parenthesis{-\frac{\norm{x - \alpha(t)x_{0}}_{2}^{2}}{2h(t)}}} \nonumber\\ & = -\frac{d}{2}(2\pi h(t))^{-\frac{d}{2}-1}(2\pi)h'(t)\exp\parenthesis{-\frac{\norm{x - \alpha(t)x_{0}}_{2}^{2}}{2h(t)}} \nonumber \\ & \qquad + (2\pi h(t))^{-d/2}\exp\parenthesis{-\frac{\norm{x - \alpha(t)x_{0}}_{2}^{2}}{2h(t)}}\parenthesis{\frac{2(x - \alpha(t)x_{0})^{\top}x_{0}\alpha'(t)}{2h(t)} + \frac{\norm{x - \alpha(t)x_{0}}_{2}^{2}h'(t)}{2h^{2}(t)}} \nonumber\\ & = \frac{p_{t\vert 0}(x \vert x_{0})}{2h^{2}(t)}\parenthesis{-dh(t)h'(t) + 2(x - \alpha(t)x_{0})^{\top}x_{0}\alpha'(t)h(t) + \norm{x - \alpha(t)x_{0}}_{2}^{2}h'(t)}. \label{eq:dev-pt0-1}
\end{align} 
Since $ \alpha'(t) = -\alpha(t)g(t)/2 $ and $ h'(t) = -2\alpha(t)\alpha'(t) = \alpha^{2}(t)g(t) $, we can rewrite (\ref{eq:dev-pt0-1}) as
\begin{align}
    & \frac{\partial }{\partial t}p_{t\vert 0}(x \vert x_{0}) \nonumber \\ & = \frac{p_{t\vert 0}(x \vert x_{0})}{2h^{2}(t)}\parenthesis{-dh(t)\alpha^{2}(t)g(t) - (x - \alpha(t)x_{0})^{\top}x_{0}\alpha(t)g(t)h(t) + \norm{x - \alpha(t)x_{0}}_{2}^{2}\alpha^{2}(t)g(t)} \nonumber \\ & = p_{t\vert 0}(x \vert x_{0})\frac{\alpha(t)g(t)}{2h^{2}(t)}\parenthesis{-dh(t)\alpha(t) - (x - \alpha(t)x_{0})^{\top}x_{0}h(t) + \norm{x - \alpha(t)x_{0}}_{2}^{2}\alpha(t)} \nonumber \\ & = p_{t\vert 0}(x \vert x_{0})\frac{\alpha(t)g(t)}{2h^{2}(t)}\parenthesis{-dh(t)\alpha(t) + \alpha(t)\norm{x}_{2}^{2} - (1 +\alpha^{2}(t))x_{0}^{\top}x + \alpha(t)\norm{x_{0}}_{2}^{2}}.
    \label{eq:dev-pt0-2}
\end{align}
Plugging (\ref{eq:dev-pt0-2}) back into (\ref{eq:dev-f*-t}), we  have
\begin{align*}
    & \int \frac{x_{0}\frac{\partial}{\partial t}p_{t\vert 0}(x \vert x_{0})p_{0}(x_{0})}{\int p_{t \vert 0}(x \vert x_{0}')p_{0}(x_{0}'){\rm d}x_{0}'}{\rm d}x_{0} \\ & = \frac{\alpha(t)g(t)}{2h^{2}(t)} \E\bracket{X_{0}\parenthesis{-dh(t)\alpha(t) + \alpha(t)\norm{X_{t}}_{2}^{2} - (1 +\alpha^{2}(t))X_{0}^{\top}X_{t} + \alpha(t)\norm{X_{0}}_{2}^{2}}\bigg\vert X_{t} = x} \\ & = \frac{\alpha(t)g(t)}{2h^{2}(t)}\bigg(-dh(t)\alpha(t)\E\bracket{X_{0}\vert X_{t} = x} + \alpha(t)\norm{x}_{2}^{2}\E\bracket{X_{0}\vert X_{t} = x} \\ & \qquad\qquad\qquad\qquad - (1 + \alpha^{2}(t))x\E\bracket{\norm{X_{0}}^{2}_{2}\big\vert X_{t} = x}+ \alpha(t)\E\bracket{X_{0}\norm{X_{0}}_{2}^{2}\vert X_{t} = x}\bigg),
    \end{align*}
    and also
    \begin{align*}
    & \int \frac{x_{0}p_{t\vert 0}(x \vert x_{0})p_{0}(x_{0})\int \frac{\partial}{\partial t}p_{t\vert 0}(x \vert x_{0}'')p_{0}(x_{0}''){\rm d}x_{0}''}{\parenthesis{\int p_{t \vert 0}(x \vert x_{0}')p_{0}(x_{0}'){\rm d}x_{0}'}^{2}}{\rm d}x_{0} \\ & = \frac{\alpha(t)g(t)}{2h^{2}(t)}\E\bracket{X_{0}\E\bracket{-dh(t)\alpha(t) + \alpha(t)\norm{X_{t}}_{2}^{2} - (1 +\alpha^{2}(t))X_{0}^{\top}X_{t} + \alpha(t)\norm{X_{0}}_{2}^{2} \bigg\vert X_{t}} \bigg\vert X_{t} = x} \\ & = \frac{\alpha(t)g(t)}{2h^{2}(t)}\E\bracket{-dh(t)\alpha(t) + \alpha(t)\norm{X_{t}}_{2}^{2} - (1 +\alpha^{2}(t))X_{0}^{\top}X_{t} + \alpha(t)\norm{X_{0}}_{2}^{2} \bigg\vert X_{t} = x}\E\bracket{X_{0} \vert X_{t} = x} \\ & = \frac{\alpha(t)g(t)}{2h^{2}(t)}\parenthesis{-dh(t)\alpha(t) + \alpha(t)\norm{x}_{2}^{2} - (1 + \alpha^{2}(t))x^{\top}\E\bracket{X_{0}\vert X_{t} = x} + \alpha(t)\E\bracket{\norm{X_{0}}^{2}_{2}\vert X_{t} = x}}\E\bracket{X_{0} \vert X_{t} = x}.
\end{align*}
Therefore, we conclude that
\begin{align*}
    \frac{\partial}{\partial t}f_{*}(x, t) & = \frac{\alpha(t)g(t)}{2h^{2}(t)}\bigg(\alpha(t)\E\bracket{\norm{X_{0}}_{2}^{2}\parenthesis{X_{0} - \E\bracket{X_{0}\vert X_{t}}}\big\vert X_{t} = x} \\ &  \qquad\qquad\qquad\qquad -(1 + \alpha^{2}(t))x\parenthesis{\E\bracket{\norm{X_{0}}_{2}^{2} \vert X_{t} = x} - \norm{\E\bracket{X_{0} \vert X_{t} = x}}_{2}^{2}}\bigg) \\ & = \frac{\alpha(t)g(t)}{2h^{2}(t)}\bigg(\alpha(t)\E\bracket{\norm{X_{0}}_{2}^{2}\parenthesis{X_{0} - \E\bracket{X_{0}\vert X_{t}}}\big\vert X_{t} = x}-(1 + \alpha^{2}(t))x{\rm Cov}(X_{0}\vert X_{t} = x)\bigg). 
\end{align*}
The Pythagorean theorem implies that $ \norm{X_{0} - \E\bracket{X_{0} \vert X_{t}}}_{2} \leq \norm{X_{0}}_{2} $. Since $ \norm{X_{0}}_{2} \leq D$ by Assumption \ref{ass:bounded target}, we can apply the triangle inequality to obtain
\begin{align*}
\norm{\frac{\partial}{\partial t}f_{*}(x, t)}_{2} & \leq  \frac{\alpha(t)g(t)}{2h^{2}(t)}\bigg[\alpha(t)\E\bracket{\norm{X_{0}}_{2}^{2}\norm{X_{0} - \E\bracket{X_{0} \vert X_{t}}}_{2} \bigg\vert X_{t} = x} \\ & \qquad\qquad\qquad + (1 + \alpha^{2}(t))\norm{x}_{2}\norm{{\rm Cov}(X_{0}\vert X_{t} = x)}_{2}\bigg] \\ & \leq \frac{\max_{t\in [T_0, T]}g(t)}{2h^2(T_0)} (D^3 + 2D^2 \sqrt{d}R) = \mathcal{O}(R),
\end{align*} 
where we have used the facts that $ \alpha(t) \leq 1 $, $ h(t) \geq h(T_{0}) $ and $ g(t) $ is uniformly bounded on $ [T_{0}, T] $.
   \end{proof}

\section{Proof of Theorem \ref{thm:coupling}}\label{sec:proof-couploing}

To prove Theorem \ref{thm:coupling}, we first show a linear convergence rate of the GD on the training dataset $S = \curly{(t_{j}, X_{0, j}, X_{t_{j}})}_{j=1}^N$. Recall the definition $ z_{j} = (X_{t_{j}}, t_{j} - T_{0}) $. Consider the Gram matrix $ H(\tau) \in \R^{dN \times dN} $ at each iteration $ \tau $ defined as the following block matrix:
\begin{align*}
H(\tau) \coloneqq 
\begin{pmatrix}
    H_{11}(\tau) & \dots & H_{1N}(\tau) \\  \vdots & \ddots & \vdots \\ H_{N1}(\tau) & \dots & H_{NN}(\tau)
\end{pmatrix},
\quad
    H^{ik}_{j\ell}(\tau) = \frac{1}{m}z_{j}^{\top}z_{\ell}\sum_{r = 1}^{m}a_{r}^{i}a_{r}^{k}\mathbb{I}\curly{z_{j}^{\top}w_{r}(\tau)\geq 0, z_{\ell}^{\top}w_{r}(\tau)\geq 0}.
\end{align*}
It is straightforward to check that $ H = \E\bracket{H(0)} $ with the expectation taken over the random initialization. For ease of presentation, recall that we have set $ C_{\max} = \sqrt{R^{2} + (T - T_{0})^{2}} $ so that $ \norm{(X_{t_{j}}, t_{j} - T_{0})}_{2} \in [\Delta, C_{\max}] $ with probability at least $1 - \delta_1$ by Lemma \ref{lemma:sampling-prob}. Moreover, we denote the activation pattern of neural $ w_{r} $ for sample $ j $ at iteration $ \tau $ as $ \mathbb{I}_{j, r}(\tau) \coloneqq \mathbb{I}\curly{w_{r}(\tau)^{\top}z_{j} \geq 0} $. The convergence of the GD algorithm is given in the next result.

\begin{Lemma}[Convergence Rate of GD]\label{thm:gd-conv}
    Suppose Assumptions \ref{ass:bounded target} and \ref{ass:gram-eigen} hold. If we set $ m = \Omega\parenthesis{\frac{d^5 N^6 C_{\max}^8 D^2}{\lambda_0^4\delta^2 \Delta^2}}$ with $ w_{r}(0) $ and $ a_{r}^{i} $ being initialized as in \eqref{eq:init a}--\eqref{eq:init w}, and we set $ \eta = \mcal{O}\parenthesis{\frac{\lambda_{0}}{(dN)^{2}C_{\max}^{4}}} $, then with probability at least $ 1 - \delta - \delta_1 $, for all $ \tau \geq 0 $ and $ r = 1, \dots, m $ simultaneously, we have
    \begin{align}
    \widehat{\mcal{L}}({\bf W}(\tau)) \leq \parenthesis{1 - \eta \lambda_{0}}^{\tau}\widehat{\mcal{L}}({\bf W}(0)),\label{eq:conv rate}
    \end{align}
    and 
    \begin{align}\label{eq:w_norm_bound}
        \norm{w_{r}(\tau) - w_{r}(0)}_{2} \leq \frac{2C_{\max} \sqrt{d}ND}{\sqrt{m}\lambda_0} \eqqcolon R_w.
    \end{align}
\end{Lemma}

\begin{proof}
    Following the idea in \cite{du2018gradient} and \cite{arora2019fine}, we prove the convergence of GD by a double induction.
    The induction is to show that (\ref{eq:conv rate}) holds for all $ \tau $. 
    It is straightforward to see the inequality holds for $ \tau = 0 $. Assuming (\ref{eq:conv rate}) holds for $ 0 \leq \tau' \leq \tau $, we will show it is also true for $ \tau' = \tau+1 $. Let $ u(\tau) = {\rm vec}(u_{1}, \dots, u_{N})(\tau) $ and $ y = {\rm vec}(y_{1}, \dots, y_{N}) $ with $ u_{j}(\tau) = f_{\bf W(\tau)}(X_{t_{j}}, t_{j}) $ and $ y_{j} = X_{0, j} $. We first need the following result {hold uniformly} for all $ \tau' = 0, \dots, \tau + 1 $:
    \begin{align}
    \norm{w_{r}(\tau') - w_{r}(0)}_{2} & = \norm{\eta \sum_{\tau'' = 0}^{\tau'-1}\frac{\partial \widehat{\mcal{L}}({\bf W}(\tau''))}{\partial w_{r}(\tau'')}}_{2} \nonumber \\ & \leq \eta \sum_{\tau'' = 0}^{\tau'-1}\norm{\frac{\partial \widehat{\mcal{L}}({\bf W}(\tau''))}{\partial w_{r}(\tau'')}}_{2} \nonumber \\ & \leq \eta C_{\max} \sum_{\tau'' = 0}^{\tau'-1} \frac{\sqrt{dN}\norm{u(\tau'') - y}_{2}}{\sqrt{m}} \label{eq:wr-1} \\ & \leq \frac{\eta C_{\max}\sqrt{dN}}{\sqrt{m}}\sum_{\tau'' = 0}^{\tau'-1}\parenthesis{1 - \eta \lambda_{0}}^{\tau''/2}\norm{u(0) - y}_{2} \label{eq:wr-2}\\ & \leq \frac{\eta C_{\max}\sqrt{dN}}{\sqrt{m}}\sum_{\tau'' = 0}^{\infty}(1 - \eta \lambda_{0}/2)^{\tau''}\norm{u(0) - y}_{2} \label{eq:wr-3}\\ & = \frac{2C_{\max}\sqrt{dN}\norm{u(0) - y}_{2}}{\sqrt{m}\lambda_{0}}. \label{eq:wr-4}
    \end{align}
    Here, we have an upper bound on gradient (\ref{eq:gradient}) to derive (\ref{eq:wr-1}). Also, (\ref{eq:wr-2}) and (\ref{eq:wr-3}) follow from the induction hypothesis (\ref{eq:conv rate}) and the fact that $ \sqrt{1 - x} \leq 1 - x/2 $. 
    Furthermore, the designed initialization implies $u(0) = 0$ and thus
    \begin{align*}
  \norm{u(0) - y}_2 = \norm{y}_2 = \sqrt{\sum_{j=1}^N \norm{y_j}_2^2} \leq \sqrt{N}D,
    \end{align*}
    where we use the fact that $ \norm{y_{j}}_{2} \leq D$.  Thus, we have
    \begin{align}
    \norm{w_{r}(\tau') - w_{r}(0)}_{2} \leq \frac{2C_{\max} \sqrt{d}ND}{\sqrt{m}\lambda_0} = R_w, \quad \forall  \tau' = 0, \dots, \tau +1, r = 1, \dots, m. \label{eq:wr-prob}
    \end{align}

    Define the following index sets
    \begin{align*}
    S_{j} \coloneqq \curly{r \in [m]: \mathbb{I}\curly{A_{j, r}} = 0}, \quad \bar{S}_{j} \coloneqq \curly{r \in [m]: \mathbb{I}\curly{A_{j, r}} \neq 0},
    \end{align*}
    where $ A_{j, r} \coloneqq \curly{\abs{w_{r}(0)^{\top}z_{j}} \leq R_{w} C_{\max}} $. Note that
    \begin{align}
    \mathbb{I}\curly{\mathbb{I}_{j, r}(\tau') \neq \mathbb{I}_{j, r}(0)} \leq \mathbb{I}\curly{A_{j, r}} + \mathbb{I}\curly{\norm{w_{r}(\tau') - w_{r}(0)}_{2} > R_w}. \label{eq:Sj-ineq}
    \end{align}
    To see this, note that if $ \norm{w_{r}(\tau') - w_{r}(0)}_{2} \leq R_{w} $  it follows $ \abs{w_{r}(\tau')^{\top}z_{j} - w_{r}(0)^{\top}z_{j}} \leq R_{w}C_{\max} $. If $ w_{r}(0)^{\top}z_{j} > R_{w}C_{\max} $, then $ w_{r}(\tau')^{\top}z_{j} > 0 $. Similarly, if $ w_{r}(0)^{\top}z_{j} < -R_{w}C_{\max} $, then $ w_{r}(\tau')^{\top}z_{j} < 0 $. Hence, we must have $ \mathbb{I}_{j, r}(\tau') = \mathbb{I}_{j, r}(0) $. From (\ref{eq:wr-prob}) and (\ref{eq:Sj-ineq}), we deduce that all neurons {with indices} in $ S_{j} $ will not change their activation pattern on $ z_{j} $ during optimization, i.e., 
    \begin{align}
    r \in S_{j} \implies  \mathbb{I}_{j, r}(\tau') = \mathbb{I}_{j, r}(0) , \quad \forall  \tau' = 0, \dots, \tau + 1. \label{eq:Sj-prop}
    \end{align}
    With such a partition, we can write the dynamics of $ u_{j}^{i}(\tau) $ as
    \begin{align}
    u_{j}^{i}(\tau+1) - u_{j}^{i}(\tau) & = \frac{1}{\sqrt{m}}\sum_{r = 1}^{m}a_{r}^{i}\bracket{\sigma(w_{r}(\tau+1)^{\top}z_{j}) - \sigma(w_{r}(\tau)^{\top}z_{j})} \nonumber \\ & = \frac{1}{\sqrt{m}}\sum_{r \in S_{j}}a_{r}^{i}\bracket{\sigma(w_{r}(\tau+1)^{\top}z_{j}) - \sigma(w_{r}(\tau)^{\top}z_{j})} \nonumber\\ & \qquad + \frac{1}{\sqrt{m}}\sum_{r \in \bar{S}_{j}}a_{r}^{i}\bracket{\sigma(w_{r}(\tau+1)^{\top}z_{j}) - \sigma(w_{r}(\tau)^{\top}z_{j})}. \label{eq:u-dynamics}
    \end{align}
    By utilizing the condition (\ref{eq:Sj-prop}), we bound the first summation in (\ref{eq:u-dynamics})  as
    \begin{align}
        & \frac{1}{\sqrt{m}}\sum_{r \in S_{j}}^{}a_{r}^{i}\bracket{\sigma(w_{r}(\tau+1)^{\top}z_{j}) - \sigma(w_{r}(\tau)^{\top}z_{j})} \nonumber \\ & = \frac{1}{\sqrt{m}}\sum_{r \in S_{j}}a_{r}^{i}\mathbb{I}_{j, r}(\tau)\parenthesis{w_{r}(\tau+1)^{\top}z_{j} -w_{r}(\tau)^{\top}z_{j}} \nonumber \\ & = \frac{1}{\sqrt{m}}\sum_{r \in S_{j}}a_{r}^{i}\mathbb{I}_{j, r}(\tau)\parenthesis{-\frac{\eta}{\sqrt{m}}\sum_{\ell=1}^{N}\sum_{k = 1}^{d}(u_{\ell}^{k}(\tau) - y_{\ell}^{k})a_{r}^{k}z_{\ell}\mathbb{I}_{\ell, r}(\tau)}^{\top}z_{j} \label{eq:Sj-1} \\ & = -\frac{\eta}{m}\sum_{\ell = 1}^{N}\sum_{k = 1}^{d}(u_{\ell}^{k}(\tau) - y_{\ell}^{k})z_{j}^{\top}z_{\ell}\sum_{r \in S_{j}}a_{r}^{i}a_{r}^{k}\mathbb{I}_{j, r}(\tau)\mathbb{I}_{\ell, r}(\tau) \nonumber \\ & = -\eta\sum_{\ell = 1}^{N}\sum_{k = 1}^{d}(u_{\ell}^{k}(\tau) - y_{\ell}^{k})H_{j\ell}^{ik}(\tau) + \epsilon_{j}^{i}(\tau),\label{eq:Sj-2}
    \end{align}
    where we set $$ \epsilon_{j}^{i}(\tau) \coloneqq \frac{\eta}{m}\sum_{\ell = 1}^{N}\sum_{k = 1}^{d}(u_{\ell}^{k}(\tau) - y_{\ell}^{k})z_{j}^{\top}z_{\ell}\sum_{r \in \bar{S}_{j}}a_{r}^{i}a_{r}^{k}(\tau)\mathbb{I}_{j, r}(\tau)\mathbb{I}_{\ell, r}(\tau). $$ Here, we use the GD update rule and the definition of $ H_{j\ell}^{ik}(\tau) $ to derive (\ref{eq:Sj-1}) and (\ref{eq:Sj-2}). We further bound the error term as
    \begin{align}
    \abs{\epsilon_{j}^{i}(\tau)} \leq \frac{\eta}{m}\sum_{\ell = 1}^{N}\sum_{k = 1}^{d}(u_{\ell}^{k}(\tau) - y_{\ell}^{k})\norm{z_{j}}_{2}\norm{z_{\ell}}_{2}\abs{\bar{S}_{j}} \leq \frac{\eta C_{\max}^{2}\abs{\bar{S}_{j}}\sqrt{dN}}{m}\norm{u(\tau) - y}_{2}. \label{eq:epsilon}
    \end{align}
    Next, we denote the second summation in (\ref{eq:u-dynamics}) by $ \bar{\epsilon}_{j}^{i}(\tau) $ and provide the following upper bound:
    \begin{align}
         \abs{\bar{\epsilon}_{j}^{i}(\tau)} & = \left|\frac{1}{\sqrt{m}}\sum_{r \in \bar{S}_{j}}a_{r}^{i}\bracket{\sigma(w_{r}(\tau+1)^{\top}z_{j}) - \sigma(w_{r}(\tau)^{\top}z_{j})}\right| \nonumber\\ & \leq \frac{1}{\sqrt{m}}\sum_{r \in \bar{S}_{j}}\abs{a_{r}^{i}}\abs{(w_{r}(\tau+1) - w_{r}(\tau))^{\top}z_{j}} \label{eq:bar-Sj-1} \\ & \leq \frac{C_{\max}}{\sqrt{m}}\sum_{r \in \bar{S}_{j}}\left\|w_{r}(\tau+1) - w_{r}(\tau)\right\|_{2} \label{eq:bar-Sj-2} \\ & = \frac{C_{\max}}{\sqrt{m}}\sum_{r \in \bar{S}_{j}}\left\|-\frac{\eta}{\sqrt{m}}\sum_{\ell=1}^{N}\sum_{k = 1}^{d}(u_{\ell}^{k}(\tau) - y_{\ell}^{k})a_{r}^{k}z_{\ell}\mathbb{I}_{\ell, r}(\tau)\right\|_{2} \label{eq:bar-Sj-3} \\ & \leq \frac{\eta C_{\max}}{m}\sum_{r \in \bar{S}_{j}}\sum_{\ell = 1}^{N}\sum_{k = 1}^{d}\abs{u_{\ell}^{k}(\tau) - y_{\ell}^{k}}\norm{z_{\ell}}_{2} \nonumber \\ & \leq \frac{\eta C^{2}_{\max}\abs{\bar{S}_{j}}\sqrt{dN}}{m}\norm{u(\tau) - y}_{2}, \label{eq:bar-Sj-4}
    \end{align}
    where we apply the $ 1 $-Lipschitz property of the ReLU activation function to obtain (\ref{eq:bar-Sj-1}). Also, we employ the facts that $ \abs{a_{r}^{i}} \leq 1 $ and $ \norm{z_{j}}_{2} \leq C_{\max} $ in (\ref{eq:bar-Sj-2}). The GD update rule is utilized to achieve (\ref{eq:bar-Sj-3}). Combining (\ref{eq:u-dynamics}), (\ref{eq:Sj-2}), \eqref{eq:epsilon} and (\ref{eq:bar-Sj-4}), we have
    \begin{align*}
        u_{j}^{i}(\tau+1) - u_{j}^{i}(\tau) = -\eta\sum_{\ell = 1}^{N}\sum_{k = 1}^{d}(u_{\ell}^{k}(\tau) - y_{\ell}^{k})H_{j\ell}^{ik}(\tau) + \epsilon_{j}^{i}(\tau) + \bar{\epsilon}_{j}^{i}(\tau),
    \end{align*}
    which can be further written in a compact form via vectorization:
    \begin{align}
    u(\tau+1) - u(\tau) & = -\eta H(\tau)(u(\tau) - y) + \epsilon(\tau) + \bar{\epsilon}(\tau) \nonumber \\ & = - \eta H(u(\tau) - y) + \eta(H - H(\tau))(u(\tau) - y) + \epsilon(\tau) + \bar{\epsilon}(\tau), \label{eq:GD-dynamics}
    \end{align}
    {where $ \epsilon(\tau) $ and $ \bar{\epsilon}(\tau) $ are defined in a  similar way as $ u(\tau) $ by vectorization}.

    We move on to show that $ H(\tau) $ is close to $ H $ when the neural network is sufficiently wide. Recall that the second half of $H^{ik}_{j\ell}(0)$ takes the same value as the first half. Thus, the Hoeffding's inequality implies that, with probability at least $ 1 - \delta' $, it holds that
    \begin{align*}
    \abs{H_{j\ell}^{ik}(0) - H_{j\ell}^{ik}} \leq C_{\max}^{2}\sqrt{\frac{\log(2/\delta')}{m}}.
    \end{align*} 
    Setting $ \delta' = \delta/(dN)^{2} $ and applying the union bound, we obtain 
    \begin{align}
    \norm{H - H(0)}_{F}^{2} = \sum_{i,k,j, \ell}\abs{H_{j\ell}^{ik}(0) - H_{j\ell}^{ik}}^{2} \leq (dN)^{2} C_{\max}^{4}\cdot \frac{\log(2(dN)^{2}/\delta)}{m}, \label{eq:H-H0}
    \end{align}
    with probability at least $ 1 - \delta $.    Next, note that (\ref{eq:Sj-ineq}) also implies
    \begin{align*}
        \sum_{r = 1}^{m}\mathbb{I}\curly{\mathbb{I}_{j, r}(\tau') \neq \mathbb{I}_{j, r}(0)} \leq \sum_{r = 1}^{m}\mathbb{I}\curly{A_{j, r}} + \sum_{r=1}^m\mathbb{I}\curly{\norm{w_{r}(\tau') - w_{r}(0)}_{2} > R_w}.
    \end{align*}
    It follows
    \begin{align*}
    \abs{H_{j\ell}^{ik}(\tau) - H_{j\ell}^{ik}(0)} & = \abs{\frac{1}{m}z_{j}^{\top}z_{\ell}\sum_{r = 1}^{m}a_{r}^{i}a_{r}^{k}\bracket{\mathbb{I}_{j, r}(\tau)\mathbb{I}_{\ell, r}(\tau) - \mathbb{I}_{j, r}(0)\mathbb{I}_{\ell, r}(0)}}  \\ & \leq \frac{C_{\max}^{2}}{m}\sum_{r = 1}^{m}\bracket{\mathbb{I}\curly{\mathbb{I}_{j,r}(\tau) \neq \mathbb{I}_{j, r}(0)} + \mathbb{I}\curly{\mathbb{I}_{\ell,r}(\tau) \neq \mathbb{I}_{\ell, r}(0)}} \\ & \leq \frac{C_{\max}^{2}}{m}\sum_{r = 1}^{m}\bracket{\mathbb{I}\curly{A_{j, r}} + \mathbb{I}\curly{A_{\ell, r}} + 2\mathbb{I}\curly{\norm{w_{r}(\tau) - w_{r}(0)}_{2} > R_{w}}}.
    \end{align*}
    Taking expectation on both sides and applying (\ref{eq:wr-prob}), we have
    \begin{align}
    & \E\bracket{\abs{H_{j\ell}^{ik}(\tau) - H_{j\ell}^{ik}(0)}} \nonumber \\ & \leq \frac{C_{\max}^{2}}{m}\sum_{r = 1}^{m}\E\bracket{\mathbb{I}\curly{A_{j, r}} + \mathbb{I}\curly{A_{\ell, r}}} + \frac{2C_{\max}^{2}}{m}\sum_{r=1}^m\E\bracket{\mathbb{I}\curly{\norm{w_{r}(\tau) - w_{r}(0)}_{2} > R_{w}}} \nonumber \\ & \leq \frac{4R_{w}C_{\max}^{3}}{\sqrt{2\pi}\Delta}, \label{eq:H-tau-H-0}
    \end{align}
    where we use \eqref{eq:wr-prob} to eliminate the second term and the following anti-concentration inequality for Gaussian random variables:
    \begin{align}
    \E\bracket{\mathbb{I}\curly{A_{j, r}}} = \P_{z \sim \mcal{N}(0, \norm{z_{j}}^{2}_{2})}\parenthesis{\abs{z} \leq R_{w}C_{\max}} = \int_{-R_{w}C_{\max}}^{R_{w}C_{\max}}\frac{1}{\sqrt{2\pi\norm{z_{j}}_{2}^{2}}}e^{-z^{2}/2\norm{z_{j}}_{2}^{2}} \leq \frac{2R_{w}C_{\max}}{\sqrt{2\pi}\Delta}. \label{eq:A-prob}
    \end{align}
    Hence, we have
    \begin{align*}
        \E\bracket{\norm{H(\tau) - H(0)}_{F}} \leq \sum_{i,k, j, \ell}\E\bracket{\abs{H_{j\ell}^{ik}(\tau) - H_{j\ell}^{ik}(0)}} \leq \frac{4(dN)^{2}R_{w}C_{\max}^{3}}{\sqrt{2\pi}\Delta}.
    \end{align*}
    Finally, by Markov's inequality, with probability at least $ 1 - \delta $ it holds  that
    \begin{align}
    \norm{H(\tau) - H(0)}_{F} \lesssim \frac{d^{5/2}N^3C_{\max}^{4}D}{\sqrt{m}\lambda_{0}\delta \Delta}. \label{eq:Htau-H0}
    \end{align}
    Therefore,  combining (\ref{eq:H-H0}) and (\ref{eq:Htau-H0}), we have with probability at least $ 1 - 2\delta$ that
    \begin{align}
    \norm{H - H(\tau)}_{2} & \leq \norm{H - H(0)}_{2} + \norm{H(0) - H(\tau)}_{2} \nonumber \\ &  \lesssim \frac{dN C_{\max}\sqrt{\log(2(dN)^{2}/\delta)}}{\sqrt{m}} + \frac{d^{5/2}N^3C_{\max}^{4}D}{\sqrt{m}\lambda_{0}\delta \Delta} \nonumber \\ & = \tilde{\mcal{O}}\parenthesis{\frac{d^{5/2}N^3C_{\max}^{4}D}{\sqrt{m}\lambda_{0}\delta\Delta}}, \;\; {\rm as}\; N \to \infty.\label{eq:H-Htau}
    \end{align}
    
    It remains to bound two error terms in (\ref{eq:GD-dynamics}). From (\ref{eq:epsilon}) and (\ref{eq:bar-Sj-4}), we know that
    \begin{align}
        \norm{\epsilon(\tau) + \bar{\epsilon}(\tau)}_{2} & \leq \norm{\epsilon(\tau) + \bar{\epsilon}(\tau)}_{1} \nonumber\\ & = \sum_{j = 1}^{N}\sum_{i = 1}^{d}\abs{\epsilon_{j}^{i}(\tau) + \bar{\epsilon}_{j}^{i}}(\tau)\nonumber \\ & \leq \sum_{j = 1}^{N}\sum_{i = 1}^{d}\frac{2\eta C_{\max}^{2}\abs{\bar{S}_{j}}\sqrt{dN}}{m}\norm{u(\tau) - y}_{2} \nonumber \\ & \leq \frac{2\eta C_{\max}^{2}d\sqrt{dN}}{m}\norm{u(\tau) - y}_{2}\sum_{j = 1}^{N}\abs{\bar{S}_{j}}. \label{eq:epsilon-sum}
    \end{align}
    Furthermore, it follows from (\ref{eq:wr-prob}) and (\ref{eq:A-prob})  that
    \begin{align*}
    \E\bracket{\abs{\bar{S}_{j}}} = \E\bracket{\sum_{r = 1}^{m}\mathbb{I}\curly{A_{j, r}}} \leq \frac{2mR_{w}C_{\max}}{\sqrt{2\pi}\Delta} \lesssim \frac{\sqrt{md}NC_{\max}^{2}D}{\lambda_{0}\Delta}.
    \end{align*}
    Thus, the Markov's inequality implies that $$ \sum_{j = 1}^{N}\abs{\bar{S}_{j}} \lesssim \frac{\sqrt{md}NC_{\max}^{2}D}{\lambda_{0}\Delta \delta}, $$ with probability at least $ 1 - \delta $. 
    
   Before proceeding to the induction hypothesis,  we need the following result, which holds under the same argument as in (\ref{eq:bar-Sj-4}) by considering all indices $r \in [m]$,
    \begin{align}
    \norm{u(\tau+1) - u(\tau)}_{2}^{2} & \leq \sum_{i, j}\abs{u_{j}^{i}(\tau+1) - u_{j}^{i}(\tau)}^{2}_{2} \nonumber \\ & \leq (dN)\parenthesis{\eta C_{\max}^{2}\sqrt{dN}\norm{u(\tau) - y}_{2}}^{2} \nonumber \\ & = \eta^{2}(dN)^{2}C_{\max}^{4}\norm{u(\tau) - y}_{2}^{2}. \label{eq:u-quadratic}
    \end{align}
    We now prove the induction hypothesis.
    With the prediction dynamics (\ref{eq:GD-dynamics}) and all the estimates (\ref{eq:H-Htau}), (\ref{eq:epsilon-sum}) and (\ref{eq:u-quadratic}), for large enough $N$ we have
    \begin{align*}
    & \norm{u(\tau+1) - y}_{2}^{2} \\ & = \norm{u(\tau+1) - u(\tau) + u(\tau) - y}_{2}^{2} \\ & = \norm{u(\tau) - y}_{2}^{2} + \norm{u(\tau+1) - u(\tau)}_{2}^{2} + 2(u(\tau+1) - u(\tau))^{\top}(u(\tau) - y) \\ & = \norm{u(\tau) - y}_{2}^{2} + \norm{u\parenthesis{\tau+1} - u(\tau)}_{2}^{2} -2\eta(u(\tau) - y)^{\top}H(u(\tau) - y ) \\ & \qquad + 2\eta(u(\tau) - y)^{\top}(H - H(\tau))(u(\tau) - y) + 2(\epsilon(\tau) + \bar{\epsilon}(\tau))^{\top}(u(\tau) - y) \\ & \lesssim \parenthesis{1 + \eta^{2}(dN)^{2}C_{\max}^{4} - 2\eta \lambda_{0} + \frac{\eta d^{5/2}N^3C_{\max}^{4}D}{\sqrt{m}\lambda_{0}\delta\Delta} + \frac{\eta d^2 N^{3/2}C_{\max}^{4}D}{\sqrt{m}\lambda_{0}\delta\Delta}}\norm{u(\tau) - y}_{2}^{2} \\ & \lesssim \parenthesis{1 + \eta^{2}(dN)^{2}C_{\max}^{4} - 2\eta \lambda_{0} + \frac{\eta d^{5/2}N^3C_{\max}^{4}D}{\sqrt{m}\lambda_{0}\delta\Delta}}\norm{u(\tau) - y}_{2}^{2},
   \end{align*}
   with probability at least $ 1 - 3\delta$. Now we choose 
   \begin{align*}
       m \gtrsim \frac{d^5 N^6 C_{\max}^8 D^2}{\lambda_0^4\delta^2 \Delta^2}, \;\; \text{and}\;\; \eta \lesssim \frac{\lambda_0}{(dN)^2C_{\max}^4},
   \end{align*}
   and consequently, we have
   \begin{align*}
    \norm{u(\tau+1) - y}_{2}^{2}   \leq (1 + \eta \lambda_0/2 - 2\eta \lambda_0 + \eta \lambda_0/2) \norm{u(\tau) - y}_{2}^{2} \leq    (1 - \eta \lambda_0)\norm{u(\tau) - y}_{2}^{2}.
   \end{align*}
   Finally, we finish the induction and conclude the proof by scaling $ \delta $. 
\end{proof}   

The non-expansive property of the projection operator and Assumption \ref{ass:bounded target} imply that
\begin{align}
     & \frac{1}{T - T_{0}}\int_{T_{0}}^{T}\E\bracket{\norm{\Pi_{D}\parenthesis{f_{\bf W(\tau)}(X_{t}, t)} - f^{K}_{\tau}(X_{t}, t)}_{2}^{2}\mathbb{I}\curly{\norm{X_{t}}_{2}\leq {R}}}{\rm d}t \nonumber\\ &  = \frac{1}{T - T_{0}}\int_{T_{0}}^{T_0 + \Delta}\E\bracket{\norm{\Pi_{D}\parenthesis{f_{\bf W(\tau)}(X_{t}, t)} - f^{K}_{\tau}(X_{t}, t)}_{2}^{2}\mathbb{I}\curly{\norm{X_{t}}_{2}\leq {R}}}{\rm d}t \nonumber\\ & \qquad + \frac{1}{T - T_{0}}\int_{T_0 + \Delta}^{T}\E\bracket{\norm{\Pi_{D}\parenthesis{f_{\bf W(\tau)}(X_{t}, t)} - f^{K}_{\tau}(X_{t}, t)}_{2}^{2}\mathbb{I}\curly{\norm{X_{t}}_{2}\leq {R}}}{\rm d}t \nonumber\\ & \leq \frac{4\Delta D^{2}}{T - T_{0}} + \frac{1}{T - T_{0}}\int_{T_0 + \Delta}^{T}\E\bracket{\norm{\Pi_D(f_{\bf W(\tau)}(X_{t}, t)) - f^{K}_{\tau}(X_{t}, t)}_{2}^{2}\mathbb{I}\curly{\norm{X_{t}}_{2}\leq {R}}}{\rm d}t.  \label{eq:coupling-1}
\end{align}
 To bound the second term in (\ref{eq:coupling-1}), we introduce a linearized neural network $f_{\bar{\bf W}(\tau)}^{\rm lin}$ defined as
\begin{align*}
f^{{\rm lin}, i}_{\bar{\bf W}(\tau)}(x, t) \coloneqq \frac{1}{\sqrt{m}}\sum_{r = 1}^{m}a_{r}^{i}\bar{w}_{r}(\tau)^{\top}(x, t - T_{0})\mathbb{I}\curly{{{w}_{r}}(0)^{\top}(x, t - T_{0}) \geq 0}.
\end{align*}
where we keep $w_r(0)$ fixed and only update parameter $\bar{w}_r(\tau)$ during training with
\begin{align*}
    \bar{w}_{r}\parenthesis{\tau+1} = \bar{w}_{r}(\tau) - \eta\nabla\widehat{\mcal{L}}^{{\rm lin}}(\bar{\bf W}(\tau)), \quad \widehat{\mcal{L}}^{{\rm lin}}(\bar{\bf W}(\tau)) = \frac{1}{2}\sum_{j = 1}^{N}\norm{f_{\bar{\bf W}(\tau)}^{\rm lin}(X_{t_{j}}, t_{j}) - X_{0, j}}_{2}^{2}.
\end{align*}
In particular, we initialize $ \bar{w}_{r}(0) = w_{r}(0) $. Our next lemma provides the coupling error between $f_{{\bf W}(\tau)}$ and $f^{{\rm lin}}_{\bar{\bf W}(\tau)}$. Let $P_{X_t}$ be the probability distribution induced by $X_t$.

\begin{Lemma}\label{lemma:f-flin}
    Assume the same conditions as in Theorem \ref{thm:coupling}. With probability at least $1 - \delta$, it holds simultaneously for each $\tau$ that
    \begin{align*}
    \frac{1}{T - T_{0}}\int_{T_{0}+\Delta}^{T}\int_{\norm{x}_{2} \leq R}\norm{f_{\bf{W}(\tau)}(x, t) - f^{{\rm lin}}_{\bar{\bf W}(\tau)}(x, t)}_{2}^{2}{\rm d}P_{X_{t}}(x){\rm d}t \lesssim \frac{d^7 N^7C_{\max}^{12}D^4}{\sqrt{m}\lambda_0^3\delta^2\Delta^2}.
    \end{align*}
\end{Lemma}
\begin{proof}
    Denote by $ \mathbb{I}_{r}(\tau) \coloneqq \mathbb{I}\curly{w_{r}(\tau)^{\top}(x, t - T_{0}) \geq 0} $. 
Note that for each $ i = 1, \dots, d $ we have
\begin{align}
    & \abs{f_{\bf{W}(\tau)}^{i}(x, t) - f^{{\rm lin}, i}_{\bar{\bf W}(\tau)}(x, t)} \nonumber\\ & = \abs{\frac{1}{\sqrt{m}}\sum_{r = 1}^{m}a_{r}^{i}\sigma\parenthesis{w_{r}(\tau)^{\top}(x, t - T_{0})} - \frac{1}{\sqrt{m}}\sum_{r = 1}^{m}a_{r}^{i}\bar{w}_{r}(\tau)^{\top}(x, t - T_{0})\mathbb{I}_{r}(0)} \nonumber\\ & \leq \abs{\frac{1}{\sqrt{m}}\sum_{r = 1}^{m}a_{r}^{i}\sigma\parenthesis{w_{r}(\tau)^{\top}(x, t - T_{0})} - \frac{1}{\sqrt{m}}\sum_{r = 1}^{m}a_{r}^{i}w_{r}(\tau)^{\top}(x, t - T_{0})\mathbb{I}_{r}(0)}\nonumber \\ & \qquad + \bigg|\frac{1}{\sqrt{m}}\sum_{r = 1}^{m}a_{r}^{i}w_{r}(\tau)^{\top}(x, t - T_{0})\mathbb{I}_{r}(0) - \frac{1}{\sqrt{m}}\sum_{r = 1}^{m}a_{r}^{i}\bar{w}_{r}(\tau)^{\top}(x, t - T_{0})\mathbb{I}_{r}(0)\bigg| \nonumber \\ & = \abs{\frac{1}{\sqrt{m}}\sum_{r = 1}^{m}a_{r}^{i}w_{r}(\tau)^{\top}(x, t - T_{0})\parenthesis{\mathbb{I}_{r}(\tau) - \mathbb{I}_{r}(0)}}  + \abs{\frac{1}{\sqrt{m}}\sum_{r = 1}^{m}a_{r}^{i}(w_{r}(\tau) - \bar{w}_{r}(\tau))^{\top}(x, t - T_{0})\mathbb{I}_{r}(0)} \nonumber \\ & \leq \frac{1}{\sqrt{m}}\sum_{r=1}^m \abs{w_r(\tau)^\top(x, t - T_0)}\mathbb{I}\curly{\mathbb{I}_{r}(\tau) \neq \mathbb{I}_{r}(0)}  + \frac{1}{\sqrt{m}}\sum_{r = 1}^{m}\abs{\parenthesis{w_{r}(\tau) - \bar{w}_{r}(\tau)}^{\top}(x, t - T_{0})}\mathbb{I}_{r}(0). \label{eq:lin-error-1}
\end{align}
Note that the following fact holds:
\begin{align}
\abs{a}\mathbb{I}\curly{{\rm sgn}(a) \neq {\rm sgn}(b)} \leq \abs{a - b}\mathbb{I}\curly{{\rm sgn}(a) \neq {\rm sgn}(b)}, \quad \forall a, b \in \R.\label{eq:sgn}
\end{align}
By taking $a = w_r(\tau)^\top(x, t - T_0)$ and $b = w_r(0)^\top(x, t - T_0)$, we further bound \eqref{eq:lin-error-1} with 
\begin{align}
\abs{f_{\bf{W}(\tau)}^{i}(x, t) - f^{{\rm lin}, i}_{\bar{\bf W}(\tau)}(x, t)}  & \leq 
    \frac{1}{\sqrt{m}}\sum_{r = 1}^{m}\abs{\parenthesis{w_{r}(\tau) - w_{r}(0)}^{\top}(x, t - T_{0})}\mathbb{I}\curly{\mathbb{I}_{r}(\tau) \neq \mathbb{I}_{r}(0)}  \nonumber\\ & \qquad\qquad\qquad+ \frac{1}{\sqrt{m}}\sum_{r = 1}^{m}\abs{\parenthesis{w_{r}(\tau) - \bar{w}_{r}(\tau)}^{\top}(x, t - T_{0})}\mathbb{I}_{r}(0). \label{eq:lin-error}
\end{align}
Taking square on both sides of (\ref{eq:lin-error}) and applying the Jensen's inequality, we have that
\begin{align}
    \abs{f_{\bf{W}(\tau)}^{i}(x, t) - f^{{\rm lin}, i}_{\bar{\bf W}(\tau)}(x, t)}^{2} & \leq 2\sum_{r = 1}^{m}\abs{\parenthesis{w_{r}(\tau) - w_{r}(0)}^{\top}(x, t - T_{0})}^{2}\mathbb{I}\curly{\mathbb{I}_{r}(\tau) \neq \mathbb{I}_{r}(0)} \nonumber \\ & \quad + 2\parenthesis{ \frac{1}{\sqrt{m}}\sum_{r = 1}^{m}\abs{\parenthesis{w_{r}(\tau) - \bar{w}_{r}(\tau)}^{\top}(x, t - T_{0})}\mathbb{I}_{r}(0)}^{2}. \label{eq:jensen}
\end{align}

To proceed, we start with the bound for the first term in (\ref{eq:jensen}). Recall that Theorem \ref{thm:gd-conv} implies $\norm{w_{r}(\tau) - w_{r}(0)}_{2} \leq R_{w}$  for all $ \tau \geq 0 $ and $ r = 1, \dots m$ simultaneously. 
Combining this result with  the Cauchy-Schwarz inequality, we deduce that it holds  uniformly  for all $ \norm{x}_{2} \leq R $ and $ t \in [T_{0}+ \Delta, T] $ that 
\begin{align}
& \sum_{r = 1}^{m}\abs{(w_{r}(\tau) - w_{r}(0))^{\top}(x, t - T_{0})}^{2}\mathbb{I}\curly{\mathbb{I}_{r}(\tau) \neq \mathbb{I}_{r}(0)} \nonumber\\ & \leq \sum_{r = 1}^{m}\norm{w_{r}(\tau) - w_{r}(0)}_{2}^{2}\norm{(x, t - T_{0})}^{2}_{2}\mathbb{I}\curly{\mathbb{I}_{r}(\tau) \neq \mathbb{I}_{r}(0)} \nonumber\\ & \leq R_{w}^{2}C_{\max}^{2}\sum_{r = 1}^{m}\mathbb{I}\curly{\mathbb{I}_{r}(\tau) \neq \mathbb{I}_{r}(0)}. \label{eq:lin-error-Cauchy}
\end{align}
Taking integration of \eqref{eq:lin-error-Cauchy} over $ \norm{x}_2 \leq R $ and $t \in [T_0 +\Delta, T]$, we obtain
\begin{align}
& \int_{T_{0}+\Delta}^{T}\int_{\norm{x}_{2}\leq R}\sum_{r = 1}^{m}\abs{(w_{r}(\tau) - w_{r}(0))^{\top}(x, t - T_{0})}^{2}\mathbb{I}\curly{\mathbb{I}_{r}(\tau) \neq \mathbb{I}_{r}(0)}{\rm d}P_{X_{t}}(x){\rm d}t \nonumber \\ & \leq R_{w}^{2}C^{2}_{\max}\int_{T_{0} + \Delta}^{T}\int_{\norm{x}_{2}\leq R}\sum_{r = 1}^{m}\mathbb{I}\curly{\mathbb{I}_{r}(\tau) \neq \mathbb{I}_{r}(0)}{\rm d}P_{X_{t}}(x){\rm d}t.\label{eq:inner-pattern}
\end{align}
Next, similar to (\ref{eq:Sj-ineq}), when $ \norm{x}_{2} \leq R $ and $ t \in [T_{0} + \Delta, T] $ we have
\begin{align}
    \mathbb{I}\curly{\mathbb{I}_{r}(\tau) \neq \mathbb{I}_{r}(0)} \leq \mathbb{I}\curly{\abs{w_{r}(0)^{\top}(x, t - T_{0})} \leq R_{w} C_{\max}} + \mathbb{I}\curly{\norm{w_{r}(\tau) - w_{r}(0)} > R_{w}}. \label{eq:actpat-x-t}
\end{align}
By taking expectation w.r.t. $ \curly{w_{r}(0)}_{r = 1}^{m} $ in (\ref{eq:actpat-x-t}),  we have
\begin{align}
    \E\bracket{\sum_{r = 1}^{m}\mathbb{I}\curly{\mathbb{I}_{r}(\tau) \neq \mathbb{I}_{r}(0)}} & \leq \sum_{r = 1}^{m}\E\bracket{\mathbb{I}\curly{\abs{w_{r}(0)^{\top}(x, t - T_{0})} \leq R_{w}C_{\max}}} + \E\bracket{\mathbb{I}\curly{\norm{w_{r}(\tau) - w_{r}(0)} > R_{w}}} \nonumber \\ & \leq \frac{2mR_{w}C_{\max}}{\sqrt{2\pi}\Delta}. \label{eq:unif-bound-actpat}
\end{align}
Now integrating over $ (x, t) $, we get
\begin{align*}
\int_{T_{0} + \Delta}^{T}\int_{\norm{x}_{2}\leq R}\E\bracket{\sum_{r = 1}^{m}\mathbb{I}\curly{\mathbb{I}_{r}(\tau) \neq \mathbb{I}_{r}(0)} }{\rm d}P_{X_{t}}(x){\rm d}t \leq (T - T_{0} - \Delta)\parenthesis{\frac{2mR_{w}C_{\max}}{\sqrt{2\pi}\Delta}}.
\end{align*}
Since the neural network parameters are initialized independent of the training data sampling procedure, the Fubini's theorem and the Markov inequality imply that with probability at least $ 1 - \delta $, the following inequality holds:
\begin{align}\label{eq:fubini}
    \int_{T_{0} + \Delta}^{T}\int_{\norm{x}_{2}\leq R}\sum_{r = 1}^{m}\mathbb{I}\curly{\mathbb{I}_{r}(\tau) \neq \mathbb{I}_{r}(0)}{\rm d}P_{X_{t}}(x){\rm d}t \leq (T - T_{0} - \Delta)\parenthesis{\frac{2mR_{w}C_{\max}}{\sqrt{2\pi}\Delta\delta}}.
\end{align}
Therefore, we combine \eqref{eq:inner-pattern} and \eqref{eq:fubini} to deduce that with probability $ 1 - \delta$ it holds
\begin{align}
    & \frac{1}{T - T_{0}}\int_{T_{0}+\Delta}^{T}\int_{\norm{x}_{2}\leq R}\sum_{r = 1}^{m}\abs{(w_{r}(\tau) - w_{r}(0))^{\top}(x, t - T_{0})}^{2}\mathbb{I}\curly{\mathbb{I}_{r}(\tau) \neq \mathbb{I}_{r}(0)}{\rm d}P_{X_{t}}(x){\rm d}t \nonumber \\ & \leq R^{2}_{w}C^{2}_{\max}\parenthesis{\frac{2mR_{w}C_{\max}}{\sqrt{2\pi}\Delta\delta}} \frac{T - T_{0} - \Delta}{T - T_{0}} \lesssim \frac{d^{3/2}N^3D^3C_{\max}^{6}}{\sqrt{m}\delta\lambda_{0}^{3}}. \label{eq:nn-lin-term1}
\end{align}

We move on to bound the second term in (\ref{eq:lin-error}). Note that for all $ \norm{x}_{2} \leq R $ and $ t \in [T_{0}+\Delta, T] $, the Cauchy-Schwarz inequality implies
\begin{align}
\frac{1}{\sqrt{m}}\sum_{r = 1}^{m}\abs{(w_{r}(\tau) - \bar{w}_{r}(\tau))^{\top}(x, t - T_{0})}\mathbb{I}_{r}(0)  & \leq \frac{1}{\sqrt{m}}\sum_{r = 1}^{m}\norm{w_{r}(\tau) - \bar{w}_{r}(\tau)}_{2}\norm{(x, t - T_{0})}_{2}\mathbb{I}_{r}(0) \nonumber \\ & \leq \frac{C_{\max}}{\sqrt{m}}\sum_{r = 1}^{m}\norm{w_{r}(\tau) - \bar{w}_{r}(\tau)}_{2}.\label{eq:second-term-coupling}
\end{align}
Recall the GD update rule for $ w_{r}(\tau) $ and $ \bar{w}_{r}(\tau) $ as follow:
\begin{align*}
w_{r}(\tau+1) & = w_{r}(\tau) - \frac{\eta}{\sqrt{m}}\sum_{j = 1}^{N}\sum_{i = 1}^{d}(u_{j}^{i}(\tau) - y_{j}^{i})a_{r}^{i}z_{j}\mathbb{I}\curly{w_{r}(\tau)^{\top}z_{j} \geq 0}, \\ 
\bar{w}_{r}(\tau+1) & = \bar{w}_{r}(\tau) - \frac{\eta}{\sqrt{m}}\sum_{j = 1}^{N}\sum_{i = 1}^{d}(u_{j}^{{\rm lin}, i}(\tau) - y_{j}^{i})a_{r}^{i}z_{j}\mathbb{I}\curly{w_{r}(0)^{\top}z_{j} \geq 0},
\end{align*}
in which we let $ u_{j}^{i}(\tau) = f_{{\bf W}(\tau)}^{i} $ and $ u_{j}^{{\rm lin}, i}(\tau) = f_{\bar{\bf W}(\tau)}^{{\rm lin}, i} $ be evaluated at the sample $ (X_{t_{j}}, t_{j}) $. Thus, we~write
\begin{align*}
    w_{r}(\tau+1) - \bar{w}_{r}(\tau+1) & = w_{r}(\tau) - \bar{w}_{r}(\tau) - \frac{\eta}{\sqrt{m}}\sum_{j = 1}^{N}\sum_{i = 1}^{d}(u_{j}^{i}(\tau) - y_{j}^{i})a_{r}^{i}z_{j}\parenthesis{\mathbb{I}_{j, r}(\tau) -  \mathbb{I}_{j, r}(0)} \\ & \quad -  \frac{\eta}{\sqrt{m}}\sum_{j = 1}^{N}\sum_{i = 1}^{d}\parenthesis{u_{j}^{i}(\tau) - u_{j}^{{\rm lin}, i}(\tau)}a_{r}^{i}z_{j}\mathbb{I}_{j, r}(0).
\end{align*}
Taking the 2-norm on both sides and applying the Cauchy-Schwarz inequality, we obtain
\begin{align*}
& \norm{w_{r}(\tau+1) - \bar{w}_{r}(\tau+1)}_{2} \\ & \leq \norm{w_{r}(\tau) - \bar{w}_{r}(\tau)}_{2} + \frac{\eta}{\sqrt{m}}\sum_{j = 1}^{N}\sum_{i = 1}^{d}\abs{u_{j}^{i}(\tau) - y_{j}^{i}}\abs{a_{r}^{i}}\norm{z_{j}}_{2}\abs{\mathbb{I}_{j, r}(\tau) - \mathbb{I}_{j, r}(0)} \\ & \qquad + \frac{\eta}{\sqrt{m}}\sum_{j = 1}^{N}\sum_{i = 1}^{d}\abs{u_{j}^{i}(\tau) - u_{j}^{{\rm lin}, i}(\tau)}\abs{a_{j}^{i}}\norm{z_{j}}_{2}\abs{\mathbb{I}_{j, r}(0)}  \\ & \leq \norm{w_{r}(\tau) - \bar{w}_{r}(\tau)}_{2} + \frac{\eta \sqrt{d}C_{\max}}{\sqrt{m}}\norm{u(\tau) - y}_2\sqrt{\sum_{j = 1}^{N}\mathbb{I}\curly{\mathbb{I}_{j, r}(\tau) \neq \mathbb{I}_{j, r}(0)}} \\ & \qquad + \frac{\eta \sqrt{d}C_{\max}}{\sqrt{m}}\norm{u(\tau) - u^{\rm lin}(\tau)}_2\sqrt{\sum_{j = 1}^{N}\mathbb{I}_{j, r}(0)}.
\end{align*}
Summing over all neurons and applying the Cauchy-Schwarz inequality again, we get
\begin{align}
& \sum_{r = 1}^{m}\norm{w_{r}(\tau+1) - \bar{w}_{r}(\tau+1)}_{2}\nonumber  \\ & \leq \sum_{r = 1}^{m}\norm{w_{r}(\tau) - \bar{w}_{r}(\tau)}_{2} + \eta \sqrt{d}C_{\max}\norm{u(\tau) - y}_2\sqrt{\sum_{r = 1}^{m}\sum_{j = 1}^{N}\mathbb{I}\curly{\mathbb{I}_{j, r}(\tau) \neq \mathbb{I}_{j, r}(0)}} \nonumber \\ & \qquad + \eta \sqrt{d}C_{\max}\norm{u(\tau) - u^{\rm lin}(\tau)}_2\sqrt{\sum_{r = 1}^{m}\sum_{j = 1}^{N}\mathbb{I}_{j, r}(0)}. \label{eq: w-bar-w}
\end{align}
Since $ w_{r}(0) = \bar{w}_{r}(0) $, telescoping sum over $\tau$ in  (\ref{eq: w-bar-w}) leads to 
\begin{align}
\sum_{r = 1}^{m}\norm{w_{r}(\tau) - \bar{w}_{r}(\tau)}_{2} & \leq \eta \sqrt{d}C_{\max}\sum_{s = 0}^{\tau - 1}\norm{u(s) - y}_{2}\sqrt{\sum_{r = 1}^{m}\sum_{j = 1}^{N}\mathbb{I}\curly{\mathbb{I}_{j, r}(\tau) \neq \mathbb{I}_{j, r}(0)}} \nonumber \\ & \qquad + \eta\sqrt{d}C_{\max}\sum_{s = 0}^{\tau - 1}\norm{u(\tau) - u^{\rm lin}(s)}_{2}\sqrt{\sum_{r = 1}^{m}\sum_{j = 1}^{N}\mathbb{I}_{j, r}(0)}. \label{eq:w-bar-w-sum}
\end{align}
Recall that Lemma \ref{thm:gd-conv} implies that
\begin{align}
\norm{u(\tau) - y}_{2}^{2} \leq (1 - \eta\lambda_{0})^{\tau}\norm{u(0) - y}_{2}^{2} =  (1 - \eta\lambda_{0})^{\tau}D^2 . \label{eq:utau-y}
\end{align}
Moreover, (\ref{eq:unif-bound-actpat}) leads to 
\begin{align*}
\E\bracket{\sum_{r = 1}^{m}\sum_{j = 1}^{N}\mathbb{I}\curly{\mathbb{I}_{j, r}(\tau) \neq \mathbb{I}_{j, r}(0)}} \leq N\parenthesis{\frac{2mR_{w}C_{\max}}{\sqrt{2\pi}\Delta}}.
\end{align*}
The Markov inequality implies with probability at least $ 1 - \delta $, we have
\begin{align}
    {\sum_{r = 1}^{m}\sum_{j = 1}^{N}\mathbb{I}\curly{\mathbb{I}_{j, r}(\tau) \neq \mathbb{I}_{j, r}(0)}} \leq N\parenthesis{\frac{2mR_{w}C_{\max}}{\sqrt{2\pi}\Delta\delta}}. \label{eq:act-pattern-sum}
\end{align}

It remains to establish a high probability bound of $ \norm{u(\tau) - u^{\rm lin}(\tau)}_{2} $. From the definitions of $ u(\tau) $ and $ u^{\rm lin}(\tau) $, we have
\begin{align*}
& u_{j}^{i}(\tau+1) - u^{{\rm lin}, i}_{j}(\tau+1) \\ & = \frac{1}{\sqrt{m}}\sum_{r = 1}^{m}a_{r}^{i}\sigma\parenthesis{w_{r}(\tau+1)^{\top}z_{j}} - \frac{1}{\sqrt{m}}\sum_{r = 1}^{m}a_{r}^{i}\bar{w}_{r}(\tau+1)^{\top}z_{j}\mathbb{I}_{j,r}(0) \\ & = \frac{1}{\sqrt{m}}\sum_{r = 1}^{m}a_{r}^{i}w_{r}(\tau+1)^{\top}z_{j}\mathbb{I}_{j, r}(\tau) + \frac{1}{\sqrt{m}}\sum_{r = 1}^{m}a_{r}^{i}w_{r}(\tau+1)^{\top}z_{j}\parenthesis{\mathbb{I}_{j, r}(\tau+1)-\mathbb{I}_{j, r}(\tau)} \\ & \quad - \frac{1}{\sqrt{m}}\sum_{r = 1}^{m}a_{r}^{i}\bar{w}_{r}(\tau+1)^{\top}z_{j}\mathbb{I}_{j,r}(0) \\ & = \frac{1}{\sqrt{m}}\sum_{r = 1}^{m}a_{r}^{i}\parenthesis{w_{r}(\tau) - \eta \frac{\partial \widehat{\mcal{L}}(\bf W(\tau))}{\partial w_{r}(\tau)}}^{\top}z_{j}\mathbb{I}_{j, r}(\tau) + \frac{1}{\sqrt{m}}\sum_{r = 1}^{m}a_{r}^{i}w_{r}(\tau+1)^{\top}z_{j}\parenthesis{\mathbb{I}_{j, r}(\tau+1)-\mathbb{I}_{j, r}(\tau)} \\ & \quad - \frac{1}{\sqrt{m}}\sum_{r = 1}^{m}a_{r}^{i}\parenthesis{\bar{w}_{r}(\tau) - \eta \frac{\partial \widehat{\mcal{L}}^{\rm lin}(\bar{\bf W}(\tau))}{\partial \bar{w}_{r}(\tau)}}^{\top}z_{j}\mathbb{I}_{j,r}(0) \\ & = u_{j}^{i}(\tau) - u_{j}^{{\rm lin}, i}(\tau) + \frac{\eta}{\sqrt{m}}\sum_{r = 1}^{m}a_{r}^{i}\parenthesis{\frac{\partial \widehat{\mcal{L}}^{\rm lin}(\bar{\bf W}(\tau))}{\partial \bar{w}_{r}(\tau)}\mathbb{I}_{j, r}(0) - \frac{\partial \widehat{\mcal{L}}(\bf W(\tau))}{\partial w_{r}(\tau)}\mathbb{I}_{j, r}(\tau)}^{\top}z_{j} \\ & \quad + \frac{1}{\sqrt{m}}\sum_{r = 1}^{m}a_{r}^{i}w_{r}(\tau+1)^{\top}z_{j}\parenthesis{\mathbb{I}_{j, r}(\tau+1) - \mathbb{I}_{j, r}(\tau)} \\ & = u_{j}^{i}(\tau) - u_{j}^{{\rm lin}, i}(\tau) + \eta\sum_{\ell = 1}^{N}\sum_{k = 1}^{d}(u_{\ell}^{{\rm lin}, k}(\tau) - y_{\ell}^{k})H_{j\ell}^{ik}(0) - \eta\sum_{\ell = 1}^{N}\sum_{k = 1}^{d}(u_{\ell}^{k}(\tau) - y_{\ell}^{k})H_{j\ell}^{ik}(\tau) \\ & \quad + \frac{1}{\sqrt{m}}\sum_{r = 1}^{m}a_{r}^{i}w_{r}(\tau+1)^{\top}z_{j}\parenthesis{\mathbb{I}_{j, r}(\tau+1) - \mathbb{I}_{j, r}(\tau)} \\ & = u_{j}^{i}(\tau) - u_{j}^{{\rm lin}, i}(\tau) + \eta\sum_{\ell = 1}^{N}\sum_{k = 1}^{d}(u_{\ell}^{{\rm lin}, k}(\tau) - u_{\ell}^{k}(\tau))H_{j\ell}^{ik}(0) - \eta\sum_{\ell = 1}^{N}\sum_{k = 1}^{d}(u_{\ell}^{k}(\tau) - y_{\ell}^{k})(H_{j\ell}^{ik}(\tau) - H_{j\ell}^{ik}(0)) \\ & \quad + \frac{1}{\sqrt{m}}\sum_{r = 1}^{m}a_{r}^{i}w_{r}(\tau+1)^{\top}z_{j}\parenthesis{\mathbb{I}_{j, r}(\tau+1) - \mathbb{I}_{j, r}(\tau)}.
\end{align*}
Here, we use the facts that $\sigma(x) = x \cdot \mathbb{I}\curly{x \geq 0}$ and the GD update rules for $w_r(\tau)$ and $\bar{w}_r(\tau)$. 
Define a block matrix $ {\bf Z}(\tau) $ such that its $ (i, j) $-th row is 
\begin{align*}
    \parenthesis{{\bf Z}^{i}_{j}}^{\top}(\tau) \coloneqq \frac{1}{\sqrt{m}}\bracket{a_{1}^{i}z_{j}^{\top}\mathbb{I}_{j, 1}(\tau), \dots,  a_{m}^{i}z_{j}^{\top}\mathbb{I}_{j, m}(\tau)}.
\end{align*}
With vectorization, we rewrite the above equation in a compact form:
\begin{align}
u(\tau+1) - u^{\rm lin}(\tau+1) & = u(\tau) - u^{\rm lin}(\tau) + \eta H(0)(u^{\rm lin}(\tau) - u(\tau)) - \eta (H(\tau) - H(0))(u(\tau) - y) \nonumber \\ & \quad + ({\bf Z}(\tau+1) - {\bf Z}(\tau)){\rm vec}(\bf W)(\tau+1) \nonumber \\ & = \parenthesis{I_{dN} - \eta H(0)}(u(\tau) - u^{\rm lin}(\tau)) - \eta \underbrace{(H(\tau) - H(0))(u(\tau) - y)}_{\eqqcolon \xi(\tau)} \nonumber \\ & \quad + \underbrace{({\bf Z}(\tau+1) - {\bf Z}(\tau)){\rm vec}(\bf W)(\tau+1)}_{\eqqcolon\bar{\xi}(\tau)}.\label{eq:u-ulin-vec}
\end{align}
Unrolling recursion (\ref{eq:u-ulin-vec}) and noticing that $ u(0) = u^{\rm lin}(0) $ by construction, we obtain
\begin{align*}
    u(\tau) - u^{\rm lin}(\tau) = \sum_{s = 0}^{\tau-1}(I_{dN} - \eta H(0))^{\tau - 1 - s}\parenthesis{-\eta \xi(s) + \bar{\xi}(s)}.
\end{align*}
The summation should be understood as $ 0 $ when $ \tau = 0 $. 
Taking 2-norm on both sides and applying the Cauchy-Schwarz inequality and the triangle inequality, we deduce that
\begin{align}
\norm{u(\tau) - u^{\rm lin}(\tau)}_{2} & \leq \sum_{s = 0}^{\tau-1}\norm{(I_{dN} - \eta H(0))^{\tau - 1 - s}}_{2}\parenthesis{\eta\norm{\xi(s)}_{2} + \norm{\bar{\xi}(s)}_{2}}. \label{eq:u-ulin-0}
\end{align}
Since $\lambda_{\min}(H) \geq \lambda_0$ by Assumption \ref{ass:gram-eigen}, we choose $m \geq \frac{(dN)^2 C_{\max}^4 \log(2(dN)^2/\delta)}{\lambda_0^2}$ and apply \eqref{eq:H-H0} to have $\lambda_{\min}(H_0) \geq \lambda_0/2$ with probability at least $ 1- \delta$. Thus, we bound \eqref{eq:u-ulin-0} with
\begin{align}
    \norm{u(\tau) - u^{\rm lin}(\tau)}_{2} \leq \sum_{s = 0}^{\tau-1}(1 - \eta \lambda_{0}/2)^{\tau - 1 - s}\parenthesis{\eta\norm{\xi(s)}_{2} + \norm{\bar{\xi}(s)}_{2}}. \label{eq:u-ulin}
\end{align}

We now turn to bound $ \norm{\xi(s)}_{2} $ and $ \norm{\bar{\xi}(s)}_{2} $. Note that (\ref{eq:Htau-H0}) and (\ref{eq:utau-y}) imply that with probability at least $ 1 - \delta $, it holds that 
\begin{align}
\norm{\xi(s)}_{2} & \leq \norm{H(0) - H(s)}_{2}\norm{u(s) - y}_{2} \lesssim  \frac{d^{5/2}N^3C_{\max}^{4}D^2}{\sqrt{m}\lambda_{0}\delta \Delta}(1 - \eta\lambda_{0})^{s/2}. \label{eq:xi-bound}
\end{align}
Next, we notice that the $ (i, j) $-entry of $\bar{\xi}(s)$ is bounded with
\begin{align}
\abs{\bar{\xi}_{j}^{i}(s)} & \leq \frac{1}{\sqrt{m}}\sum_{r = 1}^{m}\abs{a_{r}^{i}}\abs{w_{r}(s+1)^{\top}z_{j}}\abs{\mathbb{I}_{j, r}(s+1) - \mathbb{I}_{j, r}(s)} \nonumber\\ & \leq \frac{1}{\sqrt{m}}\sum_{r = 1}^{m}\abs{w_{r}(s+1)^{\top}z_{j} - w_{r}(s)^{\top}z_{j}}\abs{\mathbb{I}_{j, r}(s+1) - \mathbb{I}_{j, r}(s)} \nonumber\\ & \leq \frac{C_{\max}}{\sqrt{m}}\sum_{r = 1}^{m}\norm{w_{r}(s+1) - w_{r}(s)}_{2}\abs{\mathbb{I}_{j, r}(s+1) - \mathbb{I}_{j, r}(s)}, \label{eq:xi-bar}
\end{align}
where the second inequality follows from \eqref{eq:sgn}.
To proceed, we apply the GD update rule to get
\begin{align}
\norm{w_{r}(s+1) - w_{r}(s)}_{2} & = \norm{\frac{\eta}{\sqrt{m}}\sum_{j = 1}^{N}\sum_{i = 1}^{d}(u_{j}^{i}(s) - y_{j}^{i})a_{r}^{i}z_{j}\mathbb{I}_{j, r}(s)}_{2} \nonumber  \\ & \leq \frac{\eta C_{\max}}{\sqrt{m}}\norm{u(s) - y}_{1} \leq \frac{\eta \sqrt{dN} C_{\max}}{\sqrt{m}}\norm{u(s) - y}_{2}. \label{eq:ws}
\end{align}
Substituting (\ref{eq:ws}) into (\ref{eq:xi-bar}),  with probability at least $  1 - 3\delta $ it holds that
\begin{align}
\abs{\bar{\xi}_{j}^{i}(s)} & \leq \frac{\eta \sqrt{dN}C_{\max}^{2}}{m}\norm{u(s) - y}_{2}\sum_{r = 1}^{m}\abs{\mathbb{I}_{j, r}(s+1) - \mathbb{I}_{j, r}(s)} \nonumber \\ & \leq \frac{\eta \sqrt{dN}C_{\max}^{2}}{m}\norm{u(s) - y}_{2}\parenthesis{\sum_{r = 1}^{m}\abs{\mathbb{I}_{j, r}(s+1) - \mathbb{I}_{j, r}(0)} + \sum_{r = 1}^{m}\abs{\mathbb{I}_{j, r}(s) - \mathbb{I}_{j, r}(0)}}. \label{eq:xi-bar-1}
\end{align}
Recall that \eqref{eq:Sj-ineq} and \eqref{eq:A-prob} imply the following holds with probability $1 - \delta$ that
\begin{align}
    \abs{\mathbb{I}_{j, r}(s) - \mathbb{I}_{j, r}(0)} = \mathbb{I}\curly{\mathbb{I}_{j, r}(s) \neq \mathbb{I}_{j, r}(0)} \leq \frac{2R_w C_{\max}}{\sqrt{2\pi}\Delta \delta}. \label{eq:neuron-pattern}
\end{align}
Combining \eqref{eq:utau-y}, \eqref{eq:xi-bar-1} and \eqref{eq:neuron-pattern},  we deduce with probability at least $1 - 2\delta$ that
\begin{align*}
\abs{\bar{\xi}_{j}^{i}(s)} \lesssim  \frac{\eta\sqrt{dN}R_w C_{\max}^3 D}{\Delta \delta}(1 - \eta \lambda_0)^{s/2} \lesssim \frac{\eta dN^{3/2}C_{\max}^4 D^2}{\sqrt{m}\Delta \delta \lambda_0}(1 - \eta \lambda_0)^{s/2}. 
\end{align*}
Thus, with probability at least $ 1 - 2 \delta $ it holds that
\begin{align}
\norm{\bar{\xi}(s)}_{2} \leq \norm{\bar{\xi}(s)}_{1} = \sum_{j = 1}^{N}\sum_{i = 1}^{d}\abs{\bar{\xi}_{j}^{i}(s)} \lesssim \frac{\eta d^2N^{5/2}C_{\max}^4 D^2}{\sqrt{m}\Delta \delta \lambda_0}(1 - \eta \lambda_0)^{s/2}. 
\label{eq:xi-bar-bound}
\end{align}
Combining \eqref{eq:u-ulin}, \eqref{eq:xi-bound} and \eqref{eq:xi-bar-bound}, the inequality holds with probability at least $ 1 - 3\delta$ that
\begin{align*}
    \norm{u(\tau) - u^{\rm lin}(\tau)}_{2} \leq \frac{\eta d^{5/2}N^3 C_{\max}^4 D^2}{\sqrt{m}\Delta \delta \lambda_0}\sum_{s = 0}^{\tau-1}(1 - \eta \lambda_{0}/2)^{\tau - 1 - s/2} .
\end{align*}
Moreover, we note that
\begin{align*}
    \sum_{s = 0}^{\tau-1}(1 - \eta \lambda_{0}/2)^{\tau - 1 - \frac{s}{2}} &  = (1 - \eta \lambda_{0}/2)^{\frac{\tau - 1}{2}}\sum_{s = 0}^{\tau - 1}(1 - \eta \lambda_{0}/2)^{\frac{\tau - 1-s}{2}} \\ & = (1 - \eta \lambda_{0}/2)^{\frac{\tau - 1}{2}}\sum_{s = 0}^{\tau - 1}(1 - \eta \lambda_{0}/2)^{\frac{s}{2}} \\ &  \leq (1 - \eta \lambda_{0}/2)^{\frac{\tau - 1}{2}}\frac{1}{1 - \sqrt{1 - \eta \lambda_{0}/2}} \\ & \leq \frac{4(1 - \eta\lambda_{0}/2)^{\frac{\tau - 1}{2}}}{\eta \lambda_{0}}.
\end{align*}
where we use the fact that $\sqrt{1 - x} \leq 1 - x/2$. 
Therefore,  with probability at least $ 1 - 3\delta $, it holds~that
\begin{align}
    \norm{u(\tau)- u^{\rm lin}(\tau)}_{2}  \lesssim \frac{d^{5/2}N^3C_{\max}^{4}D^2}{\sqrt{m}\lambda_{0}^{2}\delta\Delta}(1 - \eta\lambda_{0}/2)^{\frac{\tau - 1}{2}}. \label{eq:u-ulin-l2}
\end{align}
Now, substituting (\ref{eq:utau-y}), (\ref{eq:act-pattern-sum}) and (\ref{eq:u-ulin-l2}) back into (\ref{eq:w-bar-w-sum}), we have with probability at least $ 1 - 4\delta $ that 
\begin{align}
\sum_{r = 1}^{m}\norm{w_{r}(\tau) - \bar{w}_{r}(\tau)}_{2} & \lesssim \eta \sqrt{d}C_{\max}\sum_{s = 0}^{\tau-1}(1 - \eta\lambda_{0})^{\frac{s}{2}} 
D \sqrt{\frac{mNR_wC_{\max}}{\Delta \delta}}
\nonumber \\ & \quad + \eta \sqrt{d}C_{\max}\sum_{s = 1}^{\tau-1}\frac{d^{5/2}N^{3}C_{\max}^{4}D^2}{\sqrt{m}\lambda_{0}^{2}\delta\Delta}(1 - \eta\lambda_{0}/2)^{\frac{s - 1}{2}}\sqrt{mN} \nonumber \\ & \lesssim  \frac{m^{1/4}d^{3/4}NC^2_{\max}D^{3/2}}{\sqrt{\Delta \delta}\lambda_0^{3/2}}+ \frac{d^3 N^{7/2}C_{\max}^5D^2}{\lambda_0^3 \delta \Delta} \lesssim \frac{m^{1/4}d^3N^{7/2}C_{\max}^5D^2}{\delta \Delta \lambda_0^{3/2}}.
\label{eq:w-bar-w-final}
\end{align}
Thus, Cauchy-Schwarz inequality implies (\ref{eq:second-term-coupling}) is bounded for all $ \norm{x}_{2} \leq R $ and $ t \in [T_{0} + \Delta, T] $:
\begin{align}
    \frac{1}{\sqrt{m}}\sum_{r = 1}^{m}\abs{(w_{r}(\tau) - \bar{w}_{r}(\tau))^{\top}(x, t - T_{0})}\mathbb{I}_{r}(0) & \leq \frac{C_{\max}}{\sqrt{m}}\sum_{r = 1}^{m}\norm{w_{r}(\tau) - \bar{w}_{r}(\tau)}_{2} \lesssim \frac{d^3N^{7/2}C_{\max}^6D^2}{m^{1/4}\delta \Delta \lambda_0^{3/2}},\label{eq:nn-lin-term2}
\end{align}
with probability at least $ 1 - 4\delta $.
Integrating over (\ref{eq:lin-error}) and combining (\ref{eq:nn-lin-term1}) and (\ref{eq:nn-lin-term2}), with probability at least $ 1 - 5\delta $ it holds that
\begin{align*}
    & \frac{1}{T - T_{0}}\int_{T_{0}+\Delta}^{T}\int_{\norm{x}_{2} \leq R}\abs{f_{\bf{W}(\tau)}^{i}(x, t) - f^{{\rm lin}, i}_{\bar{\bf W}(\tau)}(x, t)}^{2}{\rm d}P_{X_{t}}(x){\rm d}t \\ & \lesssim \frac{d^{3/2}N^{3}C_{\max}^{6}D^3}{\sqrt{m}\delta\lambda_{0}^{3}} + \parenthesis{\frac{d^3N^{7/2}C_{\max}^{6}D^2}{m^{1/4}\lambda_{0}^{3/2}\delta\Delta}}^{2} \lesssim \frac{d^6N^7C_{\max}^{12}D^4}{\sqrt{m}\lambda_0^3\delta^2\Delta^2}.
\end{align*}
As a consequence, with probability at least $ 1 - 5\delta $, we have
\begin{align*}
    \frac{1}{T - T_{0}}\int_{T_{0}+\Delta}^{T}\int_{\norm{x}_{2} \leq R}\norm{f_{\bf{W}(\tau)}(x, t) - f^{{\rm lin}}_{\bar{\bf W}(\tau)}(x, t)}_{2}^{2}{\rm d}P_{X_{t}}(x){\rm d}t \lesssim \frac{d^7 N^7C_{\max}^{12}D^4}{\sqrt{m}\lambda_0^3\delta^2\Delta^2}.
\end{align*}
The proof completes by scaling $\delta$.
\end{proof}

Next, we control the coupling error between the linearized neural network $f^{{\rm lin}}_{\bar{\bf W}(\tau)}$ and the kernel function $f_\tau^K$ defined as in \eqref{eq:kernel-predictor}. Recall the update rule for $ \gamma(\tau) $ is given by
\begin{align}\label{eq:gamma-update}
    \gamma(\tau+1) = \gamma(\tau) - \eta (H \gamma(\tau) - y), \quad \gamma(0) = 0.
\end{align}
Let $u^K(\tau) = H\gamma(\tau)$ be the evaluation of $f_\tau^K$ on all the training data points. 
Consequently, multiplying both sides of the update rule \eqref{eq:gamma-update} by $ H $ leads to
\begin{align*}
    u^{K}(\tau+1) = u^{K}(\tau) - \eta H(u^{K}(\tau) - y), \quad u^{K}(0) = 0.
\end{align*}
We remark that the update rule for $ \gamma $ can be viewed as a GD update rule under an alternative coordinate system. Let $ \omega = \sqrt{H}\gamma $ and define the training objective
\begin{align*}
    \widehat{\mcal{L}}^{K}(\omega) = \frac{1}{2}\norm{u^{K} - y}_{2}^{2} = \frac{1}{2}\norm{\sqrt{H}\omega - y}_{2}^{2}. 
\end{align*}
Here, we use the fact that $ u^{K} = H\gamma = \sqrt{H}\omega $. Thus, the GD update rule for $ \omega $ becomes
\begin{align}\label{eq:update-omega}
    \omega(\tau+1) = \omega(\tau) - \eta \sqrt{H}\parenthesis{u^{K}(\tau) - y}.
\end{align}
Multiplying both sides of (\ref{eq:update-omega}) by $ \sqrt{H^{-1}} $, we recover the update rule of $ \gamma(\tau) $. 

\begin{Lemma}\label{lemma:flin-fk}
    Assume the same conditions as in Theorem \ref{thm:coupling}. Then it holds with probability at least $ 1- \delta $ that
    \begin{align*}
         \frac{1}{T - T_0}\int_{T_{0}+\Delta}^{T}\int_{\norm{x}_{2} \leq R}\norm{f^{\rm lin}_{\bar{\bf W}(\tau)}(x, t) - f_{\tau}^{K}(x, t)}_{2}^{2}{\rm d}P_{X_{t}}(x){\rm d}t = \tilde{\mcal{O}}\parenthesis{\frac{d^4N^4C_{\max}^8D^2}{m\lambda_0^2\delta}}.
    \end{align*}
\end{Lemma}

\begin{proof}
    Note that the gradient of the training loss is
\begin{align*}
\frac{\partial \widehat{\mcal{L}}^{\rm lin}(\bar{\bf W})}{\partial {\rm vec}(\bar{\bf W})} = \frac{\partial}{\partial {\rm vec}(\bar{\bf W})}\frac{1}{2}\norm{u^{\rm lin} - y}_{2}^{2} = {\bf Z}(0)^{\top}(u^{\rm lin} - y).
\end{align*}
It follows for each $ \tau $, there is a vector $ \bar{\gamma}(\tau) \in \R^{dN} $ such that 
\begin{align*}
{\rm vec}(\bar{\bf W}(\tau)) & = {\rm vec}(\bar{\bf W}(\tau-1)) - \eta {\bf Z}(0)^{\top}(u^{\rm lin}(\tau-1) - y) 
\\ & = {\rm vec}(\bar{\bf W}(0)) -  {\bf Z}(0)^{\top}\underbrace{\sum_{s = 0}^{\tau - 1}\eta (u^{\rm lin}(s) - y)}_{\eqqcolon \bar{\gamma}(\tau)}.
\end{align*}
Define a matrix $ {\bf Z}(x, t) \in \R^{d \times m(d+1)} $ such that its $ i $-th row is 
\begin{align*}
    \parenthesis{{\bf Z}^{i}(x, t)}^{\top} \coloneqq \frac{1}{\sqrt{m}}\bracket{a_{1}^{i}(x, t - T_{0})^{\top}\mathbb{I}_{1}(0), \dots,  a_{m}^{i}(x, t - T_{0})^{\top}\mathbb{I}_{m}(0)}.
\end{align*}
As ${\bf Z}(x, t)(\bar{\bf W}(0)) = 0$ by the choice of $\bar{\bf W}(0)$, we rewrite 
\begin{align}
f^{\rm lin}_{\bar{\bf W}(\tau)}(x, t) - f_{\tau}^{K}(x, t) & = {\bf Z}(x, t){\rm vec}(\bar{\bf W}(\tau)) - \sum_{j = 1}^{N}K((X_{t_{j}}, t_{j}), (x, t))\gamma_{j}(\tau) \nonumber \\ & = {\bf Z}(x, t){\bf Z}(0)^{\top}\bar{\gamma}(\tau) - \sum_{j = 1}^{N}K((X_{t_{j}}, t_{j}), (x, t))\gamma_{j}(\tau) \nonumber \\ & = {\bf Z}(x, t){\bf Z}(0)^{\top}\bar{\gamma}(\tau) - \hat{K}(x, t)\gamma(\tau) \nonumber\\ & = {\bf Z}(x, t){\bf Z}(0)^{\top}\parenthesis{\bar{\gamma}(\tau) - \gamma(\tau)} - \parenthesis{{\bf Z}(x, t){\bf Z}(0)^{\top} - \hat{K}(x, t)}\gamma(\tau), \label{eq:flin-fk}
\end{align}
in which we define
\begin{align*}
    \hat{K}(x, t) \coloneqq [K((X_{t_{1}}, t_{1}), (x, t)), \dots, K((X_{t_{N}}, t_{N}), (x, t))], \quad \gamma(\tau) \coloneqq [\gamma_{1}^{\top}(\tau), \dots, \gamma_{N}^{\top}(\tau)]^{\top}.
\end{align*}
Taking square on both sides of (\ref{eq:flin-fk}), we obtain
\begin{align}
& \norm{f^{\rm lin}_{\bar{\bf W}(\tau)}(x, t) - f_{\tau}^{K}(x, t)}_{2}^{2}\nonumber \\ & \leq 2\norm{{\bf Z}(x, t){\bf Z}(0)^{\top}\parenthesis{\bar{\gamma}(\tau) - \gamma(\tau)}}_{2}^{2} + 2\norm{\parenthesis{{\bf Z}(x, t){\bf Z}(0)^{\top} - \hat{K}(x, t)}\gamma(\tau)}_{2}^{2} \nonumber\\ & \leq 2\norm{{\bf Z}(x, t){\bf Z}(0)^{\top}}_{2}^{2}\norm{{\bar{\gamma}(\tau) - \gamma(\tau)}}_{2}^{2} + 2\norm{{{\bf Z}(x, t){\bf Z}(0)^{\top} - \hat{K}(x, t)}}_{2}^{2}\norm{\gamma(\tau)}_{2}^{2}. \label{eq:flin-fk-2}
\end{align}
Since $ H(0) = {\bf Z}(0){\bf Z}(0)^{\top} $ and the Gram matrix of kernel function $ K $ is $ H $, we have
\begin{align}
u^{\rm lin}(\tau) - u^{K}(\tau)  & = H(0)\bar{\gamma}(\tau) - H\gamma(\tau) \\ & = H(0)(\bar{\gamma}(\tau) - \gamma(\tau)) + (H(0) - H)\gamma(\tau). \label{eq:u-gamma}
\end{align}

To obtain an upper bound of $\norm{\bar{\gamma}(\tau) - \gamma(\tau)}_2$, we first upper bound $ \norm{u^{\rm lin}(\tau) - u^{K}(\tau)}_{2} $. The GD update rules imply
\begin{align*}
u^{\rm lin}(\tau+1) &= u^{\rm lin}(\tau) - \eta H(0)(u^{\rm lin}(\tau) - y), \\ 
u^{K}(\tau+1) &= u^{K}(\tau) - \eta H(u^{K}(\tau) - y), 
\end{align*}
with $ u^{\rm lin}(0) = u^{K}(0) = 0 $. It follows
\begin{align}
u^{\rm lin}(\tau+1) - u^{K}(\tau+1) & = u^{\rm lin}(\tau) - u^{K}(\tau) -\eta (H - H(0))(u^{K}(\tau) - y) \nonumber \\ & \qquad - \eta H(0)(u^{\rm lin}(\tau) - u^{K}(\tau)) \nonumber \\ & = (I_{dN} - \eta H(0))(u^{\rm lin}(\tau) - u^{K}(\tau)) - \eta(H - H(0))(u^{K}(\tau) - y) \label{eq:recursion-ulin-uk}.
\end{align}
Unrolling (\ref{eq:recursion-ulin-uk}), we have
\begin{align*}
u^{\rm lin}(\tau) - u^{K}(\tau) & = (I_{dN} - \eta H(0))^{\tau}(u^{\rm lin}(0) - u^{K}(0)) \\ & \qquad - \eta \sum_{s = 0}^{\tau-1}(I_{dN} - \eta H(0))^{\tau - 1 - s}(H - H(0))(u^{K}(s) - y) \\ & = - \eta \sum_{s = 0}^{\tau-1}(I_{dN} - \eta H(0))^{\tau - 1 - s}(H - H(0))(u^{K}(s) - y).
\end{align*}
Taking 2-norm on both sides, we have
\begin{align*}
\norm{u^{\rm lin}(\tau) - u^{K}(\tau)}_{2} & \leq \eta\norm{H - H(0)}_{2}\sum_{s = 0}^{\tau - 1}\norm{I_{dN} - \eta H(0)}_{2}^{\tau - 1 - s}\norm{u^{K}(s) - y}_{2} \\ & \leq \eta\norm{H - H(0)}_{2}\sum_{s = 0}^{\tau - 1}\parenthesis{1 - \frac{\eta \lambda_{0}}{2}}^{\tau - 1 - s}\norm{u^{K}(s) - y}_{2} \\ & \leq \eta\norm{H - H(0)}_{2}\max_{0 \leq s \leq \tau-1}\norm{u^{K}(s) - y}_{2}\sum_{s = 0}^{\tau - 1}\parenthesis{1 - \frac{\eta \lambda_{0}}{2}}^{\tau - 1 - s} \\ & \leq \frac{2}{\lambda_0}\norm{H - H(0)}_{2}\max_{0 \leq s \leq \tau-1}\norm{u^{K}(s) - y}_{2}.
\end{align*}
Note that the maximum is achieved at $\tau = 0$. Consequently, 
\begin{align}
    \max_{0 \leq s \leq \tau-1}\norm{u^{K}(s) - y}_{2} & = \norm{u^{K}(0) - y}_{2} = \norm{u(0) - y}_{2} = \sqrt{N}D \label{eq:max-uk-y}.
\end{align}
With (\ref{eq:max-uk-y}), we deduce that the following holds
\begin{align}
\norm{u^{\rm lin}(\tau) - u^{K}(\tau)}_{2} & \leq \frac{2\sqrt{N}D}{\lambda_0} \norm{H - H(0)}_2. \label{eq:ulin-uk-bound}
\end{align}

We next turn to bound $ \norm{\gamma(\tau)}_{2} $. The update rule \eqref{eq:gamma-update} leads to
\begin{align*}
\gamma(\tau+1) = \gamma(\tau) - \eta (H\gamma(\tau) - y) = (I_{dN} - \eta H)\gamma(\tau) + \eta y.
\end{align*}
Unrolling the recursive formula, we have
\begin{align}
\gamma(\tau) = (I_{dN} - \eta H)^{\tau}\gamma(0) + \eta \sum_{s = 0}^{\tau - 1}(I_{dN} - \eta H)^{s}y = \eta \sum_{s = 0}^{\tau - 1}(I_{dN} - \eta H)^{s}y. \label{eq:gamma-update-unroll}
\end{align}
Note that
\begin{align*}
\sum_{s = 0}^{\tau - 1}(I_{dN} - \eta H)^{s} = (I_{dN} - (I_{dN} - \eta H)^{\tau})(\eta H)^{-1} \preceq \eta^{-1}H^{-1},
\end{align*}
where we choose $\eta$ small enough so that $I_{dN} - \eta H$ is positive definite.
Taking 2-norm on both sides of \eqref{eq:gamma-update-unroll}, we deduce
\begin{align}
\norm{\gamma(\tau)}_{2} & \leq \eta \norm{\sum_{s = 0}^{\tau - 1}(I_{dN} - \eta H)^{s}}_{2}\norm{y}_{2} \leq \norm{H^{-1}}_{2}\norm{y}_{2}  \lesssim \frac{\sqrt{N}D}{\lambda_0}.\label{eq:gamma-bound}
\end{align}
Consequently, we combine \eqref{eq:H-H0}, \eqref{eq:u-gamma}, \eqref{eq:ulin-uk-bound} and \eqref{eq:gamma-bound} to have
\begin{align}
\frac{\lambda_{0}}{2}\norm{\bar{\gamma}(\tau) - \gamma(\tau)}_{2} & \leq \norm{H(0)}_2\norm{\bar{\gamma}(\tau) - \gamma(\tau)}_{2}\nonumber \\ & \leq \norm{u^{\rm lin}(\tau) - u^K(\tau)}_2 + \norm{H - H(0)}_2\norm{\gamma(\tau)}_2 \nonumber\\ & \lesssim \frac{\sqrt{N}D}{\lambda_0}\norm{H - H(0)}_2 = \tilde{O}\parenthesis{\frac{dN^{3/2}C_{\max}^2D}{\lambda_0 \sqrt{m}}},\label{eq:bar-gamma-bound}
\end{align}
which holds with probability at least $ 1- \delta$. 

With all the above results, we are ready to bound \eqref{eq:flin-fk-2} for all $ \norm{x}_{2} \leq R $ and $ t \in [T_{0} + \Delta, T] $. Note that
\begin{align}
\norm{{\bf Z}(x, t)}_{2}^{2} \leq \sum_{i = 1}^{d}\norm{{\bf Z}^{i}(x, t)}_{2}^{2} = \sum_{i = 1}^{d}\sum_{r = 1}^{m}\norm{\frac{1}{\sqrt{m}}a_{r}^{i}(x^{\top}, t - T_{0})\mathbb{I}_{r}(0)}_{2}^{2} \leq dC_{\max}^{2}. \label{eq:z-bound}
\end{align}
and
\begin{align}
\norm{{\bf Z}(0)}_{2}^{2} \leq \sum_{i = 1}^{d}\sum_{j = 1}^{N}\norm{{\bf Z}_{j}^{i}(0)}_{2}^{2} = \sum_{i = 1}^{d}\sum_{j = 1}^{N}\sum_{r = 1}^{m}\norm{\frac{1}{\sqrt{m}}a_{r}^{i}z_{j}^{\top}\mathbb{I}_{j, r}(0)}_{2}^{2} \leq dNC_{\max}^{2}.\label{eq:z0-bound}
\end{align}
We integrate \eqref{eq:flin-fk-2} over $ \norm{x}_{2} \leq R $ and $ t \in [T_{0} + \Delta, T] $ and apply \eqref{eq:gamma-bound}, \eqref{eq:bar-gamma-bound}, \eqref{eq:z-bound} and \eqref{eq:z0-bound} to have
\begin{align}
& \int_{T_{0}+\Delta}^{T}\int_{\norm{x}_{2} \leq R}\norm{f^{\rm lin}_{\bar{\bf W}(\tau)}(x, t) - f_{\tau}^{K}(x, t)}_{2}^{2}{\rm d}P_{X_{t}}(x){\rm d}t\nonumber \\ & \lesssim \int_{T_{0} +\Delta}^{T}\int_{\norm{x}_{2} \leq R}d^{2}NC_{\max}^{4}\cdot \frac{d^2N^3C_{\max}^4 D^2}{\lambda_0^4 m} {\rm d}P_{X_{t}}(x){\rm d}t  \nonumber\\ & \qquad + \frac{ND^2}{\lambda_0^2}\int_{T_{0} +\Delta}^{T}\int_{\norm{x}_{2} \leq R}\norm{{{\bf Z}(x, t){\bf Z}(0)^{\top} - \hat{K}(x, t)}}_{2}^{2}  {\rm d}P_{X_{t}}(x){\rm d}t  \nonumber\\ & \leq \frac{d^4N^4C_{\max}^8 D^2}{\lambda_0^4 m}(T - T_0 - \Delta) + \frac{ND^2}{\lambda_0^2}\int_{T_{0} +\Delta}^{T}\int_{\norm{x}_{2} \leq R}\norm{{{\bf Z}(x, t){\bf Z}(0)^{\top} - \hat{K}(x, t)}}_{2}^{2}  {\rm d}P_{X_{t}}(x){\rm d}t, \label{eq:flin-fk-3}
\end{align}
which holds with probability at least $1 - \delta$. 
Note that for each $ i, k, j $, we can write
\begin{align*}
    \parenthesis{{\bf Z}(x, t){\bf Z}(0)^{\top}}_{j}^{ik} = \frac{1}{m}\sum_{r = 1}^{m}a_{r}^{i}a_{r}^{k}(X_{t_{j}}, t_{j} - T_{0})^{\top}(x, t - T_{0})\mathbb{I}_{j, r}(0)\mathbb{I}_{r}(0),
\end{align*}
which is a summation of $m$ independent random variables bounded by $ C_{\max}^{2}/m $ when $ \norm{x}_{2} \leq R $ and $ t\in [T_{0} + \Delta, T] $.
Taking expectation over the initialization, we have
\begin{align*}
    & \E\bracket{\abs{{\parenthesis{{\bf Z}(x, t){\bf Z}(0)^{\top}}_{j}^{jk} - \hat{K}_{j}^{ik}(x, t)}}_{2}^{2}} = {\rm Var}\parenthesis{{\parenthesis{{\bf Z}(x, t){\bf Z}(0)^{\top}}_{j}^{jk} }} \lesssim \frac{C_{\max}^{4}}{m}.
\end{align*}
Integrating over $ x $ and $ t $ gives us
\begin{align*}
\int_{T_{0}+\Delta}^{T}\int_{\norm{x}_{2} \leq R}\E\bracket{\abs{{\parenthesis{{\bf Z}(x, t){\bf Z}(0)^{\top}}_{j}^{jk} - \hat{K}_{j}^{ik}(x, t)}}^{2}}{\rm d}P_{X_{t}}(x){\rm d}t  \lesssim \frac{C_{\max}^{4}}{m}(T - T_{0} - \Delta).
\end{align*}
The Fubini's theorem and the Markov inequality imply that, with probability at least $ 1 - \delta/(d^{2}N) $,
\begin{align*}
\int_{T_{0}+\Delta}^{T}\int_{\norm{x}_{2} \leq R}\abs{{\parenthesis{{\bf Z}(x, t){\bf Z}(0)^{\top}}_{j}^{jk} - \hat{K}_{j}^{ik}(x, t)}}^{2}{\rm d}P_{X_{t}}(x){\rm d}t \lesssim \frac{C_{\max}^{4}d^{2}N}{m\delta}(T - T_{0} - \Delta). 
\end{align*}
and consequently with probability at least $ 1 - \delta$, 
\begin{align}
\int_{T_{0}+\Delta}^{T}\int_{\norm{x}_{2} \leq R}\norm{{\parenthesis{{\bf Z}(x, t){\bf Z}(0)^{\top}} - \hat{K}(x, t)}}_{2}^{2}{\rm d}P_{X_{t}}(x){\rm d}t \lesssim \frac{C_{\max}^{4}d^{4}N^2}{m\delta}(T - T_{0} - \Delta).  \label{eq:z-khat-bound}
\end{align}
Therefore, we combine \eqref{eq:flin-fk-3} and \eqref{eq:z-khat-bound} to conclude with probability at least $ 1- 2\delta$ that
\begin{align*}
    & \frac{1}{T - T_0}\int_{T_{0}+\Delta}^{T}\int_{\norm{x}_{2} \leq R}\norm{f^{\rm lin}_{\bar{\bf W}(\tau)}(x, t) - f_{\tau}^{K}(x, t)}_{2}^{2}{\rm d}P_{X_{t}}(x){\rm d}t \\ & \lesssim \frac{d^4N^4C_{\max}^8 D^2}{\lambda_0^4 m}(T - T_0 - \Delta) + \frac{ND^2}{\lambda_0^2}\cdot \frac{C_{\max}^{4}d^{4}N^2}{m\delta}(T - T_{0} - \Delta)  \\ & = \frac{d^4N^4C_{\max}^8D^2}{m\lambda_0^2\delta}(T-T_0-\Delta),
\end{align*}
which finishes the proof by scaling $\delta$. 

\end{proof}

Now we are ready to prove Theorem \ref{thm:coupling}.

\begin{proof}[Proof of Theorem \ref{thm:coupling}]
    Note that 
    \begin{align*}
    \norm{f_{\bf{W}(\tau)}(x, t) - f_{\tau}^{K}(x, t)}_{2}^{2} \leq 2\norm{f_{\bf{W}(\tau)}(x, t) - f^{{\rm lin}}_{\bar{\bf W}(\tau)}(x, t)}_{2}^{2} + 2\norm{f^{{\rm lin}}_{\bar{\bf W}(\tau)}(x, t) - f_{\tau}^{K}(x, t)}_{2}^{2}.
\end{align*}
Lemmas \ref{lemma:f-flin} and \ref{lemma:flin-fk} imply that with probability at least $ 1- 2\delta$, it holds simultaneously over all $ \tau \geq 0 $ that
\begin{align*}
    & \frac{1}{T - T_{0}}\int_{T_{0}+\Delta}^{T}\int_{\norm{x}_{2} \leq R}\norm{f_{\bf{W}(\tau)}(x, t) - f_{\tau}^{K}(x, t)}_{2}^{2}{\rm d}P_{X_{t}}(x){\rm d}t \\ & \leq \frac{2}{T - T_{0}}\int_{T_{0}+\Delta}^{T}\int_{\norm{x}_{2} \leq R}\norm{f_{\bf{W}(\tau)}(x, t) - f^{{\rm lin}}_{\bar{\bf W}(\tau)}(x, t)}_{2}^{2}{\rm d}P_{X_{t}}(x){\rm d}t \\ & \qquad + \frac{2}{T - T_{0}}\int_{T_{0}+\Delta}^{T}\int_{\norm{x}_{2} \leq R}\norm{f^{{\rm lin}}_{\bar{\bf W}(\tau)}(x, t) - f_{\tau}^{K}(x, t)}_{2}^{2}{\rm d}P_{X_{t}}(x){\rm d}t \\ & \lesssim \frac{d^7 N^7C_{\max}^{12}D^4}{\sqrt{m}\lambda_0^3\delta^2\Delta^2} + \frac{d^4N^4C_{\max}^8D^2}{m\lambda_0^2\delta} \lesssim \frac{d^7 N^7C_{\max}^{12}D^4}{\sqrt{m}\lambda_0^2\delta^2\Delta^2}.
\end{align*}
We finish the proof by scaling $\delta$.   
\end{proof}

\section{Auxiliary Results for Theorem \ref{thm:mismatch}}\label{sec:mismatch}
This appendix is devoted to the proofs of auxiliary lemmas used in the proof for Theorem \ref{thm:mismatch}.

\begin{Lemma}\label{lemma:mismatch-train}
    Assume the same conditions as in Theorem \ref{thm:approx-score-L2}. The following upper bound holds at each iteration $\tau$:
    \begin{align*}
        \norm{u^{K}(\tau) - \tilde{u}^{K}(\tau)}_{2}^{2} \leq dN A^{2}(R_{\mcal{H}}, R).
    \end{align*}
    \end{Lemma}

\begin{proof}
Note that the GD update rule leads to 
\begin{align*}
    u^{K}(\tau+1) &= u^{K}(\tau) - \eta H(u^{K}(\tau) - y) \\ & = (I_{dN} - \eta H)u^{K}(\tau) + \eta Hy \\ & = \eta \sum_{s = 0}^{\tau}(I_{dN} - \eta H)^{s}Hy \\ & = (I_{dN} - (I_{dN} - \eta H)^{\tau+1})y.
\end{align*}
Here, we have used the fact that $u^K(0) = 0$. 
Moreover, we have a similar result for $ \tilde{u}^{K}(\tau) $:
\begin{align*}
    \tilde{u}^{K}(\tau+1) = (I_{dN} - (I_{dN} - \eta H)^{\tau+1})\tilde{y}.
\end{align*}
It follows that 
\begin{align*}
u^{K}(\tau) - \tilde{u}^{K}(\tau) = (I_{dN} - (I_{dN} - \eta H)^{\tau})(y - \tilde{y}).
\end{align*}
We take 2-norm on both sides to obtain
\begin{align*}
\norm{u^{K}(\tau) - \tilde{u}^{K}(\tau)}_{2}^{2} & = \norm{(I_{dN} - (I_{dN} - \eta H)^{\tau})(y - \tilde{y})}_{2}^{2} \\ & \leq \norm{I_{dN} - (I_{dN} - \eta H)^{\tau}}_{2}^{2}\norm{y - \tilde{y}}_{2}^{2} \\ & \leq \sum_{j=1}^N\norm{y_j - \tilde{y}_j}_2^2.
\end{align*}
Recall that $y_j = f_{*}(X_{t_j}, t_j)$ and $\tilde{y}_j = f_{\mcal{H}}(X_{t_j}, t_j) $. It follows that
\begin{align*}
    \norm{u^{K}(\tau) - \tilde{u}^{K}(\tau)}_{2}^{2} & \leq \sum_{j = 1}^{N}\norm{f_{*}(X_{t_j}, t_j) - f_{\mcal{H}}(X_{t_j}, t_j)}_{2}^{2} \\ & \leq d\sum_{j = 1}^{N}\norm{f_{*}(X_{t_j}, t_j) - f_{\mcal{H}}(X_{t_j}, t_j)}_{\infty}^{2} \\ & \leq dN\sup_{\norm{x}_{\infty} \leq R}\sup_{t \in [T_{0}, T]}\norm{f_{*}(x, t) - f_{\mcal{H}}(x, t)}_{\infty}^{2} \leq dN A^2(R_{\mcal{H}}, R).
\end{align*}
Here, we obtain the second inequality by applying Lemma \ref{lemma:sampling-prob} that $ \norm{X_{t_{j}}}_{2} \leq R $ and $ t_{j} \in [T_{0} + \Delta, T] $ with probability at least $1 - \delta_1$. The last inequality follows from~\eqref{eq:approx-inf}.

\end{proof}

\section{Auxiliary Results for Theorem \ref{thm:early-stop}}\label{sec: proof bias-var}
This appendix is devoted to the proofs of auxiliary lemmas used in the proof of Theorem \ref{thm:early-stop}. Recall the localized empirical Rademacher complexity defined as in \eqref{eq:def-stop-rule}. Define the critical empirical radius $ \widehat{r}_{N} > 0 $ as the smallest positive solution to the inequality 
\begin{align}
\widehat{R}_{\kappa}(r) \leq \frac{r^{2}}{2eD}. \label{eq:def_rad_emp}
\end{align}
\begin{Lemma}[Bounds on Bias and Variance]\label{lemma:bias-variance}
    Assume the same assumptions as in Thoerem \ref{thm:early-stop}.
    At each iteration $\tau = 1, 2, \dots$, the squared bias is upper bounded with
    \begin{align}\label{eq:upper-B}
    (B_{\tau}^{i})^{2} \leq  \frac{R_\mathcal{H}}{e\tau \eta N},
    \end{align}
    where we recall $R_\mathcal{H}$ defined in Theorem \ref{thm:approx-score-L2}.
    Moreover, there is a constant $c' > 0$ such that for all $\tau = 1,\dots, \mcal{T}$, it holds with probability at least $ 1 - \exp(-c'N\widehat{r}_{N}^{2}/D^{2}) $ that
    \begin{align}\label{eq:upper-V}
    V_{\tau}^{i} \leq \frac{5}{e^{2} \eta \tau N}.
    \end{align} 
   
\end{Lemma}

\begin{proof}
We start with the upper bound for the squared bias term $(B_{\tau}^{i})^{2}$. Denote $ \widehat{\lambda}_{1} \geq \dots \geq \widehat{\lambda}_{N} \geq \lambda_{0} $ as the eigenvalues of the Gram matrix $ H^{ii} $. It follows $ \Lambda = {\rm diag}(\widehat{\lambda}_{1}, \dots, \widehat{\lambda}_{N}) $ by definition. 
Note that
\begin{align*}
(S(\tau))_{jj}^{2} = (1 - \eta \widehat{\lambda}_{j})^{2\tau} \leq \exp\parenthesis{-2\tau\eta \widehat{\lambda}_{j}} \leq \frac{1}{2e\tau\eta\widehat{\lambda}_{j}},
\end{align*}
where we use the inequality $ 1 - u \leq e^{-u} $ and the fact  that $ \sup_{u \in \R}\curly{u \exp(-u)} = 1/e $.
Then, the squared bias is upper-bounded with
\begin{align}
\frac{2}{N}\sum_{j = 1}^{N}(S(\tau))_{jj}^{2}(V^{\top} u^{i}_{\mathcal{H}})_{j}^{2} \leq \frac{1}{e\tau \eta N}\sum_{j = 1}^{N}\frac{(V^{\top}u^{i}_{\mathcal{H}})_{j}^{2}}{\widehat{\lambda}_{j}}. \label{eq:bias_upper_1}
\end{align}
where we recall the notation $u_{\mcal{H}}^{i} \in \R^N$ with $j$-th coordinate  $(u_{\mcal{H}}^{i})_j = f_{\mcal{H}}^{i}(X_{t_{j}}, t_{j}) $.
Recall that over the compact domain $\curly{(x, y): \norm{x}_2 \leq R, t \in [T_{0}, T]} $, Mercer's theorem \citep{wainwright2019high} guarantees the kernel $ \kappa $ admits an eigen-decomposition of the form:
\begin{align*}
\kappa((x, t), (x', t')) = \sum_{p = 1}^{\infty}\mu_{p}\phi_{p}(x, t)\phi_{p}(x', t'),
\end{align*}
where $ \mu_{1} \geq \mu_{2} \geq \dots \geq 0 $ are non-negative sequence of eigenvalues, and $ \curly{\phi_{p}}_{p = 1}^{\infty} $ form an orthonormal basis of $ L^{2}(\mathbb{B}_{R}^{2} \times [T_{0}, T]; {\rm d}P_{X_{t}} \otimes \frac{{\rm d}t}{T - T_{0}}) $. Here, $\mathbb{B}_{R}^{2}$ is the Euclidean ball in $\R^d$ with radius $R$ centered at zero.  Let $\ell^2(\N)$ be the set of all square-integrable real-valued sequences. Similar to \cite[Section 4.1.1]{raskutti2014early}, consider the linear operator $ \Phi: \ell^{2}(\N) \to \R^{N} $ defined element-wise via $ [\Phi]_{jp} = \phi_{p}(X_{t_{j}}, t_{j}) $. Moreover, let $ D:\ell^{2}(\N) \to \ell^{2}(\N) $ be a diagonal linear operator with diagonal entries $ [D]_{pp} = \mu_{p} $. As a consequence of Mercer's theorem, we can find a sequence $ \beta^{i} \in \ell^{2}(\N) $ such that
\begin{align}
u^{i}_{\mathcal{H}} = \Phi D^{1/2}\beta^{i}.\label{eq:u_H-indentity}
\end{align}
With this identity, we can re-write $ H^{ii} = \Phi D \Phi^{\top} $. Recall that we already have $ H^{ii} = V \Lambda V^{\top} $. Since $ D^{1/2}\Phi^{\top} $ is a compact operator, there is a linear operator $ \Psi: \R^{N} \to \ell^{2}(\N) $  with adjoint $ \Psi^{*} $ such that $ \Psi^{*}\Psi = I_{N} $ and the following holds
\begin{align}
\Phi D^{1/2} = V\Lambda^{1/2} \Psi^{*}. \label{eq:operator}
\end{align}
With identities \eqref{eq:u_H-indentity} and \eqref{eq:operator}, the RHS of \eqref{eq:bias_upper_1} becomes
\begin{align*}
    \frac{1}{e\tau \eta N}\sum_{j = 1}^{N}\frac{(V^{\top}u^{i}_{\mathcal{H}})_{j}^{2}}{\widehat{\lambda}_{j}} & = \frac{1}{e\tau \eta N}\sum_{j = 1}^{N}\frac{(V^{\top}\Phi D^{1/2}\beta^{i})_{j}^{2}}{\widehat{\lambda}_{j}}  = \frac{1}{e\tau \eta N}\sum_{j = 1}^{N}\frac{(V^{\top}V\Lambda^{1/2} \Psi^{*}\beta^{i})_{j}^{2}}{\widehat{\lambda}_{j}}.
\end{align*}
Since $V$ is orthonormal and $\Lambda_{jj} = \widehat{\lambda}_j$, we deduce
\begin{align*}
    \frac{1}{e\tau \eta N}\sum_{j = 1}^{N}\frac{(V^{\top}u^{i}_{\mathcal{H}})_{j}^{2}}{\widehat{\lambda}_{j}} & = \frac{1}{e\tau \eta N}\sum_{j = 1}^{N}\frac{(\Lambda^{1/2} \Psi^{*}\beta^{i})_{j}^{2}}{\widehat{\lambda}_{j}} \\ & = \frac{1}{e\tau \eta N}\sum_{j = 1}^{N}\frac{\widehat{\lambda}_{j}(\Psi^{*}\beta^{i})_{j}^{2}}{\widehat{\lambda}_{j}} \\ & = \frac{1}{e\tau \eta N}\norm{\Psi^{*}\beta^{i}}_{2}^{2} \\ & = \frac{1}{e\tau \eta N}\norm{\beta^{i}}_{2}^{2}.
\end{align*}
Here, the ultimate equality follows from the fact that $\Psi^*$ is an isometry.  
Applying the Mercer's theorem again, we have
\begin{align*}
\norm{\beta^{i}}_{2}^{2} = \norm{ f^{i}_{\mcal{H}}}_{\kappa}^{2} \leq R_\mathcal{H}.
\end{align*}
where we use the assumption $ \norm{f_{\mcal{H}}^{i}}^{2}_{\kappa} \leq R_{\mcal{H}} $ in Theorem \ref{thm:early-stop}. 
Hence, the squared bias term is upper bounded with
\begin{align*}
(B_{\tau}^{i})^{2} = \frac{2}{N}\sum_{j = 1}^{N}(S(\tau))_{jj}^{2}(V^{\top}u_{\mcal{H}}^{i})_{j}^{2} \leq \frac{\norm{\beta^{i}}_{2}^{2}}{e\tau \eta N} \leq  \frac{R_\mathcal{H}}{e\tau \eta N}.
\end{align*}

Next, we turn to establish an upper bound for the variance term $ V_{\tau}^{i} $. Consider the diagonal matrix $ Q \coloneqq {\rm diag}\curly{(1 - S_{jj}(\tau))^2, j = 1, \dots, N} $. We re-write the variance as
\begin{align*}
V_{\tau}^{i} = \frac{2}{N}\sum_{j = 1}^{N}(1 - S_{jj}(\tau))^{2}(V^{\top}\varepsilon^{i})_{j}^{2} = \frac{2}{N}(V^\top \varepsilon^i)^{\top}Q(V^\top \varepsilon^i) = \frac{2}{N}{\rm Tr}\parenthesis{VQV^{\top}\varepsilon^{i}\parenthesis{\varepsilon^{i}}^{\top}}. 
\end{align*}
Given the input dataset $\curly{(X_{t_j}, t_j)}_{j=1}^N$, taking expectation on both sides leads to
\begin{align*}
\E\bracket{V_{\tau}^{i}} = \frac{2}{N}\Tr(VQV^{\top}\E\bracket{\varepsilon^{i}(\varepsilon^{i})^{\top}}) = \frac{2D^{2}}{N}\Tr\parenthesis{VQV^{\top}\E\bracket{\varepsilon^{i}(\varepsilon^{i})^{\top}/D^{2}}}.
\end{align*}
Recall that $ \curly{\varepsilon^{i}_j}_{j = 1}^{N} $ are independent zero mean random variables with $ \abs{\varepsilon^{i}_j} \leq D $.  It follows $ \E\bracket{\varepsilon^{i}(\varepsilon^{i})^{\top}/D^{2}}  \preceq I_{d} $. Since $ VQV^{\top} $ is a PSD matrix, the sub-multiplicative property of trace results in
\begin{align}
\E\bracket{V_{\tau}^{i}} \leq \frac{2D^{2}}{N}{\rm Tr}\curly{VQV^{\top} I_{d}} = \frac{2D^{2}}{N}{\rm Tr}\parenthesis{QV^\top V} = \frac{2D^{2}}{N}{\rm Tr}\parenthesis{Q}. \label{eq:v-qudra}
\end{align}
Since $ 1 - S_{jj}(\tau) \leq 1 - \max\curly{0, 1 - \eta \tau \widehat{\lambda}_{j}} \leq \min\curly{1, \eta \tau \widehat{\lambda}_{j}} $, we deduce that
\begin{align}
\frac{1}{N}{\rm Tr}(Q) = \frac{1}{N}\sum_{j = 1}^{N}(1 - S_{jj}(\tau))^{2} \leq \frac{1}{N}\sum_{j = 1}^{N}\min\curly{1, \parenthesis{\eta \tau \widehat{\lambda}_{j}}^{2}} \leq \frac{1}{N}\sum_{j = 1}^{N}\min\curly{1, \eta \tau \widehat{\lambda}_{j}}. \label{eq:tr-Q}
\end{align}
Recall the localized empirical Rademacher complexity defined as in \eqref{eq:def-stop-rule}. It follows from \eqref{eq:tr-Q} and \eqref{eq:def-stop-rule} that $ \frac{1}{N}{\rm Tr}(Q) \leq \eta{\tau}N \widehat{R}^2_{\kappa}(1/\sqrt{\eta{\tau}N}) $. Consequently, we obtain from \eqref{eq:v-qudra} an upper bound for the variance: 
\begin{align*}
\E\bracket{V_{\tau}^{i}} \leq 2D^2N\eta\tau\widehat{R}^{2}_{\kappa}(1/\sqrt{\eta\tau N}). 
\end{align*}
To proceed, we need a bound on the two-sided tail probability $ \P\parenthesis{\abs{V_{\tau}^{i} - \E\bracket{V_{\tau}^{i}}} \geq \delta} $.  Consider $ X = (X_1, \dots, X_N)^\top$ with i.i.d.~components that are zero-mean and sub-Gaussian random variables with parameter $K$. Given a matrix $A \in \R^{N\times N}$, \cite[Theorem 1.1]{rudelson2013hanson} proves that there exists a numerical constant $ c $ such that
\begin{align*}
\P\parenthesis{X^\top AX - \E\bracket{X^\top AX} \geq \delta} \leq \exp\parenthesis{-c\min\curly{\frac{\delta}{K^2\norm{A}_{2}}, \frac{\delta^{2}}{K^4\norm{A}_{F}^{2}}}},
\end{align*}
We apply the result with $ A = \frac{2}{N}VQV^{\top} $ and $ X = \varepsilon^i $. Note that
\begin{align*}
\norm{A}_{2} & = \frac{2}{N}\norm{VQV^\top}_{2} = \frac{2}{N}\norm{Q}_2 \leq \frac{2}{N}, \quad {\rm and} \\ 
\norm{A}_{F}^{2} & = \frac{4}{N^{2}}{\rm Tr}\parenthesis{VQV^{\top}VQV^{\top}} = \frac{4}{N^{2}}{\rm Tr}(Q^{2}) \leq \frac{4}{N^{2}}{\rm Tr}(Q) \leq 4\eta\tau\widehat{R}^{2}_{\kappa}(1/\sqrt{\eta\tau N}).
\end{align*}
Consequently, by letting $K = D$, we deduce that
\begin{align}
\P\parenthesis{V_{\tau}^{i} - \E\bracket{V_{\tau}^{i}} \geq \delta} & \leq \exp\parenthesis{-c\min\curly{\frac{\delta N}{2D^2}, \frac{\delta^{2}}{4D^4}\parenthesis{\eta\tau\widehat{R}^2_{\kappa}(1/\sqrt{\eta\tau N})}^{-1}}} \nonumber \\ & = \exp\parenthesis{-\frac{c\delta }{4D^2}\min\curly{2N, \frac{\delta}{D^2} \parenthesis{\eta\tau\widehat{R}^2_{\kappa}(1/\sqrt{\eta\tau N})}^{-1}}}. \label{eq:V-concentration-1}
\end{align}
Recall the definition of the stopping time
\begin{align*}
\mathcal{T} = \argmin\curly{\tau \in \N: \widehat{R}_{\kappa}(1/\sqrt{\eta\tau N}) > (2e D \eta\tau N)^{-1}} - 1.
\end{align*}
It was shown that the integer $ \mcal{T} $ is finite and unique \citep{raskutti2014early}. 
For any $ \tau \leq \mcal{T} $, it follows from \eqref{eq:V-concentration-1} that
\begin{align}
\P(V_{\tau}^{i} - \E\bracket{V_{\tau}^{i}} \geq \delta) \leq \exp\parenthesis{-\frac{c\delta }{4D^{2}}\min\curly{2N, 4e^{2}\eta \tau N^2 \delta}}. \label{eq:V-concentration-3}
\end{align}
Setting $ \delta = 3/(e^{2}\eta \tau N) $ in \eqref{eq:V-concentration-3} leads to
\begin{align}
\P(V_{\tau}^{i} - \E\bracket{V_{\tau}^{i}} \geq 3/(e^{2}\eta\tau )) \leq \exp\parenthesis{-\frac{c' N}{D^{2}\eta \tau N}}, \label{eq:V-concentration-2}
\end{align}
where $ c' $ is a numerical constant. Recall that $\widehat{r}_N$ is defined as the smallest positive solution satisfying \eqref{eq:def_rad_emp}.
Equivalently, we have
\begin{align*}
\widehat{r}_{N} & = \argmin\curly{r > 0: \widehat{R}_{\kappa}(r) \leq r^{2}/(2eD)} \\ & = \argmin\curly{r > 0: \sum_{j = 1}^{N}\min\curly{\frac{r^{-2}\widehat{\lambda}_{j}}{N}, 1} \leq Nr^{2}/(4e^{2}D^{2})}. 
\end{align*}
By monotonicity, it is straightforward to see that $ \widehat{r}_{N} $ is well-defined and is unique. By the definition of $ \mcal{T} $ again, for any $ \tau \leq \mcal{T} $, it holds that
\begin{align*}
N^2\eta{\tau}\widehat{R}^{2}_{\kappa}(1/\sqrt{\eta{\tau}N}) = \sum_{j = 1}^{N}\min\curly{\eta{\tau} \widehat{\lambda}_{j}, 1} \leq \frac{1}{4e^{2}D^{2}\eta{\tau}}.
\end{align*}
Hence, we deduce that $ \widehat{r}_{N} \leq 1/\sqrt{\eta{\tau}N}$. Therefore, we conclude from \eqref{eq:V-concentration-2} that it holds with probability at least $ 1 - \exp(-c'N\widehat{r}_{N}^{2}/D^{2}) $:
\begin{align*}
V_{\tau}^{i} \leq \E\bracket{V_{\tau}^{i}} + \frac{3}{e^{2}\eta \tau N} \leq \frac{5}{e^{2} \eta \tau N}.
\end{align*}
Therefore, we complete the proof.
\end{proof}

\begin{Lemma}\label{lemma:in-F}
    Assume the same assumptions as in Theorem \ref{thm:early-stop}. 
    With probability at least $1 - \exp(-cN\widehat{r}_N^2)$, the following inequality holds:
\begin{align}
\norm{(\tilde{f}_{\tau}^{K, i}(z) - f_{\mcal{H}}^{i}(z)) \mathbb{I}\curly{\norm{z}_2 \leq C_{\max}}}_{\infty} & \leq BC_{\max}^{2}, \label{eq:F-unif}
\end{align}
where we recall $B = \sqrt{3 + 2R_\mathcal{H}} + \sqrt{R_\mathcal{H}}$.
\end{Lemma}
\begin{proof}
We start with an upper bound for $ \norm{\tilde{f}_{\tau}^{K, i}}_{\kappa} $ at each $ \tau $. From the definition we have
\begin{align}
\norm{\tilde{f}_{\tau}^{K, i}}_{\kappa}^{2} & = \norm{\sum_{j = 1}^{N}K((X_{t_{j}}, t_{j}), \cdot)\tilde{\gamma}^{i}(\tau)}_{\kappa}^{2}  = (\tilde{\gamma}^{i}({\tau}))^{\top}H^{ii}(\tilde{\gamma}^{i}({\tau})) = (\tilde{u}^{K, i}(\tau))^{\top}(H^{ii})^{-1}\tilde{u}^{K, i}(\tau). \label{eq:tilde-f-norm-2}
\end{align}
The update rule of $ \tilde{u} $ implies 
\begin{align*}
\tilde{u}^{K, i}(\tau) = (I_{N} - (I_{N} - \eta H^{ii})^{\tau})\tilde{y}^{i}.
\end{align*}
Set $ M_{\tau} \coloneqq (I_{N} - (I_{N} - \eta H^{ii})^{\tau})^{\top}(H^{ii})^{-1}(I_{N} - (I_{N} - \eta H^{ii})^{\tau}) $. The eigen-decomposition $H^{ii} = V\Lambda V^\top$ leads to
\begin{align*}
    M_\tau & =  (I_{N} - (I_{N} - \eta H^{ii})^{\tau})^{\top}(H^{ii})^{-1}(I_{N} - (I_{N} - \eta H^{ii})^{\tau}) \\ & = (VV^\top - (VV^\top - \eta V\Lambda V^\top)^\tau)V\Lambda^{-1} V^\top (VV^\top - (VV^\top - \eta V\Lambda V^\top)^\tau) \\ & = (VV^\top - V(I_N - \eta \Lambda)^\tau V^\top)V\Lambda^{-1} V^\top(VV^\top - V(I_N - \eta \Lambda)^\tau V^\top).
\end{align*}
Recall the definition of $S(\tau)$, we further deduce
\begin{align}
     M_\tau  & = V(I_N - S(\tau))V^\top V \Lambda^{-1} V^\top V(I_N - S(\tau))V^\top \nonumber\\ & = V(I_N - S(\tau))\Lambda^{-1} (I_N - S(\tau))V^\top \nonumber \\ & = V(I_N - S(\tau))^2\Lambda^{-1}V^\top. \label{eq:M_tau}
\end{align}
With \eqref{eq:M_tau} and the fact $\tilde{y}^i = u_\mathcal{H}^i + \varepsilon^i$, we re-write \eqref{eq:tilde-f-norm-2} as 
\begin{align*}
    \norm{\tilde{f}_{\tau}^{K, i}}_{\kappa}^{2} & = \norm{\tilde{y}^{i}}_{M_{\tau}}^{2}  = \norm{u_{\mcal{H}}^{i}}^{2}_{M_{\tau}} + 2\inpro{u_{\mcal{H}}^{i}, \varepsilon^{i}}_{M_{\tau}} + \norm{\varepsilon^{i}}_{M_{\tau}}^{2}
\end{align*}
We apply \cite[Proposition 11]{kuzborskij2022learning} with $ g = \frac{f_{\mcal{H}}^{i}}{ \norm{f_{\mcal{H}}^{i}}_{\kappa}} $ to have $ \norm{u_{\mcal{H}}^{i}}^{2}_{M_{\tau}} \leq \norm{f_{\mcal{H}}^{i}}_{\kappa}^{2} $, $ \norm{\varepsilon^i}_{M_{\tau}}^{2} \leq 2 $ and $ \abs{\inpro{\varepsilon^{i}, u_{\mcal{H}}^{i}}_{M_{\tau}}} \leq \norm{f_{\mcal{H}}^{i}}_{\kappa} $ with probability at least $ 1 - e^{-cN\widehat{r}_{N}^2} $ for some constant $c>0$. 

Thus, with probability at least $ 1 - e^{-cN\widehat{r}_{N}^2} $, it holds for all $ \tau \leq \mcal{T} $ that
\begin{align*}
    \norm{\tilde{f}_{\tau}^{K, i}}_{\kappa}^{2} \leq \norm{f_{\mcal{H}}^{i}}_{\kappa}^{2}  + 2\norm{f_{\mcal{H}}^{i}}_{\kappa} + 2 \leq 2 \norm{f_{\mcal{H}}^{i}}_{\kappa}^{2} + 3.
\end{align*}
Consequently, we deduce that with probability at least $1 - \exp(-cN\widehat{r}_N^2)$ it holds that
\begin{align*}
\norm{(\tilde{f}_{\tau}^{K, i}(z) - f_{\mcal{H}}^{i}(z)) \mathbb{I}\curly{\norm{z}_2 \leq C_{\max}}}_{\infty} & = \sup_{\norm{z}_{2} \leq C_{\max}}\abs{(\tilde{f}_{\tau}^{K, i} - f_{\mcal{H}}^{i})(z)} \nonumber \\ & = \sup_{\norm{z}_{2} \leq C_{\max}}\abs{\inpro{\tilde{f}_{\tau}^{K, i} - f_{\mcal{H}}^{i}, \kappa(z, \cdot)}_{\kappa}} \nonumber \\ & \leq \norm{\tilde{f}_{\tau}^{K, i} - f_{\mcal{H}}^{i}}_{\kappa}\sup_{\norm{z}_{2} \leq C_{\max}}\kappa(z, z) \nonumber \\ & \leq BC_{\max}^{2}.
\end{align*}
Therefore, we complete the proof. 
\end{proof}   

\section{Verification of Assumptions}
In this section, we verify Assumptions \ref{ass:lip-f*} and \ref{ass:gram-eigen}. The following lemma provides an upper bound for $ \beta_1 $ (defined in Assumption \ref{ass:lip-f*}).

\begin{Lemma}\label{lemma:betax-bound}
    Suppose that Assumption \ref{ass:bounded target} holds. The Lipschitz constant $ \beta_1 $ in Assumption \ref{ass:lip-f*} can be bounded with
    \begin{align*}
    \beta_1 \leq \frac{D^2}{h(T_{0})}.
    \end{align*}
\end{Lemma}

\begin{proof}
    The proof essentially follows from the Tweedie's formula. We first observe that
    \begin{align*}
    p_{t\vert 0}(x \vert x_{0}) & \propto \exp\parenthesis{-\frac{1}{2h(t)}\norm{x - \alpha(t) x_{0}}_{2}^{2}} \\ & = \exp\parenthesis{-\frac{\norm{x}_{2}^{2}}{2h(t)}}\exp\parenthesis{\frac{\alpha(t)x^{\top}x_{0}}{h(t)}}\exp\parenthesis{-\frac{\alpha^{2}(t)\norm{x_{0}}_{2}^{2}}{2h(t)}}.
    \end{align*}
    Let $ \phi(x) = \exp\parenthesis{-\frac{\norm{x}_{2}^{2}}{2h(t)}} $ and $ T(x_{0}) = \alpha(t)x_{0}/h(t) $. We can write
    \begin{align*}
    p_{t\vert 0}(x \vert x_{0}) = \phi(x)\exp\parenthesis{x^{\top}T(x_{0})}\exp\parenthesis{\psi(x_{0})}.
    \end{align*}
    Here, $ \psi(\cdot) $ is a normalization function such that the integration of $ p_{t\vert 0}(\cdot \vert x_{0}) $ equals one. The Bayes' rule implies
    \begin{align*}
    p_{0 \vert t}(x_{0} \vert x) = \frac{p_{t\vert 0}(x \vert x_{0})p_{0}(x_{0})}{p_{t}(x)} = \exp\parenthesis{-\nu(x) + x^{\top}T(x_{0})}\bracket{p_{0}(x_{0})e^{\psi(x_{0})}},
    \end{align*}
    in which we define $ \nu(x) = \log(p_{t}(x)/\phi(x)) $. Since $ p_{0\vert t} $ is a probability density, we must have
    \begin{align*}
    0 & = \nabla_{x}\int p_{0 \vert t}(x_{0} \vert x){\rm d}x_{0} \\ & = \nabla_{x} \curly{e^{-\nu(x)}\int e^{x^{\top}T(x_{0})}p_{0}(x_{0}e^{\psi(x_{0})}){\rm d}x_{0}} \\ & = -\nabla \nu(x)e^{-\nu(x)}\int e^{x^{\top}T(x_{0})}p_{0}(x_{0}e^{\psi(x_{0})}){\rm d}x_{0} \\ & \qquad + e^{-\nu(x)}\int T(x_{0}) e^{x^{\top}T(x_{0})}p_{0}(x_{0}e^{\psi(x_{0})}){\rm d}x_{0} \\ & = -\nabla \nu(x) \int p_{0 \vert t}(x_{0} \vert x){\rm d}x_{0}  + \int T(x_{0})p_{0\vert t}(x_{0} \vert x){\rm d}x_{0} \\ & = -\nabla \nu(x) + \E\bracket{T(X_{0}) \vert X_{t} = x}.
    \end{align*}
    It follows that $ \nabla \nu(x) = \E\bracket{T(X_{0}) \vert X_{t} = x} $. Similarly, taking the second-order derivative yields
    \begin{align*}
        0 & = \nabla_{x}^{2}\int p_{0 \vert t}(x_{0} \vert x){\rm d}x_{0} \\ & = \nabla_{x} \bigg\lbrace-\nabla \nu(x)e^{-\nu(x)}\int e^{x^{\top}T(x_{0})}p_{0}(x_{0}e^{\psi(x_{0})}){\rm d}x_{0} \\ & \qquad + e^{-\nu(x)}\int T(x_{0}) e^{x^{\top}T(x_{0})}p_{0}(x_{0}e^{\psi(x_{0})}){\rm d}x_{0}\bigg\rbrace \\ & = - \parenthesis{\nabla^{2}\nu(x)e^{-\nu(x)} + \nabla \nu(x)(\nabla \nu(x))^{\top}e^{-\nu(x)}}\int e^{x^{\top}T(x_{0})}p_{0}(x_{0}e^{\psi(x_{0})}){\rm d}x_{0} \\ & \qquad - \nabla \nu(x)\parenthesis{e^{-\nu(x)}\int T(x_{0})e^{x^{\top}T(x_{0})}p_{0}(x_{0}e^{\psi(x_{0})}){\rm d}x_{0}}^{\top} \\ & \qquad - \nabla \nu(x)\parenthesis{e^{-\nu(x)}\int T(x_{0})e^{x^{\top}T(x_{0})}p_{0}(x_{0}e^{\psi(x_{0})}){\rm d}x_{0}}^{\top} \\ & \qquad + e^{-\nu(x)}\int T(x_{0})T(x_{0})^{\top}e^{x^{\top}T(x_{0})}p_{0}(x_{0}e^{\psi(x_{0})}){\rm d}x_{0} \\ & =-\nabla^{2}\nu(x) - \nabla \nu(x)(\nabla \nu(x))^{\top} - 2\nabla\nu(x)\parenthesis{\E\bracket{T(X_{0}) \vert X_{t} = x}}^{\top} \\ & \qquad + \E\bracket{T(X_{0})T(X_{0})^{\top} \vert X_{t} = x} \\ & = -\nabla^{2}\nu(x) +  \E\bracket{T(X_{0})T(X_{0})^{\top} \vert X_{t} = x} - \E\bracket{T(X_{0}) \vert X_{t} = x}\parenthesis{\E\bracket{T(X_{0}) \vert X_{t} = x}}^{\top}.
    \end{align*}
    We deduce that $ \nabla^{2}\nu(x) = {\rm Cov}(T(X_{0}) \vert X_{t} = x) $. Combined with  the definition of $ T(X_{0}) $, we have
    \begin{align*}
    \nabla_{x}\E\bracket{X_{0} \vert X_{t} = x} = \frac{\alpha(t)}{h(t)}{\rm Cov}(X_{0} \vert X_{t} = x).
    \end{align*}
    Since $ \alpha(t) \leq 1 $ and $ h(t) \geq h(T_{0}) $, Assumption \ref{ass:bounded target} implies that
    \begin{align*}
    \beta_1 \leq \norm{ \nabla_{x}\E\bracket{X_{0} \vert X_{t} = x} }_{2} \leq \frac{\alpha(t)}{h(t)}\norm{{\rm Cov}(X_{0} \vert X_{t} = x)}  \leq\frac{D^2}{h(T_0)},
    \end{align*}
    where we apply the Popoviciu’s inequality for the last line. Hence, we finish the proof.   
\end{proof}

Next, we provide a justification of Assumption \ref{ass:gram-eigen}.
Recall that  $ H $ denotes the Gram matrix of $ K $ and $ H^{ii} = [H^{ii}]_{jk} $ the Gram matrix of $ \kappa $ (independent of $ i $). For the scalar-valued NTK $ \kappa $, we refer the readers to \cite{nguyen2021tight} for a comprehensive analysis on the properties of $ H^{ii} $. Our next lemma shows that  $ H $ and $ H^{ii} $ share the same smallest eigenvalue for all $ i \in [d] $. 
\begin{Lemma}\label{lemma:H-Hii-eigen}
    Let $ H $ and $ H^{ii} $ be the  Gram matrices of matrix-valued NTK $ K $ and real-valued NTK $ \kappa $ respectively. Then, $ \lambda_{\min}(H) = \lambda_{\min}(H^{ii}) $. 
\end{Lemma}

\begin{proof}
    Denote $ v = (v_{1}^{\top}, \dots, v_{N}^{\top})^{\top}\in \R^{dN} $ with $ v_{j} = (v_{j}^{1}, \dots, v_{j}^{d})^{\top} \in \R^{d} $ and let $\lambda_0 \in \R$. Then, it holds that
    \begin{align*}
    v^{\top}Hv = \sum_{j = 1}^{N}\sum_{\ell = 1}^{N}v_{j}^{\top}H_{j\ell}v_{\ell} = \sum_{j = 1}^{N}\sum_{\ell}^{N}\sum_{i = 1}^{d}\sum_{k = 1}^{d}v_{j}^{i}H_{j\ell}^{ik}v_{\ell}^{k} = \sum_{i = 1}^{d}\sum_{k = 1}^{d}(v^{i})^{\top}H^{ik}v^{k} = \sum_{i = 1}^{d}(v^{i})^{\top}H^{ii}v^{i}.
    \end{align*}
    We first assume $ \lambda_{\min}(H) \geq \lambda_{0} $. Let $ i \in [d] $ be fixed and consider $ v $ with $ v^{k} = 0 $ for $ k \neq i $. Then we have
    \begin{align*}
    v^{\top}Hv = (v^{i})^{\top}H^{ii}v^{i} \geq \lambda_{0} (v^{i})^{\top}v^{i},
    \end{align*} 
  which follows $ \lambda_{\min}(H^{ii}) \geq \lambda_{0} $ since $ v^{i} $ is arbitrary. Conversely, suppose that $ \lambda_{\min}(H^{ii}) \geq \lambda_{0} $. For any $ v $, we must have
    \begin{align*}
    v^{\top}Hv \geq \lambda_{0}\sum_{i = 1}^{d}(v^{i})^{\top}v^{i} = \lambda_{0}v^{\top}v
    \end{align*}
    Since $ v $ is arbitrary, we conclude that $ \lambda_{\min}(H) \geq \lambda_{0}  $. Since $\lambda_0$ is arbitrary, we complete the~proof.   
\end{proof}

\end{document}